\documentclass[11pt,reqno]{amsart}
\usepackage[pdfstartview=FitB]{hyperref}
\usepackage{natbib}
\usepackage{geometry}
\usepackage{subcaption} 
\usepackage{amsmath,amssymb,amsthm}
\usepackage{graphicx}
\usepackage{todonotes}
\usepackage[doublespacing]{setspace}
\usepackage{float}
\usepackage{multirow}
\usepackage[utf8]{inputenc}
\usepackage{multicol}
\usepackage{xcolor}
\usepackage{lscape}
\usepackage{tabularx}

\usepackage{fix-cm}

\usepackage{algorithm} 

\usepackage{algpseudocode} 

\makeatletter

\newcommand{\algmargin}{\the\ALG@thistlm}

\makeatother

\newlength{\whilewidth}

\algnewcommand{\parState}[1]{\State%
\parbox[t]{\dimexpr\linewidth-\algmargin}{\strut #1\strut}}

\theoremstyle{plain}
\newtheorem{theorem}{Theorem}
\newtheorem{comment}{Comment}
\newtheorem{lemma}{Lemma}
\newtheorem{corollary}{Corollary}

\newtheorem{assumption}{Assumption}
\theoremstyle{definition}
\newtheorem{definition}{Definition}

\newtheorem{assumptionA}{A.\!\!}
\newtheorem{assumptionB}{B.\!\!}

\renewcommand{\sc}{\mathsf{c}}

\begin{document}

\title[L\textsubscript{2}-Boosting: Inference of Treatment Effects]{Estimation and Inference of Treatment Effects with L\textsubscript{2}-Boosting in High-Dimensional Settings}
\author{Jannis Kueck, Ye Luo, Martin Spindler, Zigan Wang}\thanks{\singlespacing Corresponding author: Martin Spindler, University of Hamburg, Hamburg Business School, Moorweidenstr. 18, 20148 Hamburg, Germany, martin.spindler@uni-hamburg.de. We thank Ross Levine and Chen Lin who kindly shared their original data for the empirical application. We also thank Oliver Schacht who supported us with the simulation studies.
We are grateful for valuable comments to Philipp Bach, Victor Chernozhukov, Sven Klaassen, Whitney Newey and participants at the International Workshop on \textquotedblleft Causal Inference, Program Evaluation, and External Validity\textquotedblright\ (Luxemburg) and at the annual meeting of the German Economic Association.}
\date{June, 2021.}
\maketitle

\begin{footnotesize}
\textbf{Abstract.} 
Empirical researchers are increasingly faced with rich data sets containing many controls or instrumental variables, making it essential to choose an appropriate approach to variable selection. In this paper, we provide results for valid inference after post- or orthogonal $L_2$-Boosting is used for variable selection. We consider treatment effects after selecting among many control variables and instrumental variable models with potentially many instruments.  
To achieve this, we establish new results for the rate of convergence of iterated post-$L_2$-Boosting and orthogonal $L_2$-Boosting in a high-dimensional setting similar to Lasso, i.e., under approximate sparsity without assuming the beta-min condition. These results are extended to the 2SLS framework and valid inference is provided for treatment effect analysis.
We give extensive simulation results for the proposed methods and compare them with Lasso. In an empirical application, we construct efficient IVs with our proposed methods to estimate the effect of  pre-merger overlap of bank branch networks  in the US on the post-merger stock returns of the acquirer bank.


\end{footnotesize}
\doublespacing
\section{Introduction}\label{sec1}
Boosting algorithms are  popular in machine learning and have proven to be useful for prediction and variable selection (\cite{nr:buhlmann.hothorn:2007}).
Nevertheless, in many applications,  researchers are interested in inference on pre-specified variables, and this requires a different approach. 
Often these variables are so-called treatment or policy variables that  researchers would like to learn and make inferences about, particularly in high-dimensional settings.  



It is well-known that inference after model selection may lead to invalid results when the selection step is not taken into account or safeguarded against. For an overview of the pitfalls and challenges of post-selection inference in general, we refer to \cite{nr:leeb.potscher:2005}. In this paper, we describe our approach to address this issue. First, we substantially extend the techniques in \cite{SL_2016} to provide a set of novel convergence results on iterated post-$L_2$-Boosting and orthogonal $L_2$-Boosting. These results are derived under approximate sparsity and without assuming the beta-min condition. Second, we provide results for valid inference in high-dimensional settings in which post- or orthogonal $L_2$-Boosting is applied for variable selection and estimation. For this, we combine the convergence results with the  techniques of orthogonalized moment conditions as described in \cite{CCDDHN:2016}.

Boosting has proven very valuable for prediction. We show in this paper that it can also be applied to (causal)  inference. In particular, we consider the case of the estimation of a treatment effect with many control variables, and the case of instrumental variable (IV) estimation such as two stage least squares procedure with many potential instruments. The first case can also be interpreted as estimation and inference on a prespecified variable in a high-dimensional linear regression model estimated with $L_2$-Boosting, where all other variables are considered as nuisance variables. 
As the first ingredient of our estimation method, it relies on the so-called orthogonalized or Neyman moment conditions. 
This theory was developed by Belloni, Chernozhukov, Hansen, and coauthors in a series of papers.
The case of instrumental variables is analyzed in \cite{BCCH:2012}, and the treatment effect case in \cite{BelloniChernozhukovHansen2011}.
Papers with extensions of the general idea are \cite{CCDDHN:2016} and \cite{CHS:2016}.
The second ingredient is rates of convergence that are fast enough to estimate the nuisance part of the orthogonal moment condition. Our results only rely on the standard assumptions in the high-dimensional literature, and nearly-Lasso convergence rates are derived for the iterated post-$L_2$-Boosting ($\mathtt{I}$-$\mathtt{pBA}$) and orthogonal $L_2$-Boosting ($\mathtt{oBA}$) algorithms. In addition, a data-driven early stopping criterion is proposed and we show that such stopping criterion achieves nearly-Lasso convergence rate under our assumptions. These results might be of independent interest since Boosting is widely used and the results for Boosting in the high-dimensional case under such weak assumptions have been missing in the literature to date.

Additionally, we conduct extensive simulation studies and compare the Boosting algorithm with Lasso in various high-dimensional settings. These results might give guidance to empirical researchers to decide when to choose which algorithm. We conclude that $L_2$-Boosting with a theoretical-grounded stopping criterion is the more efficient method with respect to computation time and also provides the best estimation accuracy in our simulation study. We also show the computational advantage of $L_2$-Boosting over Lasso by comparing the CPU time in a high-dimensional (causal) inference framework. Finally, we conduct an empirical analysis of how the pre-merger overlap of bank branch networks in the United States can affect the post-merger stock returns of acquirer banks. In this setting, many potential instrumental variables are available. We show how $L_2$-Boosting can be used to select these in a data-driven, principled way and to estimate the treatment effect of interest.


Boosting algorithms are one of the major recent advances in machine learning and statistics.
One of the earliest examples is Freund and Schapire's AdaBoost algorithm for classification (\cite{freund:1997}),  which has led to numerous variants that have proven to be competitive in terms of their prediction accuracy in a variety of applications and their strong resistance to overfitting as shown in \cite{nr:buhlmann.hothorn:2007}.
Boosting methods were originally proposed as ensemble methods, which rely on the principle of generating multiple predictions and majority voting (averaging) of the individual classifiers. An important development in the analysis of Boosting algorithms was Breiman’s interpretation of Boosting as a gradient descent algorithm in a function space (\cite{nr:breiman:1996}, \cite{nr:breiman:1998}), inspired by numerical optimization and statistical estimation.
Building on this insight, \cite{nr:friedman.hastie.ea:2000} and \cite{nr:friedman:2001} embedded Boosting algorithms into the framework of statistical estimation and additive basis expansion.
This also enabled Boosting to be applied to regression analysis. 
Boosting for regression was proposed by \cite{nr:friedman:2001}, and then \cite{nr:buhlmann.yu:2003} defined and introduced $L_2$-Boosting. \cite{nr:zhang.yu:2005} show the consistency of $L_2$-Boosting and derive the rate of convergence in low-dimensional settings, and  \cite{buhlmann2006sparse} show the consistency of $L_2$-Boosting in high-dimensional settings without stating the rate. \cite{SL_2016} derive rates for boosting and variants in high dimensions under a beta-min condition and 
 exact sparsity assumption. An extensive overview of the development of Boosting and its manifold applications is given in \cite{nr:buhlmann.hothorn:2007}.

The paper is structured as follows: In Section \ref{sec2}, we discuss the problem of estimating treatment effects in high-dimensional settings. In Section \ref{sec3}, we introduce $L_2$-Boosting and two variants. In Section \ref{main_results}, we present the formal results for the predictive performance of Boosting and for valid inference on (low-dimensional) treatment effects in a possibly high-dimensional setting. In Sections \ref{sec5} and \ref{sec6}, we provide a simulation study and an empirical application. Finally, we conclude in Section \ref{sec7}. Proofs and additional results of the simulation study are available in Appendices  \ref{app:proofs} and \ref{app:add_sim}.

\section{Estimation of Treatment Effects in High Dimensions}\label{sec2}
The goal is to estimate the treatment effect $\alpha$ of a treatment variable $D$ on an outcome variable $Y$ in a high-dimensional regression model, namely
\begin{equation}
\label{treateff}
Y = \gamma + \alpha D +X\beta + \varepsilon, \quad\mathbb{E}(\varepsilon|D,X)=0,
\end{equation}
where $\gamma$ denotes the intercept and $\varepsilon$ a statistical error term. The coefficient $\alpha$ of treatment variable $D$ is the treatment effect of interest. 
There are two reasons for including covariates $X=(X_1,\ldots,X_p)$ in model (\ref{treateff}) for the estimation of the treatment effect.
First, in randomized control trials the treatment variable is assigned randomly, and additional covariates improve the precision of the estimate of the average treatment effect. This argument has already been made in \cite{cox1958} and, more recently, \cite{lin2013agnostic}.
Second, in observational studies, covariates might be need to be included to establish unconfoundedness, in order to ensure that $\mathbb{E}(\varepsilon|D,X)=0$. This means that given the variables $X$, the treatment can be regarded as if it had been randomly assigned and there are no unobserved confounders, as described in  \cite{ImbensRubin:2015}.

A related question is that of which variables to include in model (\ref{treateff}) from a set of potential covariates.
In high-dimensional settings, when the number of covariates $p$ is larger than the sample size $n$, variable selection is inevitable because the least squares estimate is not well defined.
Even when $p$ is smaller than $n$ but the ratio $p/n\rightarrow c$ for some positive constant $c>0$, ordinary least squares estimates are unreliable and, again, variable selection is needed because including too many (noisy) covariates might overfit and disguise the true treatment effect. In the empirical study in Section \ref{sec6}, we analyze how the pre-merger overlap of bank branch networks in the US affects the post-merger stock returns of acquirer banks. The data set contains 442 observations and 153 potential instrumental variables, comprising a high-dimensional setting in which variable selection is needed in the first-stage regression.

In a naive approach to estimating the treatment effect in model (\ref{treateff}), one might first select the relevant covariates by using classical $t$-tests or modern machine learning methods, such as Lasso or $L_2$-Boosting, and then estimate the treatment effect by including only the selected variables and continue with standard inference methods.
However, this procedure, while often used in applied work, might fail to provide valid post-selection inference. This is because many modern methods like Lasso or Boosting yield consistent model selection only under very strong assumptions that are often unrealistic, particularly in applications in economics. 

In general, machine learning methods may miss key variables  that are moderately correlated with the outcome but highly correlated with the treatment variable. Although this does not hurt predictive performance, it introduces an omitted variable bias that leads to invalid post-selection inference.

We demonstrate this by a simple simulation study with a binary treatment variable and $100$ covariates.
The data generating process is given by
$$y_i=d_i\alpha_0+ x_i^\intercal \beta+\varepsilon_i\quad\text{and}\quad d_i=1_{\{0.2\cdot x_i^\intercal \gamma+\nu_i>0\}}$$
with
$\alpha_0=0$ for $i=1,\dots,100$. The high-dimensional coefficient vectors $\beta\in\mathbb{R}^{100}$ and $\gamma\in\mathbb{R}^{100}$ are sparse with $\|\beta\|_0=3$ and $\|\gamma\|_0=3$, respectively. The noise $\varepsilon_i$ is normally distributed with $\varepsilon_i \sim N(0,3)$ and $\nu_i$ following a standard logistic distribution. For the high-dimensional vector $x_i=(x_{i,1},\ldots,x_{i,100})^\intercal$, we set
\[x_i\sim \tiny{N\left (0,  \left[\begin{array}{ccccc}  1 & 0.8 & ...& ... & ...\\ 0.8 & 1 & 0.8 & ... & ... \\ 0 & 0.8 & 1 & 0.8 & ... \\ .... & ... & ... & ... & ... \\ ... & ... & ... & 0.8 & 1\\
\end{array}\right] \right)}.\]
The results for $10,000$ repetitions of the estimate $\hat \alpha_0$ are displayed in Figure \ref{naivevsdouble}.
The resulting distribution using the naive approach is highly biased, shows heavy tails and is not in line with a standard normal distribution (see Figure \ref{Naive}).

To deal with such a problem and provide valid post-selection inference, we apply the double selection approach, which is described in detail in Section \ref{controlsetting}  and implicitly creates an orthogonal moment condition. 
Figure \ref{DS} shows the empirical distribution of the estimates when employing the double selection methods introduced in \cite{BelloniChernozhukovHansen2011} and reviewed and extended in \cite{CCDDHN:2016}. The estimates are nearly unbiased and can be approximated by a normal distribution. The intuition of the double selection method is that it cures the omitted variable bias seen with machine learning methods by running an auxiliary regression/step. This is equivalent to using an orthogonal moment function. This property means that estimating the nuisance parameters with machine learning methods and plugging it into the moment function has no first order effect on the estimation of the target parameter, as long as the machine learning estimator does not converge too slowly. 

 

\begin{figure}
  \begin{subfigure}[b]{0.4\textwidth}
    \includegraphics[width=\textwidth]{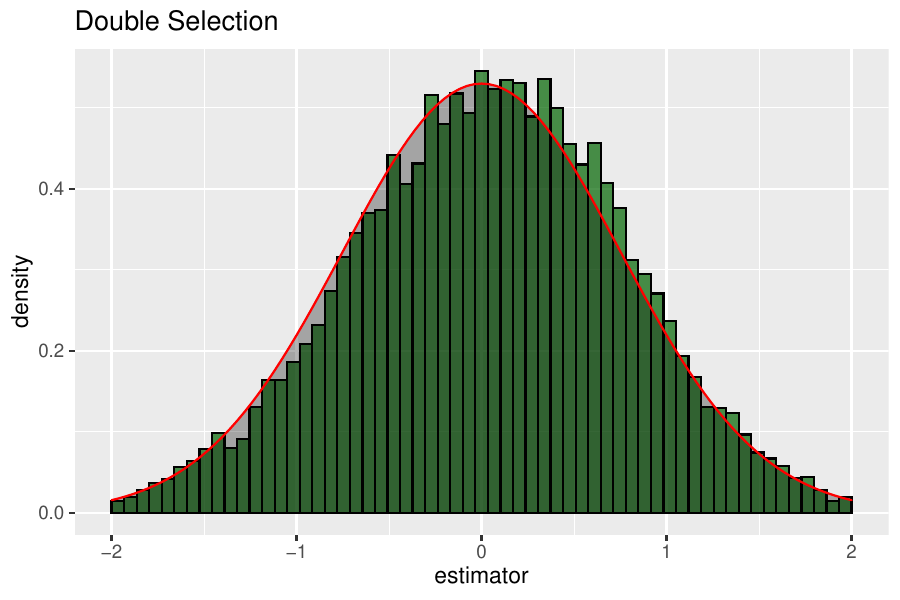}
    \caption{}
    \label{DS}
  \end{subfigure}
  \begin{subfigure}[b]{0.4\textwidth}
    \includegraphics[width=\textwidth]{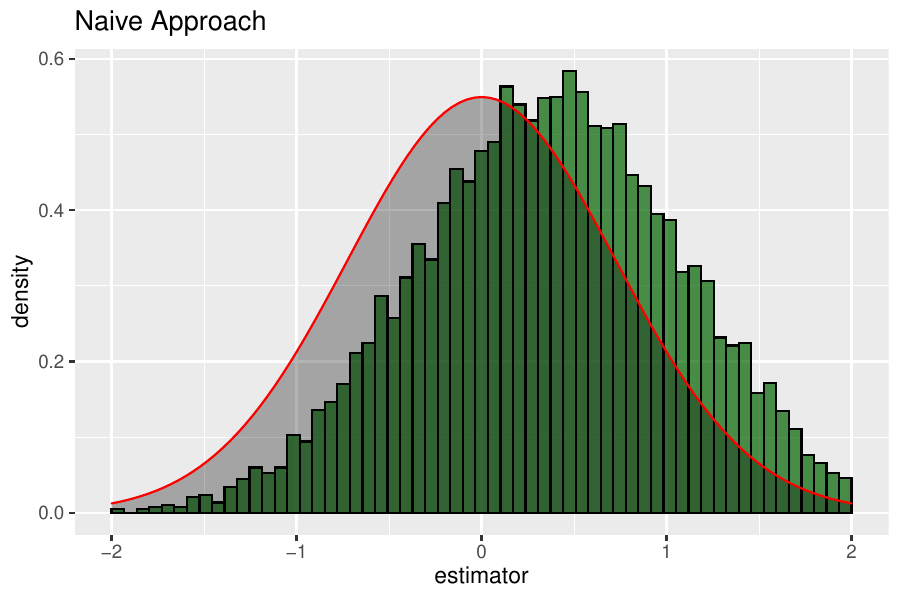}
    \caption{}
    \label{Naive}
  \end{subfigure}
	\caption{Histograms of the estimates $\hat \alpha$ of the treatment effect with the double selection method and the naive approach under a DGP with $\alpha_0=0$.}
	\label{naivevsdouble}
\end{figure}


\section{L\textsubscript{2}-Boosting}\label{sec3}
In this section, we describe the $L_2$-Boosting algorithm, namely the original Boosting algorithm for regression defined in \cite{nr:buhlmann.yu:2003} and two variants thereof (iterated post-$L_2$-Boosting and  orthogonal $L_2$-Boosting).
\subsection{L\textsubscript{2}-Boosting}
\noindent
To define the Boosting algorithm for linear models, we consider the following regression setting:
\begin{equation} y_i=x_i^\intercal \beta+\varepsilon_i, \quad i=1,\ldots,n, \end{equation}
where $x_i=(x_{i,1},\ldots,x_{i,p_n})^\intercal$ is a vector that consists of $p_n$ predictor variables. $\beta$ is a $p_n$-dimensional coefficient vector and $\varepsilon_i$ is a random, zero-mean error term with $\mathbb{E}[\varepsilon_i|x_i]=0$. We allow the dimension of the predictors $p_n$ to grow with the sample size $n$.
Also, the case $dim(\beta)=p_n\gg n$ is allowed.
In this setting, a so-called sparsity condition is unavoidable.
This means that there is a large set of potential variables, but the number of variables that have non-zero coefficients, denoted by $s$, is small compared to the sample size, i.e., $s<n$.
In our theoretical analysis, we weaken this assumption to approximate sparsity. In the rest of this  paper, we drop the dependence of $p_n$ on the sample size and denote it by $p$ so no confusion will arise.
Let $X$ denote the $n \times p$ design matrix where the single observations $x_i$ form the rows of $X$. Let
$X_j$ denote the $j$th column of the design matrix, and $x_{i,j}$ is the $j$th component of the vector $x_i$. For a set $\widetilde{T}\subset \{1,2,...,p\}$, we define $X_{\widetilde{T}}$ as the submatrix of $X$ with respect to column indices in $\widetilde{T}$. We assume a fixed design with $\max_{1\le j\le p}x_{i,j}\le C$ for all $i=1,\dots,n$ and $c\le\min_{1\le j\le p}\mathbb{E}_n[x_{i,j}^2]$ for absolute constants $0<c<C<\infty$. Without loss of generality, we consider standardized regressors, i.e., $\mathbb{E}_n[x_{i,j}]=0$ and $\mathbb{E}_n[x_{i,j}^2]=1$ for $j=1,\ldots,p$. Further assumptions will be imposed in the next sections.

The basic principle of $L_2$-Boosting is as follows: The criterion function that we want to  minimize is the sum of squared residuals, as is the case with ordinary least squares (OLS).
We initialize the estimator $\hat \beta$ to zero (strictly speaking, a $p$-dimensional vector consisting of zeros) and subsequently calculate the residuals, which in this case are equivalent to the observations in this first step.
Then, we conduct $p$ univariate regressions by  regressing the residuals (in the first round, the observations) on each of the $p$ regressors separately, resulting in $p$ univariate regressions.
After this, we select the variable  that explains most of the residuals and update the coordinate of our estimated coefficient vector in this direction. The above procedure is repeated until some stopping criterion is reached. 
The algorithm for $L_2$-Boosting with componentwise least squares is given below.

\begin{algorithm}[h]

	\caption{Standard $L_2$-Boosting ($\mathtt{BA}$)} 
	\label{alg1}
\begin{enumerate}
		\item Start/ initialization: $\beta^0 = 0$ ($p$-dimensional vector), $f^0=0$, set maximum number of iterations $m_{stop}$ and set iteration index $m$ to $0$.
    \item At the $(m+1)^{th}$ step, calculate the residuals $U_i^m=y_i - x_i^\intercal  \beta^m$.
    \item For each predictor variable $j=1,\ldots,p$, calculate the correlation with the residuals:
    \begin{equation*} \gamma^m_{j}:=\frac{\sum_{i=1}^n U_i^m x_{i,j}}{\sum_{i=1}^n x_{i,j}^2}=\frac{<U^m,X_j>_n}{\mathbb{E}_n[x_{i,j}^2]}.
    \end{equation*}
    Select the variable $j^m$ that is most correlated with the residuals, i.e., $$j^m:=\textrm{argmax}_{1\leq j\leq p}|corr(U^{m},X_{j})|.$$
    \item Update the estimator: $\beta^{m+1}:=\beta^{m}+\gamma^m e_{j^m}$, where $e_{j^m}$ is the $j^m$th index vector, $\gamma^m = \gamma^m_{j^m}$, 
     and
    $f^{m+1}:=f^{m}+\gamma^m X_{j^m}$.
    \item Increase $m$ by one. If $m<m_{stop}$, continue with (2); otherwise stop.
\end{enumerate}


\end{algorithm}

The version above and the orthogonal version, introduced later, are also known in deterministic settings as the pure greedy algorithm (PGA) and the orthogonal greedy algorithm (OGA), respectively.
Boosting is a gradient descent method.
In the $L_2$-case, the gradient of squared loss equals the residuals, and the residuals are iteratively fitted by a so-called base learner -- here, componentwise univariate regressions. Hence the gradient is approximated in each step  by a simple parametric model, namely a univariate linear model.
In the low-dimensional case, the estimator converges to the OLS solution.
In the high-dimensional case, overfitting can occur in the absence of the stopping criteria because the algorithms will ultimately build a model that explains all the variances in $y$.

The act of stopping is crucial for Boosting algorithms because stopping too late or never stopping  leads to overfitting, and therefore some kind of penalization is required. A suitable solution is to stop early, i.e., before overfitting takes place. Early stopping can be considered an unusual penalization/regularization scheme.
Similar to Lasso, early stopping might induce a bias through shrinkage.
A potential way to decrease the bias is to perform the post-$L_2$-Boosting approach  defined in the next section.

\subsection{Post- and Orthogonal L\textsubscript{2}-Boosting}\label{OGA}
Post-$L_2$-Boosting ($\mathtt{pBA}$) is a post-model selection estimator that applies ordinary least squares (OLS) to the model selected by  $L_2$-Boosting in the first step. To formally define this estimator, we make the following definitions:   $T_0:=supp(\beta)$ and $\hat{T}:=supp(\beta^{m})$, which are the support of the true model and the support of the model estimated by $L_2$-Boosting as described in Algorithm \ref{alg1} stopping at $m$. Further, $T^c$ denotes the complement of a set $T$ with regard to $\{1,\ldots,p\}$.

In the context of Lasso, OLS after model selection was analyzed in \cite{belloni:2013}.
Given the above definitions, the post-model selection estimator or post-$L_2$-Boosting estimator takes the form
\begin{equation}
\widetilde{\beta}= \arg\min\limits_{\beta \in \mathbb{R}^p} Q_n(\beta): \beta_j=0 \quad \mbox{for each} \quad j \in \hat{T}^c.
\end{equation}
where $Q_n(\beta)$ denotes the squared sum of residuals defined as $\sum_{i=1}^n (y_i - x_i^\intercal \beta)^2$.

In our theoretical analysis, we  consider a modified version of post-$L_2$-Boosting called iterated post-$L_2$-Boosting ($\mathtt{I}$-$\mathtt{pBA}$), which is given in Algorithm \ref{algo:iterated-pba} below. Instead of one final projection, a projection is conducted regularly, described by a sequence $l_m$, $m=0,1,...$

\begin{algorithm}[h]

	\caption{Iterated Post-$L_2$-Boosting ($\mathtt{I}$-$\mathtt{pBA}$)} 
	\label{algo:iterated-pba}
\begin{enumerate}
		\item Given a sequence $l_m\in \{0,1\}$, $m=0,1,...$. Start/ initialization: $\beta^0 = 0$ ($p$-dimensional vector), $f^0=0$, set maximum number of iterations $m_{stop}$ and set iteration index $m$ to $0$.
    \item At the $m^{th}$ step, calculate the residuals $U_i^m=y_i - x_i^\intercal \beta^m$, for $i=1,2,...,n$.
    \item For each predictor variable $j=1,\ldots,p$, calculate the correlation with the residuals:
    \begin{equation*} \gamma^m_{j}:=\frac{\sum_{i=1}^n U_i^m x_{i,j}}{\sum_{i=1}^n x_{i,j}^2}=\frac{<U^m,X_j>_n}{\mathbb{E}_n[x_{i,j}^2]}.
    \end{equation*}
    Select the variable $j^m$ that is the most correlated with the residuals, i.e., $$ j^m:=\textrm{argmax}_{1\leq j\leq p}|corr(U^{m},X_{j})|.$$
    \item If $l_m=0$, update the estimator according to gradient descent: $\beta^{m+1}:=\beta^{m}+\gamma^m_{j^m} e_{j^m}$, where $e_{j^m}$ is the $j^m$th index vector
     and
    $f^{m+1}:=f^{m}+\gamma^m_{j^m}X_{j^m}$.
    \item If $l_m=1$, perform post-$L_2$-Boosting:
   Define $\beta^{m+1}:=(X_{T^{m+1}}^\intercal X_{T_{m+1}})^{-1}X_{T_{m+1}}^\intercal Y$ with  $T^{m+1}:=\cup_{t=0}^m \{j^t\}$. 
    \item Increase $m$ by one. If $m<m_{stop}$, continue with (2); otherwise stop.
\end{enumerate}
\end{algorithm}

Another variant of the Boosting algorithm is orthogonal Boosting ($\mathtt{oBA}$), described in Algorithm \ref{alg3}, or the OGA in its deterministic version.
Only the updating step is changed in comparison to Algorithm \ref{alg1}. In every step, an orthogonal projection of the response variable is carried out on all the variables that have been selected up to that point.
The advantage of this method is that any variable is selected at the most once in this procedure, whereas in the previous version the same variable might be selected at different steps making the analysis much more complicated.
More formally, the method is described in Algorithm \ref{alg3} by modifying step (4) in Algorithm \ref{alg1}.

\begin{algorithm}[H]
	\caption{Orthogonal  $L_2$-Boosting ($\mathtt{oBA}$)} 
	\label{alg3}
\begin{enumerate}
		\item Start/ initialization: $\beta^0 = 0$ ($p$-dimensional vector), $f^0=0$, set maximum number of iterations $m_{stop}$ and set iteration index $m$ to $0$.
    \item At the $(m+1)^{th}$ step, calculate the residuals $U_i^m=y_i - x_i^\intercal \beta^m$, for $i=1,2,...,n$.
    \item For each predictor variable $j=1,\ldots,p$, calculate the correlation with the residuals:
    \begin{equation*} \gamma^m_{j}:=\frac{\sum_{i=1}^n U_i^m x_{i,j}}{\sum_{i=1}^n x_{i,j}^2}=\frac{<U^m,X_j>_n}{\mathbb{E}_n[x_{i,j}^2]}.
    \end{equation*}
    Select the variable $j^m$ that is most correlated with the residuals, i.e., $$j^m:=\textrm{argmax}_{1\leq j\leq p}|corr(U^{m},X_{j})|.$$
    \item Update the estimator: $\beta^{m+1}_o:=  ({X_{T^{m+1}_o}}^\intercal {X_{T^{m+1}_o}})^{-1} {X_{T^{m+1}_o}}^\intercal Y$, with $T_o^{m+1}:=\cup_{t=0}^m \{j^t\}$ 
     and
    $f^{m+1}:=X_{T^{m+1}_o} \beta_o^{m+1}$.
    \item Increase $m$ by one. If $m<m_{stop}$, continue with (2); otherwise stop.
\end{enumerate}
\end{algorithm}
As mentioned, early stopping is crucial for Boosting to avoid overfitting. Therefore, in the last step of Algorithms \ref{alg1}, \ref{algo:iterated-pba} or \ref{alg3}, an early stopping criteria is needed. The standard approaches for determining the ``optimal'' stopping time are cross-validation and a corrected Akaike information criterion (\cite{nr:buhlmann:2006}). Although both lack a theoretical foundation in a high-dimensional setting, they are applied by practitioners and often give competitive results. We provide a data-driven early stopping criterion later in Section \ref{main_results}, first developed in \cite{SL_2016}, with beta-min condition and exact sparsity needed for guaranteeing theoretical performance. In this paper, we prove the nearly-Lasso convergence rate of such early stopping criterion without the strong beta-min and exact sparsity conditions. We demonstrate the good performance of such early stopping criterion in our simulation study.

\subsection{Computational Perspective and Comparison to Lasso}

\citet{friedman2001elements} were the first to point out a strong relationship between $L_2$-Boosting with componentwise linear least squares and Lasso. Although these methods are generally not equivalent, \citet{efron2004least} proved an approximate equivalence between $L_2$-Boosting and Lasso and confirmed that $L_2$-Boosting and Lasso are closely related. Compared to Lasso, Boosting uses an unusual ``penalization scheme" by early stopping. $L_2$-Boosting can therefore be interpreted as an approximate and implicit regularized optimization, whereas Lasso directly solves a complex penalized optimization problem.

In this paper, we show that iterated  post-$L_2$-Boosting and orthogonal $L_2$-Boosting achieve nearly the same rate of convergence as Lasso in a high-dimensional linear regression model.  This aligns with the work of \citet{nr:buhlmann.hothorn:2007} who did not find Lasso to be superior overall to $L_2$-Boosting (or vice versa) in terms of empirical performance for prediction. More recently, \cite{hepp2016approaches} also performed an extensive simulation study comparing $L_2$-Boosting and Lasso. Their results suggest that in high-dimensional settings when $p>n$, Boosting has a slight advantage in terms of prediction accuracy, whereas Lasso leads, on average, to sparser models. They also point out that the amount of regularization imposed on the models clearly determines the quality of the estimation methods.

Next, we extend these results and perform simulation studies that compare $L_2$-Boosting and Lasso within advanced inferential procedures in high-dimensional linear regression models.
A simulation of this nature is lacking in the literature and in general little is known about the empirical performance and computation time of machine learning methods in causal/inferential settings (see, for example,  \cite{knaus2020double}).
As discussed, $L_2$-Boosting solves univariate regressions and is therefore computationally efficient and feasible in high dimensions, having a linear complexity in the number of predictor variables $p$ (cf. \cite{buhlmann2006sparse}, \cite{mayr2014evolution}). The computational complexity of $L_2$-Boosting, including $\mathtt{I}$-$\mathtt{pBA}$ and $\mathtt{oBA}$, is $O(npm_1)$, where $m_1=O(s\log(n))$ denotes the number of steps. Lasso has a computational complexity of $O(npm_2)$, where $m_2=O(\min\{n,p\})$ (cf. \citet{efron2004least}). This leads to a computational superiority of $L_2$-Boosting over Lasso in high-dimensional settings if $s\log(n)<< n$. This phenomenon has been observed in \citet{nr:buhlmann.hothorn:2007} who compare the computing time of $L_2$-Boosting to that of Lasso in high-dimensional regressions. Boosting is employed in practice in cases where explicitly solving regularized optimization problems is not practical. This  usually happens in very high-dimensional settings when $p>>n$, see, e.g., \citet{efron2004least}. In our simulation studies in Section \ref{sec5}, we confirm the computational advantage of Boosting over Lasso by comparing the CPU time in high-dimensional (causal) inference settings. Such advantage can be substantial when $n$ and $p$ become large.

Moreover, in settings where the number of observations $n$ is very large (\textquotedblleft big data\textquotedblright), an advantage of $L_2$-Boosting is that the estimation can be easily parallelized and distributed to different cores (e.g. CPUs) or machines. It is a gradient based method. In each step, $p$ univariate regressions are calculated  and are independent of each other and hence can be computed in a  distributed fashion. This allows for efficient implementation and the application to big data.

\section{Main Results}\label{main_results}

In this section, we provide the theoretical results of our paper, including  novel results  on the convergence rate of iterated post-$L_2$-Boosting ($\mathtt{I}$-$\mathtt{pBA}$) and orthogonal $L_2$-Boosting ($\mathtt{oBA}$) in subsection \ref{subsec:boosting-result}. Note that, compared to existing literature, the results rely only on approximate sparsity, and no beta-min condition is required. Based on these results, we present, in subsection \ref{boosting_inf}, inference results for treatment effects in a setting with high-dimensional controls and in an instrumental variable model.

\subsection{New Results on Boosting}\label{subsec:boosting-result}


Consider the following high-dimensional approximate sparse linear regression model
$$y_i=x_i^\intercal  \beta+r_{i}+\varepsilon_i, \quad i=1,\ldots,n,$$
where $r_i$ denotes the approximation error of the sparse model and $\varepsilon_i$ is i.i.d. across $i$ with variance $\sigma^2$.\footnote{In this subsection we consider homoscedastic error terms. Our proofs can be easily generalized to heteroscedastic errors.} We assume the following conditions to establish new rates of convergence for $L_2$-Boosting:
\begin{assumption}[Approximate Sparsity]\label{assump:AS}\
\begin{enumerate}
    \item[(i)] $\|\beta\|_0\le s$;
    
    \item[(ii)]  $\mathbb{E}_n[r_i^2]\leq \frac{C_r s\log(2p/\alpha)}{n}$ for some generic constant $C_r>0$ and $\alpha>0$ defined in Assumption \ref{assump:tail};
    
    \item[(iii)]  There exists a constant $K>1$ such that $\|X\beta\|_n^2\leq n^{K-1}$.
\end{enumerate}
\end{assumption}
\noindent
\begin{comment}
Assumption \ref{assump:AS} (i) and (ii) are standard sparsity conditions in high-dimensional linear regression literature, e.g., \cite{BelloniChernozhukovHansen2011}. Assumption \ref{assump:AS} (iii) is a very mild condition on the norm of the regression coefficient $\beta$.
\end{comment}

For the next assumption, we first define the smallest and largest restricted eigenvalues of a matrix:
\begin{definition}
The smallest and largest restricted eigenvalues of matrix $M$ with size $s'$ is defined as:
\begin{gather*}
    \phi_s(s',M) := \min_{W\in \Sigma(s',M)}\phi_s(W),\\
    \phi_l(s',M) := \max_{W\in \Sigma(s',M)} \phi_l(W),
\end{gather*}
where $\phi_s(W)$ and $\phi_l(W)$ denote the smallest and largest eigenvalues of matrix $W$ and $\Sigma(s',M)$ is the set of all $s'$-dimensional square diagonal-submatrices of $M$.
\end{definition}

\begin{assumption}[SE]\label{assump:SE}
Consider the Gram matrix $\Sigma = \mathbb{E}_n[x_i x_i^\intercal ]$. Assume that all the elements on the diagonal of $\Sigma$ are equal to one. We assume that there exist positive constants $c_\phi$ and $C_\phi$ such that 
$$0< c_{\phi}\leq \phi_s(s',\Sigma)\leq \phi_l(s',\Sigma)\leq C_{\phi}<\infty$$ 
holds for $s'\leq M_n$, where $M_n$ is a sequence such that $M_n\rightarrow \infty$ slowly along with $n$ and $M_n\geq  C_M s\log(n)$, where $C_M$ is a large enough generic constant.
\end{assumption}
\begin{comment}
This condition is standard for Lasso and other machine learning methods in a high-dimensional setting. It allows for a more general behavior requiring only that the sparse eigenvalues of the Gram matrix are bounded from above and away from zero. A standard assumption in traditional econometric research is to assume that the (population) Gram matrix has eigenvalues bounded from above and away from zero. The condition is fulfilled for many relevant designs. For various examples, we refer to \cite{belloni:2013}. An extensive overview of different conditions on matrices and how they are related is given in \cite{van2009conditions}.
\end{comment}

\noindent
Next, we require the following assumption on the tails of $X$ and $\varepsilon$:
\begin{assumption}[Tails]\label{assump:tail}
For any $\alpha$ small enough, with probability $\geq 1-\alpha$, we have:
$$\max_{1\leq j\leq p}|<X_{j},\varepsilon>_n|\leq 2\sigma \sqrt{\frac{\log (2p/\alpha)}{n}}:=\lambda_n,$$
where $<X_{j},\varepsilon>_n$ denotes the empirical inner product of two vectors with size $n\times 1$. In addition, we require that {$\hat\sigma^2:=\frac{1}{n}\sum_{i=1}^n\varepsilon_i^2 $} satisfies that 
\begin{equation*}
    |\hat\sigma^2-\sigma^2|\leq \eta \sigma^2,
\end{equation*}
for some small enough fixed constant $\eta\in (0,\frac{1}{2})$.
\end{assumption}
\noindent
\begin{comment}
Assumption \ref{assump:tail} holds, for example, if the error terms are i.i.d. normally distributed random variables. This in turn can be generalized/weakened to cases of non-normality by self-normalized random vector theory \citep{delapena} or the approach introduced in \cite{CCK:2014}.
\end{comment}

For the theoretical analysis, first, we consider the (iterated) Post-$L_2$-Boosting algorithm that is defined in Algorithm \ref{algo:iterated-pba}.
\begin{theorem}[Convergence of iterated-pBA ($\mathtt{I}$-$\mathtt{pBA}$)]\label{theo:PBA}
Consider Algorithm \ref{algo:iterated-pba}. 
Suppose that the sequence $l_m$ satisfies the condition that $l_m=1$ if and only if $m=k M$ for a positive integer $M\leq K_M\sqrt{s}$ and $k=1,2,\dots$, where $K_M$ is a fixed positive constant. Suppose the Assumptions \ref{assump:AS}-\ref{assump:tail} hold with $M_n\geq C_M s\log(n)$ for some large enough constant $C_M>0$. Let $m^*$ be a constant such that $m^*:=kM+1$ with $k=\left\lceil \frac{C_{m^*}s\log(n)}{M}\right\rceil$ such that 
    \begin{equation*}
    C_M-1\geq C_{m^*}\geq 1 + \left\lceil\frac{K}{\min\left(\frac{1}{2}, \frac{c_\phi^2}{16 C_\phi K_M^2},\frac{c_\phi^2}{64 C_\phi }\right)}\right\rceil.
\end{equation*}
Suppose that the growth condition $\frac{s\log(n)\log(2p/\alpha)}{n}\rightarrow 0$ holds. Then, for $n$ large enough, with probability $\geq 1-\alpha$, the iterated post-$L_2$-Boosting algorithm satisfies:
\begin{equation}\label{eq1pba}
    \|X(\beta^{m^*}-\beta)\|_n^2 \precsim \frac{s\log(p\vee n)\log n}{n}
\end{equation}
and 
\begin{equation}\label{eq2pba}
    \|\beta^{m^*}-\beta\|^2 \precsim \frac{s\log(p\vee n)\log n}{n}. 
\end{equation}
\end{theorem}

\begin{comment}
Theorem \ref{theo:PBA} shows that the iterated post-$L_2$-Boosting algorithm nearly achieves the Lasso rate of convergence $s\log(p\vee n)/n$. This algorithm implements post-$L_2$-Boosting iteratively and therefore, it shares the computational advance of the standard $\mathtt{pBA}$. The additional $\log(n)$ factor in the convergence rate result is inevitable due the nature of iterative algorithms like $\mathtt{I}$-$\mathtt{pBA}$ in the absence of beta-min conditions: in the best case, the $\mathtt{I}$-$\mathtt{pBA}$ algorithm and others alike will have $(1-\frac{\kappa}{s})^m$ convergence on $\|X(\beta^m-\beta)\|^2_n$ at $m^{th}$ step for some fixed constant $\kappa>0$. To accommodate that $s$ can  grow with $n$ and there can exist small coefficients in $\beta$, the number of iterations must be at least of order $O(s\log(n))$ in order to have either (1) the variables with small coefficients are already selected, or (2) the unselected variables' corresponding  coefficients must be small enough. The additional $\log(n)$ term can be removed if certain beta-min conditions are imposed.
It's not clear to us and in the literature if the standard $\mathtt{pBA}$ algorithm achieves the same rate of convergence under these mild assumptions and therefore, more restrictive growth conditions are required for $\mathtt{pBA}$ unless strong conditions such as the beta-min condition are employed.
\end{comment}

\begin{corollary}[$\mathtt{I}$-$\mathtt{pBA}$ Stopping Criterion]\label{coro:pba}
Assume that all assumptions in Theorem \ref{theo:PBA} hold. Consider the stopping criteria such that the Algorithm \ref{algo:iterated-pba} stops at the first $m$ such that:
\begin{equation*}
    \frac{\|U^{m+1}\|_n^2}{\|U^{m}\|_n^2}\geq 1- 4 \zeta_0 \frac{\log(2p/\alpha)}{n},
\end{equation*} where $\zeta_0$ is a fixed positive constant. Denote such $m$ as $m^*_A$. Suppose $M_n$ is large enough such that $M_n>\bar{M}$, where $\bar{M}$ is a generic  constant defined in \eqref{eq:def-C-prime} in the Appendix. 
Then, with probability $\geq 1-\alpha$, we have that:
$m^*_A\leq \bar{M}s\log(n)$. In addition, \eqref{eq1pba} and \eqref{eq2pba} hold when replacing $m^*$ with $m^*_A$.
\end{corollary}

\begin{comment}
Corollary \ref{coro:pba} is a natural extension of Theorem \ref{theo:PBA}. It provides a practical stopping criteria that achieves the same rate as stopping at $m^*$. In practice, $m^*_A$ could be much smaller than $m^*$. $\zeta_0$ is a tuning parameter that is similar to the penalty of Lasso. When $\zeta_0$ becomes larger, the algorithm stops earlier, leading to smaller number of iterations in the algorithm. In practice, we recommend to pick $\zeta_0=1$ as a rule of thumb, which is used in the simulation.
\end{comment}

The next theorem provides the same rates of convergence for the orthogonal $L_2$-Boosting ($\mathtt{oBA}$) algorithm that has also been defined in Section \ref{sec3}.

\begin{theorem}[Convergence of orthogonal $L_2$-Boosting  ($\mathtt{oBA}$)]\label{theo:OBA}
Suppose the Assumptions \ref{assump:AS}-\ref{assump:tail} hold with $M_n\geq C_M s \log(n)$ for some large enough constant $C_M>0$. Let $m^*$ be a constant such that $m^*:=\lceil C_{m^*}s\log(n)\rceil+1$ with 
    \begin{equation*}
    C_M -1 \geq C_{m^*}\geq 1 + \left\lceil\frac{K}{\min\left(\frac{1}{2}, \frac{c_\phi^2}{64 C_\phi }\right)}\right\rceil.
\end{equation*} Suppose that the growth condition $\frac{s\log(n)\log(2p/\alpha)}{n}\rightarrow 0$ holds. Then, for $n$ large enough,  with probability $\geq 1-\alpha$, we have:
\begin{equation}\label{eq1oba}
    \|X(\beta_o^{m^*}-\beta)\|_n^2 \precsim \frac{s\log(p\vee n)\log n}{n}
\end{equation}
and 
\begin{equation}\label{eq2oba}
    \|\beta_o^{m^*}-\beta\|^2 \precsim \frac{s\log(p\vee n)\log n}{n}. 
\end{equation}
\end{theorem}
\noindent
The proofs of Theorem \ref{theo:PBA} and Theorem \ref{theo:OBA} and the corresponding corollaries are given in Appendix \ref{app:proofs}.

\begin{comment}
Again, it is worth noting that neither exact sparsity nor the beta-min condition are required to show the results in Theorem \ref{theo:OBA}. Therefore, this theorem is a strong generalization of previous results in the literature.
\end{comment}
Similar to Corollary \ref{coro:pba}, the following corollary is a natural extension of Theorem \ref{theo:OBA}.
\begin{corollary}[$\mathtt{oBA}$ Stopping Criterion]\label{coro:oba}
Assume that all assumptions in Theorem \ref{theo:OBA} hold. Consider the stopping criteria such that the Algorithm \ref{alg3} stops at the first $m$ such that:
\begin{equation*}
    \frac{\|U^{m+1}\|_n^2}{\|U^{m}\|_n^2}\geq 1- 4\zeta_0 \frac{\log(2p/\alpha)}{n},
\end{equation*} where $\zeta_0$ is a fixed positive constant. Denote such $m$ as $m^*_A$. Suppose $M_n$ is large enough such that $M_n>\bar{M}$, where $\bar{M}$ is a constant defined in \eqref{eq:def-C-prime}. 
Then, with probability $\geq 1-\alpha$, we have that:
$m^*_A\leq \bar{M}s\log(n)$. In addition, \eqref{eq1oba} and \eqref{eq2oba} hold when replacing $m^*$ with $m^*_A$.
\end{corollary}

\subsection{Inference for Treatment Effects}\label{boosting_inf}
In this section, we  present  inference  results  for  treatment  effects using $L_2$-Boosting. 
We first consider the case in which a researcher is interested in estimating the treatment effect $\alpha_0$ of a treatment variable $d$ after selecting among many control variables as discussed in Section \ref{sec2}. Then, we provide new results for inference in an instrumental variable model with potentially many instruments when post- or orthogonal $L_2$-Boosting
is used for the variable selection. Essential ingredients for valid inference in these settings are the rates of convergence of $L_2$-Boosting that have been provided in the previous subsection.

\subsubsection{Inference Results after Selecting among Many Controls}\label{controlsetting}

We consider the model
\begin{eqnarray}
\label{Controls1}
y_i &=& d_i \alpha_0 +  x_i^\intercal \beta + r_{1,i} + \xi_i,\quad\mathbb{E}[\xi_i|d_i, x_i]=0\\
\label{Controls2}
d_i &=& x_i^\intercal \gamma + r_{2,i} + \nu_i,\quad\mathbb{E}[\nu_i|x_i]=0,
\end{eqnarray}
where $r_{1,i}$ and $r_{2,i}$ are approximation errors, $i=1,\dots,n$. The estimation method consists of the following three steps, the first two of which involve model selection with $L_2$-Boosting:
\begin{enumerate}
	\item Run an iterated  post- or orthogonal $L_2$-Boosting regression of $d_i$ on $x_i$.
The set of variables that is selected will be denoted by $\hat{I}_1$.
	\item Run an iterated  post- or orthogonal $L_2$-Boosting regression of $y_i$ on $x_i$.
The set of variables that is selected will be denoted by $\hat{I}_2$.
	\item Run an OLS regression of $y_i$ on the treatment variable $d_i$ and the set of variables selected in the first two steps.
This set can be augmented by additional variables.
\end{enumerate}
The estimated regression coefficient of the treatment variable in step (3) above is the double selection estimator $\check{\alpha}$. To analyze this estimator based on $L_2$-Boosting, we impose the following assumptions:
\begin{assumptionA}\label{A1} Let $c_\phi, C_\phi$ and $C_r$ be absolute positive constants. The following assumptions hold:
\begin{enumerate}
\item[(i)] We observe $w_i=(y_i,d_i,x_i)$ $i.i.d.$ \text{on}  $(\Omega, \mathcal{F}, P)$ obeying (\ref{Controls1}) and (\ref{Controls2}) for $i=1,\ldots,n$.
\item[(ii)] The model is approximately sparse: $\vert\vert \beta\vert\vert_0 \leq s$ and $\vert\vert \gamma\vert\vert_0 \leq s$, and the approximation errors $r_{1,i}$ and $r_{2,i}$ satisfy: 
$\mathbb{E}_n[r_{1,i}^2]\leq \frac{C_r s\log(2p/\alpha)}{n}$ and $\mathbb{E}_n[r_{2,i}^2]\leq \frac{C_r s\log(2p/\alpha)}{n}$ for some fixed $\alpha>0$. Furthermore, there exists a constant $K>1$ such that $\|X\beta\|_n^2\leq n^{K-1}$ and $\|X\gamma\|_n^2\leq n^{K-1}$.
\end{enumerate}
\end{assumptionA}
\begin{assumptionA}\label{A2} 
We assume that $0 <c_\phi\leq \phi_{s}(s',\mathbb{E}_n[x_i x_i^\intercal ])\leq \phi_{l}(s',\mathbb{E}_n[x_i x_i^\intercal ])\leq C_\phi < \infty$ for any $s'\leq M_n$,
where $M_n$ is a sequence such that $M_n\rightarrow \infty$ slowly along with $n$, and $M_n\geq C_M s\log(n)$, for some large enough constant $C_M$.
\end{assumptionA}


\begin{assumptionA}\label{A3}
With probability greater or equal $1-\alpha$, we have $$\max_{1\leq j \leq p} |<x_{ij},\varepsilon_i>_n| \leq 2 \widetilde\sigma \sqrt{\frac{\log (2p/\alpha)}{n}}$$ for $\varepsilon_i = \xi_i$ and $\varepsilon_i = \nu_i$ with $\widetilde\sigma:=\sqrt{Var(\varepsilon_i)}$ and $|\hat{\widetilde\sigma}^2-\widetilde\sigma^2|\leq \eta \widetilde\sigma^2$ for fixed $\eta\in (0,\frac{1}{2})$.
\end{assumptionA}

\begin{assumptionA}\label{A4} There exists an absolute constant $4<q<\infty$ such that
\item[(i)] $\mathbb{E}\left[ |\xi|^q + |\nu|^q \right] \leq C$. Further, it holds that $c\leq \mathbb{E}[\xi^2|d_i, x_i] \leq C$ a.s. and $c\leq \mathbb{E}[\nu^2|x_i]\leq C$ for all $i=1,\dots,n$, where $c$ and $C$ are some generic positive constants.
\item[(ii)] It holds that 
\begin{enumerate}
\item[(a)] $\frac{s^2 \log^2(p \vee n)\log(n)^2}{n} \rightarrow 0$, and
\item[(b)] $n^{2/q}\frac{s^2 \log(p \vee n)\log(n)^2}{n} \rightarrow 0$.
\end{enumerate}
\end{assumptionA}

Assumptions A.\ref{A1}--A.\ref{A3} allow us to apply Theorem \ref{theo:PBA} and Theorem \ref{theo:OBA} to ensure sufficiently fast estimation rates for $L_2$-Boosting in models (\ref{Controls1}) and (\ref{Controls2}). Assumption A.\ref{A4} (i) imposes weak technical conditions on the moments of the error terms to allow for valid inference. Assumption A.\ref{A4} (ii) restricts the growth of the number of parameters. Moreover, we assume that, in the $L_2$-Boosting algorithms, early stopping takes place.
This ensures that the conditions for the stopping time $m^*$ in Theorem \ref{theo:PBA} and Theorem \ref{theo:OBA} are satisfied. With these assumptions, we can now formulate our first main theorem regarding inference.

\begin{theorem}
\label{ControlTheorem}

Let $\{P_n\}$ be a sequence of data generating processes for which Assumptions A.\ref{A1}-A.\ref{A4} hold for $P=P_n$ and each $n$.
Then, the double-selection estimator based on iterated post-$L_2$-Boosting/orthogonal $L_2$-Boosting $\check{\alpha}$ satisfies
\begin{equation}
\hat{\sigma}_n^{-1} \sqrt{n}(\check{\alpha} - \alpha_0) \rightarrow_D N(0,1)
\label{eq:}
\end{equation} 
with $$\hat{\sigma}^2_n = [\mathbb{E}_n \hat{\nu}_i^2]^{-1}  \mathbb{E}_n [\hat{\nu}_i^2 \hat{\xi}_i^2] [\mathbb{E}_n \hat{\nu}_i^2]^{-1}$$
 for $\hat{\xi_i}:=(y_i - d_i \check{\alpha} - x_i^\intercal \hat{\beta})(n/(n-\hat{s}-1))^{1/2}$ and $\hat\nu_i:=d_i - x_i^\intercal  \hat{\gamma}, i=1,\ldots,n$, where $\hat{\beta}$ 
denotes the post-double selection estimator and $\hat{s}=||\hat{T}||_0$.
\end{theorem}
The proof of Theorem \ref{ControlTheorem} is given in Appendix \ref{app:proofs}. This result can be used to conduct valid inference on the regression coefficient $\alpha_0$. The construction of uniformly valid confidence intervals is given in the following corollary.
\begin{corollary}\label{coro1}
Let $\mathbf{P}_n$ the collection of all data generating processes $P$ for which the assumptions of Theorem \ref{ControlTheorem} hold for given $n$. Further, let $\mathbf{P}$ be the collection of data-generating processes for which the conditions above hold for all $n\ge n_0$, and define $c(1-\xi):=\Phi^{-1}(1-\xi/2)$. The confidence regions based upon $\check{\alpha}$ and $\hat{\sigma}_n$ are uniformly valid in $P \in \mathbf{P}$:
\[ \lim_{n \rightarrow \infty} \sup_{P \in \mathbf{P}} |P(\alpha_0 \in [\check{\alpha} \pm c(1-\xi)\hat{\sigma}_n/\sqrt{n}]) - (1-\xi)|=0.\]
\end{corollary}

\subsubsection{Inference on Treatment Effects in an Instrumental Variable Model}\label{IVsetting}
In this section, we consider the following instrumental variable model with potentially very many instruments
\begin{align*}
y_i &= d_i \alpha_0 + x_i^\intercal \beta + \varepsilon_i,\quad \mathbb{E}[\varepsilon_i|z_i]=0\\
d_i &= z_i^\intercal \gamma+r_i+\nu_i
\end{align*}
with instrument function $D_i=D(z_i)=\mathbb{E}[d_i|z_i]= z_i^\intercal \gamma+r_i$. Again, we allow for approximate sparsity in the second equation. For simplicity, in our technical analysis, we consider the model above without controls $x_i$ in the first equation and a regular fix design $Z$ with observations $z_i$:
\begin{eqnarray}
\label{ins1}
y_i &=& d_i \alpha_0+ \varepsilon_i,\quad \mathbb{E}[\varepsilon_i|z_i]=0\\
\label{ins2}
d_i &=&z_i^\intercal \gamma +r_i+\nu_i.
\end{eqnarray}
To estimate the coefficient $\alpha_0$ of the endogenous treatment variable, we employ the following two-stage least squares (2SLS) procedure:
In the first step, we estimate and predict the instrument $\hat{D}_i=z_i^\intercal \hat{\gamma}$ by iterated post- or orthogonal $L_2$-Boosting. Finally, we estimate $\hat{\alpha}_0$ by regressing the outcome variable $y$ on the predicted instrument $\hat{D}_i$. To analyze this estimator based on $L_2$-Boosting, we impose the following assumptions:
\begin{assumptionB} \label{B1} Let $c_\phi,C_\phi$ and $C_r$ be absolute constants. The following assumptions hold:
\begin{enumerate}
\item[(i)] 
The data  $(y_i, d_i, z_i)$ are i.i.d. on $(\Omega, \mathcal{F}, P)$ and obey  the linear IV model in (\ref{ins1}) and (\ref{ins2}).
\item[(ii)] The optimal instrument function $D_i=z_i^\intercal \gamma +r_i$ can be approximated sufficiently well by $s$ instruments:
\[||\gamma||_0 \leq s,\quad \mathbb{E}_n[r_{i}^2]\leq  \frac{C_r s\log(2p/\alpha)}{n}\quad \text{and}\quad \|Z\gamma\|_n^2\leq n^{K-1}\]
for a constant $K>1$.
\end{enumerate}
\end{assumptionB}
\begin{assumptionB}\label{B2}
We assume that $0 < c_\phi\leq \phi_{s}(s',\mathbb{E}_n[z_i z_i^\intercal ])\leq \phi_{l}(s',\mathbb{E}_n[z_i z_i^\intercal ])\leq C_\phi < \infty$ for any $s'\leq M_n$,
where $M_n$ is a sequence such that $M_n\rightarrow \infty$ slowly along with $n$, and $M_0\geq C_M s\log(n)$ for some large enough positive constant $C_M$.
\end{assumptionB}
\begin{assumptionB}\label{B3}
With probability greater or equal $1-\alpha$, it holds $\max_{1\leq j \leq p} |<z_{ij},\nu_i>_n| \leq 2 \widetilde\sigma \sqrt{\frac{\log (2p/\alpha)}{n}}$ for $\widetilde\sigma:=\sqrt{Var(\nu_i)}$ and $|\hat{\widetilde\sigma}^2-\widetilde\sigma^2|\leq \eta \widetilde\sigma^2$ for $\eta\in (0,\frac{1}{2})$.
\end{assumptionB}

\begin{assumptionB}\label{B4} There exists an absolute constant $q > 4$ such that
\item[(i)] $\mathbb{E}[|\varepsilon|^{q}] + \mathbb{E}[|\nu|^{q}]\le C$. Furthermore, it holds that $c\leq \mathbb{E}[\varepsilon^2|z_i]\leq C$ for all $i=1,\dots,n$, where $c$ and $C$ are absolute positive constants.
\item[(ii)] It holds that 
		\begin{enumerate}
			\item[(a)] $\frac{s^2 \log^2(p \vee n)\log^2(n)}{n} \rightarrow 0$.
		\end{enumerate}
\end{assumptionB}
Assumptions B.\ref{B1}-B.\ref{B4} are essentially the same as  Assumptions A.\ref{A1}-A.\ref{A4} except for some small modifications in B.\ref{B4} related to the different underlying setting in subsection \ref{controlsetting}. It is worth noting that the growth condition B.\ref{B4}(ii) is slightly weaker than the growth condition A.\ref{A4}(ii) because Assumption A.\ref{A4}(ii) implies Assumption B.\ref{B4}(ii). Assumption A.\ref{A1}(ii)(b) is needed to ensure consistency of the respective variance estimators. With these assumptions, we can show that the instrumental variable estimator $\hat{\alpha}_0$ following the two-stage least squares (2SLS) procedure is asymptotically normally distributed. This result is provided by the following theorem.
\begin{theorem}\label{ivtheorem}
Let $\{P_n\}$ be a sequence of data generating processes for which Assumptions B.\ref{B1}-B.\ref{B4} hold for $P=P_n$ and each $n$. Then, the instrumental variable estimator $\hat{\alpha}$ based on iterated post-$L_2$-Boosting or orthogonal $L_2$-Boosting of the optimal instrument satisfies
\[ (\hat{Q}^{-1} \hat{\Omega} \hat{Q}^{-1})^{-1/2} \sqrt{n} (\hat{\alpha}_0 - \alpha_0) \rightarrow_D N(0,1) \]
for $\hat{\Omega}:=\mathbb{E}_n[\hat{\varepsilon}_i^2 \hat{D}(z_i)^2]$ and $\hat{Q}:=\mathbb{E}_n[\hat{D}(z_i)^2]$ with $\hat{\varepsilon}_i=y_i-d_i\hat{\alpha}_0$ and $\hat{D}(z_i)=z_i^\intercal \hat{\gamma}$.
\end{theorem}
\noindent
This also enables us to construct uniformly valid confidence intervals for the treatment effect as in Corollary \ref{coro1}. The proof of Theorem \ref{ivtheorem} is given in Appendix \ref{app:proofs}.

\section{Simulation Study}\label{sec5}
In this section, we present simulation results for both settings.
\subsection{Setting with High-Dimensional Controls}\label{sim_controls}
First, we consider the following i.i.d. data generating process:
\begin{align}
y_i &= d_i \alpha_0 + x_i^\intercal \theta_g + \xi_i\\
d_i &= x_i^\intercal \theta_m + \nu_i,
\end{align}
where $x_i \sim N(0, \Sigma)$. The second auxiliary equation captures the confounders. We consider two different covariance structures $\Sigma$: First, an i.i.d. structure with $\Sigma=I_p$ being the $p \times p$ identity matrix and secondly a structure with correlated covariates, for which we generate a Toeplitz design $\Sigma_{kj}=0.5^{|j-k|}$ but then perform a random covariate rotation. 
The parameter of interest, $\alpha_0$, is equal to $0.5$.
We consider both a sparse setting (\textit{Control-1}) and an approximately sparse setting (\textit{Control-2}). In the sparse setting, the first $s=20$ coefficients are set to one and all other parameters $p-s$ are equal to zero. In the approximately sparse setting, the first $s=10$ coefficients are set to one and the remaining coefficients are declining:
\begin{align*}
\theta_g=\theta_m=\begin{cases}1 \quad & \text{for }j\le s\\
0.8^{(j-s)} & \text{for }j>s.\end{cases}
\end{align*}
The residuals $\xi_i\sim N(0,\sigma^2)$ and $\nu_i\sim N(0,\sigma^2)$ are generated independently with $$\sigma = \sqrt{\theta_g^\intercal\Sigma\theta_g/\mathrm{SNR}}$$ to control the empirical signal-to-noise ratio ($\mathrm{SNR}$),
$$\mathrm{SNR}=\mathbb{E}_n[(x_i^\intercal\theta_g)^2]/\mathbb{E}_n[\xi_i^2]=\mathbb{E}_n[(x_i^\intercal\theta_m)^2]/\mathbb{E}_n[\nu_i^2].$$
We vary the sample size $n$, the number of covariates $p$ and the $\mathrm{SNR}$. The number of repetitions is $R=500$.
For estimation of the treatment effect $\alpha_0$ we use the double selection procedure presented in Section \ref{controlsetting}. Because we use different approaches for variable selection, we are able to compare $L_2$-Boosting and Lasso in terms of how capable they are in treatment effect estimation. The results are provided in subsection \ref{sim_results}.

\subsection{IV Estimation with Many Instruments}\label{sim_IV}
In the setting with many instrumental variables, we consider the following i.i.d. data generating process, which is  similar to that used in the simulation experiment in \cite{BCCH:2012}:
\begin{align}
y_i &= d_i\alpha_0+ \varepsilon_i,\\
d_i &= \gamma^\intercal  z_i + \nu_i
\end{align}
with
\begin{align*}
(\varepsilon_i,\nu_i) &\sim N \left(0 ,\left( \begin{array}{cc}
	\sigma_\varepsilon^2 & \sigma_{\varepsilon\nu}\\
	\sigma_{\varepsilon\nu} & \sigma_\nu^2
\end{array} \right) \right).
\end{align*}
The parameter of interest is set to $\alpha_0=1$. The regressors $z_i=(z_{i1}, \ldots, z_{ip})^\intercal$ are drawn from a normal distribution $N(0, \Sigma_Z)$ for the same covariance structures as in subsection \ref{sim_controls}. Further, we set $corr(\varepsilon,\nu)=0.1$ and $\sigma^2_\varepsilon=1$.
Let $\sigma^2_{\nu}=1-\gamma^\intercal\Sigma_z \gamma$ such that the unconditional variance of the endogenous variable equals one.
The first-stage coefficients are set according to $\gamma= C \widetilde{\gamma}$.
For the high-dimensional regression coefficient $\widetilde{\gamma}$, we use a sparse design, i.e., $\widetilde{\gamma}=(1, \ldots,1,0,\ldots,0)$ with $s$ coordinates equal to one and all other $p-s$ equal to zero. The constant $C$ is set in such a way that we generate data with an SNR of $1$ and $3$. This leads to the concentration parameter $\mu^2 = \frac{n \gamma^\intercal\Sigma_z \gamma}{\sigma_{\nu}^2}$ of size $n$ or $3n$, which determines the behavior of the instrumental variable estimators as described in \cite{hansen2008estimation}. As in subsection \ref{sim_controls}, we also vary the sample size $n$ and the number of covariates $p$.
The number of repetitions in the simulations study is again $R=500$. We subsequently employ two-stage least squares with $L_2$-Boosting and Lasso to compare the performance in estimating the treatment effect $\alpha_0$ in a high-dimensional setting with many instrumental variables.

\subsection{Simulation Results}\label{sim_results}
The results of both simulation settings described in subsection \ref{sim_controls} and \ref{sim_IV} are summarized in Tables \ref{simresults_n400} -- \ref{simresults_n800_SNR3} in Appendix \ref{app:add_sim}.
 We present the mean absolute error (MAE) of $\hat{\alpha}$,
$$\frac{1}{R}\sum\limits_{r=1}^R|\hat{\alpha}_r-\alpha_0|,$$
the standard deviation of the estimators and the empirical rejection rates at a significance level $0.05$. In our analysis, \texttt{pBA} refers to post-$L_2$-Boosting, \texttt{I-pBA} to the iterated version, and \texttt{oBA} to the orthogonal variant as in Section \ref{OGA}. Each of these uses the theoretically justified stopping criterion from Corollary \ref{coro:pba} and Corollary \ref{coro:oba}. The \texttt{Lasso} method refers to the Post-Lasso estimator with theoretically justified selection of the shrinkage parameter $\lambda$.
For comparison, we also use an implementation of Post-$L_2$-Boosting with cross-validated stopping (\texttt{CV-pBA}) 
as well as a Lasso implementation with cross-validated $\lambda$.
The estimators based on $L_2$-Boosting yield a low MAE throughout all these settings. On average, the methods based on $L_2$-Boosting also provide empirical rejection rates that are close to the nominal level of $0.05$. 
Compared to Lasso, the MAE provided by $L_2$-Boosting is slightly lower and, on average, 
the empirical rejection rates are slightly closer to the nominal level of $0.05$. Nevertheless, the differences in MAE between $L_2$-Boosting and Lasso are all within one standard error of each other. In summary, both $L_2$-Boosting and Lasso provide empirical rejection rates that are close to the nominal level of $0.05$. This is in line with the theory that both $L_2$-Boosting and Lasso can be used for valid estimation of treatment effects in high-dimensional settings. Further, the cross-validated methods yield a higher MAE and higher standard deviations than the methods with a theoretically justified tuning parameter. Table \ref{simresults_n600_part} below summarizes the results for the setting with $n=600$ and $SNR=1$.

\begin{table}[h]
\resizebox{0.8\textwidth}{!}{%
\begin{tabular}{|r|rrrrrr|rrrrrr|r|}
  \hline
  &\multicolumn{6}{c|}{iid covariates}&\multicolumn{6}{c|}{correlated covariates}&\\
  \hline
 & \textbf{Lasso} & \textbf{pBA}& \textbf{I-pBA} &  \textbf{oBA} & \textbf{CV-pBA} & \textbf{CV-Lasso} & \textbf{Lasso} & \textbf{pBA}& \textbf{I-pBA} & \textbf{oBA} & \textbf{CV-pBA} & \textbf{CV-Lasso} &  \\  
   \hline
  \textbf{p}&\multicolumn{12}{c|}{Mean Absolute Error}&\\
  \hline
  200 & 0.033 & 0.033 & 0.035&0.034 & 0.037 & 0.038 & 0.035 & 0.036 & 0.035&0.036 & 0.038 & 0.037 & \parbox[t]{2mm}{\multirow{5}{*}{\rotatebox[origin=c]{-90}{Control-1}}}\\ 
  600 & 0.033 & 0.034 & 0.037&0.034 & 0.042 & 0.050 & 0.035 & 0.036 & 0.035&0.036 & 0.041 & 0.044 &\\ 
  1000 & 0.037 & 0.036 &0.036& 0.036 & 0.049 & 0.059 & 0.036 & 0.036 &0.036& 0.037 & 0.044 & 0.052& \\ 
  1400 & 0.038 & 0.037 &0.038& 0.038 & 0.050 & 0.066 & 0.038 & 0.035 &0.035& 0.037 & 0.048 & 0.061& \\ 
  1800 & 0.042 & 0.037 &0.039& 0.038 & 0.053 & 0.073 & 0.041 & 0.037 &0.037& 0.038 & 0.050 & 0.066& \\  \hline
  200 & 0.036 & 0.036 & 0.033&0.036 & 0.039 & 0.039 & 0.035 & 0.035 & 0.033&0.035 & 0.038 & 0.035& \parbox[t]{2mm}{\multirow{5}{*}{\rotatebox[origin=c]{-90}{Control-2}}} \\ 
  600 & 0.036 & 0.034 & 0.035&0.034 & 0.043 & 0.048 & 0.036 & 0.035 & 0.036&0.035 & 0.044 & 0.044& \\ 
  1000 & 0.035 & 0.034 &0.034& 0.034 & 0.048 & 0.051 & 0.035 & 0.035 &0.035& 0.035 & 0.045 & 0.047& \\ 
  1400 & 0.038 & 0.037 &0.035& 0.038 & 0.052 & 0.056 & 0.037 & 0.037 &0.035& 0.038 & 0.047 & 0.053 &\\ 
  1800 & 0.036 & 0.035 &0.036& 0.035 & 0.052 & 0.059 & 0.036 & 0.035 &0.034& 0.035 & 0.050 & 0.054& \\  \hline
  200 & 0.045 & 0.044 & 0.044&0.044 & 0.043 & 0.044 & 0.079 & 0.079 & 0.079&0.078 & 0.070 & 0.077  &  \parbox[t]{2mm}{\multirow{5}{*}{\rotatebox[origin=c]{-90}{IV}}}\\ 
  600 & 0.046 & 0.045 & 0.045&0.046 & 0.045 & 0.047 & 0.091 & 0.086 & 0.087&0.086 & 0.073 & 0.098 &\\ 
  1000 & 0.043 & 0.042 &0.042& 0.042 & 0.044 & 0.044 & 0.093 & 0.090 &0.091& 0.090 & 0.073 & 0.208& \\ 
  1400 & 0.046 & 0.046 &0.045& 0.046 & 0.047 & 0.047 & 0.096 & 0.095 &0.095& 0.094 & 0.078 & 0.380& \\ 
  1800 & 0.046 & 0.046 &0.044& 0.046 & 0.048 & 0.047 & 0.098 & 0.102 &0.096& 0.100 & 0.080 & 0.895& \\ 
  \hline
  \textbf{p}&\multicolumn{12}{c|}{Standard Deviation}&\\
  \hline
  200 & 0.042 & 0.041 & 0.043&0.041 & 0.047 & 0.044 & 0.044 & 0.045 & 0.043&0.045 & 0.048 & 0.046 & \parbox[t]{2mm}{\multirow{5}{*}{\rotatebox[origin=c]{-90}{Control-1}}}\\ 
  600 & 0.040 & 0.040 & 0.045&0.040 & 0.052 & 0.046 & 0.043 & 0.043 & 0.044&0.043 & 0.051 & 0.046& \\ 
  1000 & 0.045 & 0.044 &0.044& 0.044 & 0.061 & 0.056 & 0.044 & 0.044 &0.042& 0.044 & 0.056 & 0.050& \\ 
  1400 & 0.045 & 0.044 &0.045& 0.045 & 0.064 & 0.053 & 0.044 & 0.042 &0.041& 0.043 & 0.061 & 0.055& \\ 
  1800 & 0.047 & 0.044 &0.044& 0.045 & 0.068 & 0.059 & 0.044 & 0.043 &0.044& 0.044 & 0.063 & 0.053& \\  \hline
  200 & 0.043 & 0.043 & 0.043&0.043 & 0.048 & 0.045 & 0.044 & 0.044 & 0.043&0.044 & 0.049 & 0.045 & \parbox[t]{2mm}{\multirow{5}{*}{\rotatebox[origin=c]{-90}{Control-2}}}\\ 
  600 & 0.044 & 0.044 & 0.044&0.044 & 0.054 & 0.054 & 0.044 & 0.044 & 0.044&0.044 & 0.055 & 0.050 &\\ 
  1000 & 0.044 & 0.043 &0.043& 0.043 & 0.059 & 0.054 & 0.044 & 0.043 &0.043& 0.044 & 0.057 & 0.051& \\ 
  1400 & 0.048 & 0.046 &0.044& 0.047 & 0.065 & 0.060 & 0.046 & 0.046 &0.044& 0.047 & 0.060 & 0.054 &\\ 
  1800 & 0.045 & 0.043 &0.044& 0.043 & 0.063 & 0.058 & 0.046 & 0.045 &0.043& 0.045 & 0.062 & 0.054& \\  \hline
  200 & 0.056 & 0.054 & 0.054&0.055 & 0.052 & 0.055 & 0.097 & 0.095 & 0.094&0.093 & 0.080 & 0.093  &  \parbox[t]{2mm}{\multirow{5}{*}{\rotatebox[origin=c]{-90}{IV}}}\\ 
  600 & 0.057 & 0.057 & 0.057&0.057 & 0.053 & 0.056 & 0.112 & 0.105 & 0.106&0.105 & 0.076 & 0.195 &\\ 
  1000 & 0.053 & 0.053 &0.052& 0.053 & 0.048 & 0.052 & 0.114 & 0.106 &0.107& 0.106 & 0.072 & 0.782 &\\ 
  1400 & 0.057 & 0.058 &0.057& 0.057 & 0.051 & 0.057 & 0.119 & 0.115 &0.111& 0.114 & 0.079 & 2.407& \\ 
  1800 & 0.057 & 0.058 &0.056& 0.058 & 0.052 & 0.057 & 0.117 & 0.119 &0.113& 0.117 & 0.077 & 5.709 &\\ 
        \hline
  \textbf{p}&\multicolumn{12}{c|}{Rejection Rate}&\\
  \hline
  200 & 0.056 & 0.058 & 0.064&0.056 & 0.048 & 0.070 & 0.064 & 0.072 & 0.066&0.072 & 0.064 & 0.072 & \parbox[t]{2mm}{\multirow{5}{*}{\rotatebox[origin=c]{-90}{Control-1}}}\\ 
  600 & 0.052 & 0.050 & 0.088&0.052 & 0.032 & 0.184 & 0.062 & 0.072 & 0.060&0.074 & 0.046 & 0.126 &\\ 
  1000 & 0.082 & 0.076 &0.072& 0.072 & 0.062 & 0.268 & 0.064 & 0.074 &0.058& 0.072 & 0.056 & 0.186& \\ 
  1400 & 0.080 & 0.076 &0.086& 0.084 & 0.052 & 0.314 & 0.074 & 0.064 &0.062& 0.068 & 0.060 & 0.248& \\ 
  1800 & 0.110 & 0.078 &0.072& 0.086 & 0.060 & 0.358 & 0.104 & 0.082 &0.070& 0.098 & 0.044 & 0.312& \\  \hline
  200 & 0.066 & 0.056 & 0.060&0.050 & 0.048 & 0.076 & 0.070 & 0.074 & 0.060&0.068 & 0.074 & 0.068 & \parbox[t]{2mm}{\multirow{5}{*}{\rotatebox[origin=c]{-90}{Control-2}}}\\ 
  600 & 0.062 & 0.064 & 0.060&0.064 & 0.062 & 0.176 & 0.074 & 0.068 & 0.054&0.060 & 0.054 & 0.124& \\ 
  1000 & 0.080 & 0.064 &0.056& 0.066 & 0.058 & 0.210 & 0.062 & 0.058 &0.052& 0.058 & 0.070 & 0.148& \\ 
  1400 & 0.088 & 0.068 &0.076& 0.084 & 0.084 & 0.206 & 0.088 & 0.074 &0.058& 0.086 & 0.064 & 0.224& \\ 
  1800 & 0.072 & 0.044 &0.066& 0.056 & 0.052 & 0.238 & 0.082 & 0.066 &0.054& 0.074 & 0.058 & 0.226 &\\  \hline
  200 & 0.046 & 0.042 & 0.042&0.044 & 0.052 & 0.048 & 0.070 & 0.070 & 0.062&0.072 & 0.088 & 0.074 &  \parbox[t]{2mm}{\multirow{5}{*}{\rotatebox[origin=c]{-90}{IV}}} \\ 
  600 & 0.066 & 0.066 & 0.046&0.072 & 0.110 & 0.068 & 0.072 & 0.074 & 0.056&0.082 & 0.122 & 0.084 &\\ 
  1000 & 0.044 & 0.052 &0.048& 0.052 & 0.062 & 0.046 & 0.068 & 0.054 &0.072& 0.054 & 0.116 & 0.052 &\\ 
  1400 & 0.066 & 0.072 &0.064& 0.068 & 0.090 & 0.072 & 0.070 & 0.074 &0.076& 0.078 & 0.126 & 0.066& \\ 
  1800 & 0.066 & 0.066 &0.052& 0.078 & 0.106 & 0.078 & 0.062 & 0.066 &0.078& 0.064 & 0.142 & 0.060& \\ 
  \hline
  \end{tabular}}
  \caption{Simulation results for $n=600$ and $SNR=1$.} 
\label{simresults_n600_part}
\end{table}

In addition, we provide the average CPU time of the proposed methods for a single iteration in Tables \ref{simresults_n400} -- \ref{simresults_n800_SNR3} in the Appendix \ref{app:add_sim}. These results refer to an implementation of all methods in Python to make the different computation times of the applied methods comparable. \footnote{We are using the Python version 3.6.5. The operating system is Linux/Ubuntu. The hardware is Intel(R) Xeon (R) CPU E5-2650 v3 @2.30GHz.} We provide the results for the setting with $n=600$ and $SNR=1$ again in Figure \ref{time_n600} below. It is clearly visible that $L_2$-Boosting has a computational advantage over Lasso. For example, in the sparse setting (\textit{Control-1}) for i.i.d. covariates with $p=1800$, post-Lasso requires a CPU time of $0.213$ seconds per iteration, whereas post-$L_2$-Boosting (\texttt{pBA}) takes only $0.096$ seconds. On average, the CPU time for Lasso is more than twice that for post-L2-Boosting. For an overview of how the computation time increases with sample size $n$, we refer to Figure \ref{time_all} in the appendix.
Obviously, the methods with cross-validated penalty are significantly slower than those with a theoretically justified tuning parameter, cf.~Appendix \ref{app:add_sim}. In summary, we conclude that $L_2$-Boosting with a theoretically justified stopping criterion is the most efficient method with respect to computation time, and it also provides the best estimation accuracy in our simulation study.
 \begin{figure}[H]
\centering
\includegraphics[width=\linewidth]{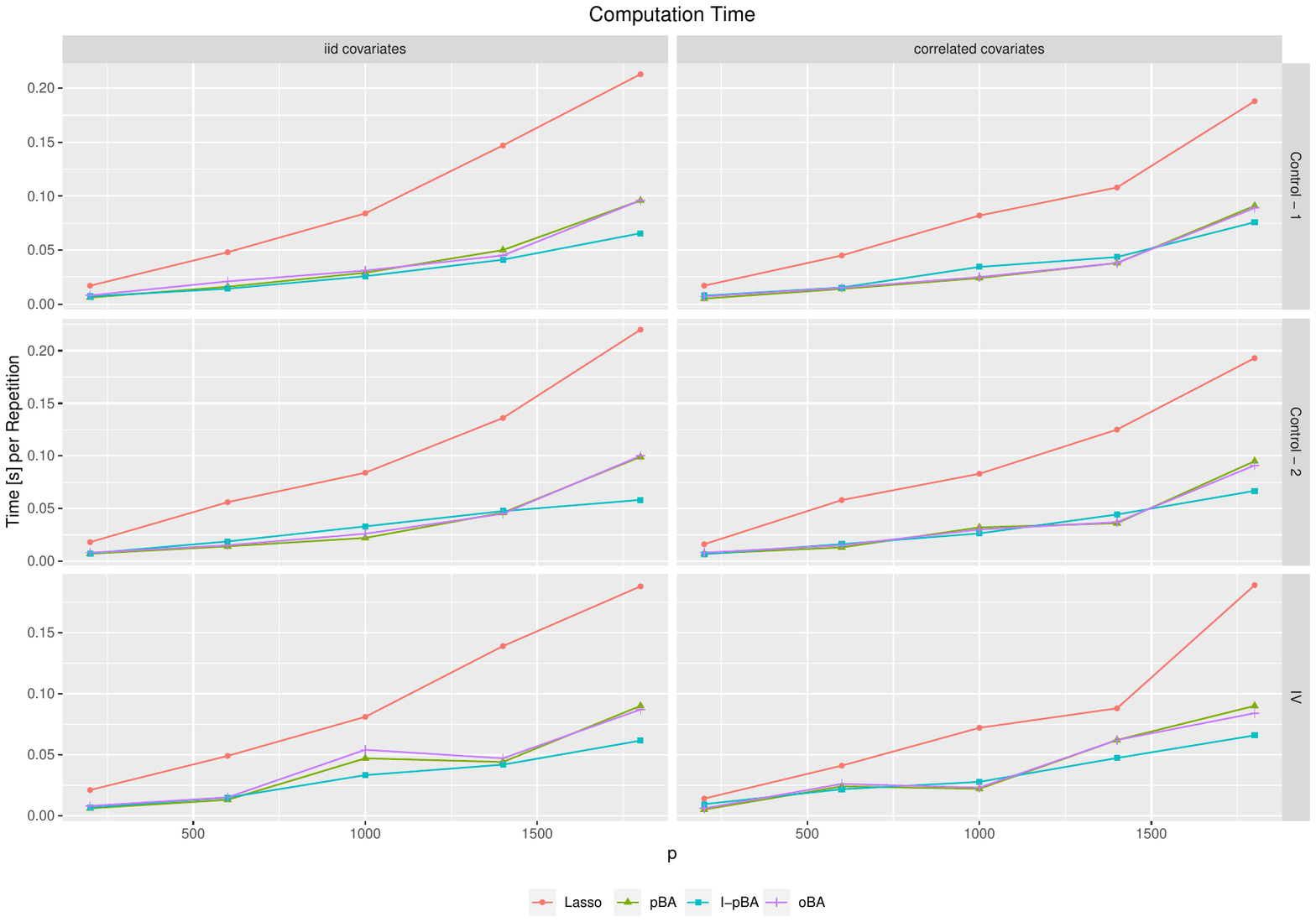}
\caption{Computation Time for $n=600$ and $SNR=1$.}
\label{time_n600}
\end{figure}

\section{Application: Empirical Analysis of Bank Mergers}\label{sec6}
In this section, we apply $L_2$-Boosting for instrumental variable estimation in a setting with many instruments. We build on the empirical study in \cite{levine2020bank} and show how our methodology can improve IV estimation in high-dimensional settings.

\subsection{Brief Introduction}

We use a data set from \cite{levine2020bank}, who study how the pre-merger overlap of bank branch networks in the US can affect the post-merger stock returns of acquirer banks over the period 1984 to 1995. The target variable of interest is the cumulative abnormal returns (CAR) of the acquirer bank within a window of five days around the announcement of the acquisition, denoted CAR(-2,2).\footnote{Acquirer CAR(-2, 2) is the 5-day CAR (cumulative abnormal
returns) during the window (-2, +2), where day 0 is the announcement date of the acquisition. Abnormal returns are defined by using the difference between actual and projected returns. For a definition of the projected returns we refer to \cite{levine2020bank}.} 

We are interested in estimating the causal effect of the overlap of bank branch networks on the outcome of the acquisition, i.e., CAR(-2,2). This overlap is measured by three variables, namely the \textit{Overlap},  \textit{Correlation}, and \textit{Cosine Distance}. The \textit{Overlap} variable is measured as the quotient of the  number of states in which both banks are present and the total number of states in which at least one bank is present. \textit{Correlation} and \textit{Cosine Distance} are derived from the vector for each bank, which denotes if they are present in each state.  For details on the construction of the three measurements of bank branch network overlap, we refer to \cite{levine2020bank}. Additionally, 13 variables are available as controls.\footnote{The covariates are: Acquirer run-up (-200, -11), Acquirer Net Income, Acquirer Tobin’s Q, Deal size adjusted by Acquirer Assets, Cash deal dummy, Stock Deal dummy, Attitude dummy, Target public dummy, Percentage of shares acquired, Acquirer Total Assets, Acquirer/Target Assets Ratio, a dummy for whether the acquiring banks acquired another target during the last three years, and number of offers that the target received. The regressions also contain Year fixed effects and Acquirer Bank State*Target Bank State fixed effects.}

The endogeneity problem arises because measures of the overlap of bank branch networks  and bank merger decisions and outcomes may be affected by unobservable hidden factors. In many states in the US, banks were not allowed to set up branches in other states. In the 1980s and 1990s, states gradually relaxed  this restriction, which was largely exogenous to the individual banks (\cite{Jayaratne}). 
Therefore, a connection forms when a state allows another state’s banks to open branches, and these connections generate a directed spatial state network between the 50 states in the United States and Washington D.C.. As a result, interstate deregulation events generate  exogenous changes of the state network, which allows information from such state network to serve as instruments to the endogenous bank network, e.g., the three overlap measures mentioned above.


It is common to utilize neighborhood structure of nodes in a network to serve as instruments to deal with endogeneity in the literature, e.g., \cite{Chandrasekhar}. Specifically, for any pairs of nodes $S_1,S_2$, we can define $d(S_1,S_2)$ as the minimum number of edges that link $S_1$ to $S_2$. For any $k\in \mathbb{Z}^+$, we can define the $k^{th}$ order neighbor of a node $S_1$, denoted as a set $N_k(S_1)$, as the collection of those nodes $S_2$ such that $d(S_1,S_2) = k$. 

Note that the state network has a spatial structure. Consider the case that a bank (the acquirer) in $S_1$ acquires a bank in $S_2$. Consider measuring ``overlap" in the state network for two states $S_1,S_0$ subjective to the target bank's state, $S_2$. We can  construct $k^{th}$ order (geographical) distance-weighted overlap between two states $S_1$ and $S_0$, subject to $S_2$, denoted as $k\textrm{-}NN(S_0,S_1,S_2)$, as follows:
\begin{equation}\label{eq:knn1}
    k\textrm{-}NN(S_0,S_1,S_2):=\frac{\sum_{S\in \mathcal{S}} 1(S\in N_k(S_1) \cap N_k(S_0)) }{\sum_{S\in \mathcal{S}} 1(S\in N_k(S_1)\cup N_k(S_0)) }GDR(S_0,S_1,S_2),
\end{equation}
where $GDR(S_0,S_1,S_2)$ is a function that measures ``geographic distance ratio", and $\mathcal{S}$ is the collection of all states and Washington D.C., which has a cardinality of $51$. In \cite{levine2020bank}, the authors define three different ways to construct geographic distance ratio (GDR) and $k$-$NN$. For details, see Appendix \ref{details_application}.

Specifically, in \eqref{eq:knn1}, we call $1\textrm{-}NN$ as FNN, and $2\textrm{-}NN$ as SNN, referring to ``first" and ``second order neighbor's distance weighted network overlap", respectively. Note that for each acquisition, given $S_1,S_2$, there exists $51$ possible $S_0\in \mathcal{S}$ and three ways to measure FNN, leading to $51\times 3=153$ FNN instruments.

Using the state network, \cite{levine2020bank} construct three IVs based on SNNs, namely weighted \textit{Overlap}, \textit{Correlation  Coefficient}, and \textit{Cosine Distance} of states that allowed market  entry.\footnote{Details on the construction can be found in the online appendix of \cite{levine2020bank}.}  We instead use the $153$ FNN variables as potential IVs. Because we have $442$ observations, this leads to a high $p/n$ ratio, only a small set of which, however, might be relevant.\footnote{Note that the number of observations differ slightly to the original study because we corrected a minor data mistake in the sample. This mistake is confirmed after communication with one of the authors.} In fact, our later analysis shows that there are only four to six FNNs with non-zero coefficients after applying iterated post- or orthogonal $L_2$-Boosting, meaning that most FNNs are filtered out during the selection step.\footnote{Note that in our estimation with $L_2$-Boosting, we do not include the original, constructed IVs that \cite{levine2020bank} use. However, if we include them in the selection process, the results do not change – indeed, these are still better than those obtained using the combination of simple instrumental variables.}

The FNNs appear to be much weaker than SNNs because none of the FNNs are found to be significantly correlated to the explanatory variable in the first stage. For this reason, \cite{levine2020bank} use SNNs instead of FNNs. However, discarding the FNNs leads to a loss of efficiency. The advantage of $L_2$-Boosting algorithm is that it utilizes the useful information of FNNs and selects a combination of FNNs to form IVs that may be even better than SNNs.

\subsection{Results}
We apply iterated post-$L_2$-Boosting and orthogonal $L_2$-Boosting on the bank merger data, conduct inference as described in Section \ref{IVsetting}, and present the estimates in Table \ref{application_results}. We control for all firm-year characteristics, year fixed effects, and acquirer-state times target-state fixed effects as was done in the original study. The numbers in parentheses are standard errors.

Columns 1-3 in Table \ref{application_results} give the estimated coefficients and standard errors using original IVs (SNNs) of \cite{levine2020bank}.\footnote{Note that the estimates differ slightly to the original study since we corrected a minor data mistake in the sample.} 
Columns 4-6 give the  estimates of the coefficients and standard errors using iterated post-$L_2$-Boosting (I-pBA). Rows 4-6 in Table \ref{application_selection} give the FNNs  that are ultimately selected for use. For example, as shown in Row 4, six FNNs ($FNN1_{DC}$, $FNN1_{IA}$, $FNN1_{VA}$, $FNN2_{MD}$, $FNN2_{NJ}$, $FNN2_{KS}$) are selected to jointly instrument for \textit{Overlap} in our estimation procedure. These six instruments provide a smaller standard error (3.53) than the original IVs of \cite{levine2020bank} (5.27) -- indicating the selection of more efficient IVs. The standard errors using the IVs selected from I-pBA for \textit{Correlation Coefficient} and \textit{Cosine Distance} are also smaller than those obtained using the  original IVs. This shows that the estimation is more efficient when using our selection procedure. Moreover, all estimated coefficients using I-pBA are also statistically significant.

Orthogonal $L_2$-Boosting also works well on this sample. Columns 7-9 present the estimates of the coefficients and standard errors  using orthogonal $L_2$-Boosting. Rows 7-9 in Table \ref{application_selection} present the FNNs that are selected. For example, as shown in Row 7, six IVs ($FNN1_{AR}$, $FNN1_{DC}$, $FNN1_{IA}$, $FNN1_{VA}$, $FNN2_{MD}$, $FNN2_{NJ}$) are selected to jointly instrument for \textit{Overlap} and produce a smaller standard error (3.52) than that produced by the SNNs (5.27). The standard errors using the combined FNNs selected from orthogonal $L_2$-Boosting for \textit{Correlation Coefficient} and \textit{Cosine Distance} are also smaller than the SNNs in \cite{levine2020bank}. 

In the last three Columns 10-12 in Table \ref{application_results}, we present the estimates using post-Lasso for a comparison. As shown in the bottom row of the table, the computing time for post-Lasso is higher than that for post-$L_2$-Boosting.

\begin{landscape}
	\begin{table}
		\caption{Second stage estimates using iterated post- and orthogonal $L_2$-Boosting.}
		\label{application_results}
	\begin{scriptsize}
\resizebox{2.35\textwidth}{!}{
	\begin{tabularx}{38cm}{lcccccccccccc}
		\cline{1-13}
	\multirow{3}{*}{} & (1) 	& (2)	&	(3)	&	(4) &	(5)	&	(6) 	&	(7)	&	(8)	&	(9)&	(10)&	(11)	&	(12) \\
			 &  \multicolumn{12}{c}{Dependent Variable: Acquirer CAR (-2, +2)}     \\
			\cline{2-13}	& \multicolumn{3}{c}{Original IV}	& \multicolumn{3}{c}{iterated post-$L_2$-Boosting}	& \multicolumn{3}{c}{orthogonal $L_2$-Boosting}	& \multicolumn{3}{c}{post-Lasso}	\\
			\cline{1-13} &&&&&&&&&&&& \\
			Overlap & 17.18***	&	&	& 8.827**	& 	&	& 8.753**	&	&	& 3.575	&	& \\
			& (5.27)	&	&	& (3.53)	&	&	& (3.52)	&	&	& (4.39)	&	& \\
			Correlation Coefficient &	& 12.93***	&	&	& 11.50***	&	&	& 12.34***	&	&	& 5.102	& \\
			&							& (4.86)	&	&	& (4.44)	&	&	& (4.68)	&	&	& (4.73)	& \\
			Cosine Distance 			&	& 	& -17.66***	&	& 	& -13.64***	&	& 	& -8.386	&	& 	& -5.364 \\
										&	&	& (5.63)	&	&	& (5.28)	&	&	& (5.12)  & &  & (4.84)  \\
			&	&	&	&	&	&	&	&	&	&	&	& \\
			Firm-year Controls & Yes	& Yes	& Yes	& Yes	& Yes	& Yes	& Yes	& Yes	& Yes	& Yes	& Yes	& Yes \\
			Year Fixed Effects & Yes	& Yes	& Yes	& Yes	& Yes	& Yes	& Yes	& Yes	& Yes	& Yes	& Yes	& Yes \\
			Acquirer-State $\times$ Target-State Fixed Effects & Yes	& Yes	& Yes	& Yes	& Yes	& Yes	& Yes	& Yes	& Yes	& Yes	& Yes	& Yes \\
			&&&&&&&&&&&& \\
			Observations & 442	&442	&442	&442	&442	&442	&442	&442	&442	&442	&442	&442 \\
			R-squared 	& 0.088	& 0.298	& 0.144	& 0.404	& 0.338	& 0.288	& 0.406	& 0.315	& 0.410	& 0.466	& 0.446	& 0.446 \\
			Algorithm Execution Time (ms) & & & & 3.313 & 2.381 & 2.426 & 3.658 & 3.866 & 4.227 & 19.928 & 5.260 & 5.298 \\
			
			\cline{1-13}
	\end{tabularx}
	}	
\end{scriptsize}
\end{table}

\begin{table}
	\caption{Selected instrumental variables that are used in Table \ref{application_results}.}
	\label{application_selection}
	\begin{scriptsize}
		\begin{tabularx}{.7\textwidth}{cl}
			\hline
			\textbf{Column} 	&	\textbf{Instrumental Variables} \\
			\hline
			(1)				&  Original SNN IV for Overlap \\
			(2)				&  Original SNN IV for Correlation Coefficient \\
			(3)				&  Original SNN IV for Cosine Distance \\
			(4)				&  $FNN1_{DC}, FNN1_{IA}, FNN1_{VA}, FNN2_{MD}, FNN2_{NJ}, FNN2_{KS}$ \\
			(5)				&  $FNN1_{AR}, FNN1_{VA}, FNN2_{MD}, FNN3_{KS}$ \\
			(6)				&  $FNN1_{AR}, FNN2_{TX}, FNN3_{KS}, FNN3_{NY}, FNN3_{RI}$\\
			(7) 			&  $FNN1_{AR}, FNN1_{DC}, FNN1_{IA}, FNN1_{VA}, FNN2_{MD}, FNN2_{NJ}$ \\
			(8) 			&  $FNN1_{AR}, FNN1_{IA}, FNN1_{VA}, FNN2_{MD}, FNN3_{KS}$\\
			(9) 			&  $FNN1_{AR}, FNN1_{IA}, FNN2_{OK}, FNN2_{TX}, FNN3_{KS}, FNN3_{RI}$\\
			(10) 			&  $FNN2_{MI}, FNN2_{NY}, FNN3_{AK}, FNN3_{LA}$ \\
			(11) 			&  $FNN2_{MI}, FNN2_{NY}, FNN3_{LA}$\\
			(12) 			&  $FNN2_{MI}, FNN2_{NY,} FNN2_{PA}, FNN3_{LA}$\\
			\hline \hline
		\end{tabularx}
	\end{scriptsize}
\end{table}
\end{landscape}

\section{Conclusion}\label{sec7}
In this paper, we apply $L_2$-Boosting, namely the iterated post- and orthogonal versions, for estimation and inference of treatment effects in the setting of many controls and many instruments. We derive uniformly valid results for the asymptotic distribution of estimated treatment effects.
We use the framework of orthogonalized moment conditions introduced by Belloni, Chernozhukov, Hansen, and coauthors in a series of papers to derive these results.
We also provide new results on the rate of convergence of iterated post- and orthogonal $L_2$-Boosting, which are needed as an ingredient of our approach, but might also be of independent interest.

To derive these rates, we do not require the beta-min condition, and we rely only on approximate sparsity, which is a substantial  improvement on earlier work. In an extensive simulation study, we show that our proposed method works well and has a computational advantage over Lasso. Finally, we analyze how the pre-merger overlap of bank branch networks in the US can affect the post-merger stock returns of acquirer banks. Our results show that IV selection with our procedure gives more efficient estimates than standard approaches and is capable of selecting the relevant instruments from among a large set of potential instruments.

\newpage
\appendix


\section{Proofs of Theoretical Results.}\label{app:proofs}
\noindent
Throughout this section, we assume that Assumptions \ref{assump:AS} - \ref{assump:tail} hold. 
Define $$T^m:=\cap_{t=0}^{m-1} \{j^t\}$$ as the set of selected variables at the beginning of the   $m^{th}$ iteration.
We require the following lemma for proving our main results. For any set $T\subset\{1,2,...,p\}$, define $\mathcal{P}[X_T]$ as the projection operator
$$X_{T}(X_{T}^\intercal X_{T})^{-1}X_{T}^\intercal.$$

\begin{lemma}[Bounds on Residuals]\label{lemma:bound-res}
Given the Assumptions \ref{assump:AS}-\ref{assump:tail}, for any $T\subset \{1,2,...,p\}$ such that $|T|\leq M_n$, we have:
\begin{align}
    \|\mathcal{P}[X_T]\varepsilon\|_n^2&\leq \frac{1}{c_\phi}|T| \lambda_n^2,\\
    \|\mathcal{P}[X_T](r+\varepsilon)\|^2_n&\leq 2\left({\frac{C_r s\log(2p/\alpha) + 4\sigma^2/c_{\phi}\log(2p/\alpha)|T|}{n}}\right).
\end{align}
\end{lemma}
\begin{proof}
We know that 
\begin{equation}\label{eq:p-ep}
     \|\mathcal{P}[X_T]\varepsilon\|_n^2 = \frac{1}{n}(\varepsilon^\intercal X_T)\left(\frac{X_T^\intercal X_T}{n}\right)^{-1} \frac{1}{n}(X_T^\intercal \varepsilon).
\end{equation}
By Assumption \ref{assump:tail}, we have that 
\begin{equation}\label{eq:X-ep}
    \frac{1}{n}\|X_T^\intercal \varepsilon\|_n^2 = \sum_{j\in T} <X_j,\varepsilon>_n^2\leq |T|\lambda_n^2.  
\end{equation}
By Assumption \ref{assump:SE}, since $|T|\leq M_n$, we have that
\begin{equation}\label{eq:eigen-min}
    \phi_{max}\left(\frac{X_T^\intercal X_T}{n}\right)^{-1}\leq \frac{1}{\phi_{min}\left(\frac{X_T^\intercal X_T}{n}\right)}\leq \frac{1}{c_\phi}.
\end{equation}
Plugging in \eqref{eq:X-ep} and \eqref{eq:eigen-min} to \eqref{eq:p-ep}, we have that:
\begin{equation*}
     \|\mathcal{P}[X_T]\varepsilon\|_n^2\leq \frac{1}{c_\phi}|T| \lambda_n^2.
\end{equation*}
Consequently, by statement (ii) of Assumption \ref{assump:AS}, we have that:
\begin{align*}
    \|\mathcal{P}[X_T](r+\varepsilon)\|^2_n &\leq 2(\|\mathcal{P}[X_T]r\|^2_n+\|\mathcal{P}[X_T]\varepsilon\|^2_n)\\
    &\leq 2(\|r\|_n^2+\|\mathcal{P}[X_T]\varepsilon\|^2_n)\\
    &\leq 2\left(\frac{C_r s\log(2p/\alpha) + 4\sigma^2/c_{\phi}\log(2p/\alpha)|T|}{n}\right).
\end{align*}
\end{proof}
From now on, denote $V^m:=X(\beta-\beta^m)$ as the vector of errors between the true model and the estimated model at period $m$. By definition, it is easy to see that for all $m\geq 0$, we have that:
\begin{equation*}
    U^m = Y-X\beta^m = V^m+r+\varepsilon.
\end{equation*}

\begin{proof}[Proof of Theorem \ref{theo:PBA}]
If the number of selected variables is less than $M_n$, it is easy to see that statement (\ref{eq1pba}) implies statement (\ref{eq2pba}) due to Assumption \ref{assump:SE}.
Therefore, it suffices to prove that (a) statement (\ref{eq1pba}) is true and (b) the number of variables selected in $T^{m^*}$ is less than or equal to $M_n$, which is implied by our assumption that $m^*\leq M_n$. WLOG., we can assume that $ \frac{C_{m^*}s\log(n)}{M} $ is an integer so that $k=\lceil \frac{C_{m^*}s\log(n)}{M}\rceil = \frac{C_{m^*}s\log(n)}{M}$. As a result, $m^*=kM+1 = C_{m^*}s\log(n)+1$.

By assumption, define $m_k:=$ the $k^{th}$ value of $m$ such that  $l_m=1$. Further, define $m_0=0$ and the $k^{th}$ epoch as the periods of $m_{k-1}+1,...,m_k$, for $k\geq 1$.  
At the beginning of the $k^{th}$ epoch, the corresponding vector of residuals $U^{m_{k-1}+1}$ is 
\begin{equation}\label{eq:epoch}
    U^{m_{k-1}+1}:=\left(I-\mathcal{P}[X_{{T}^{m_{k-1}+1}}]\right) Y.
\end{equation}
Recall that $T_0:=supp(\beta)$ is the support of the true parameter $\beta$.
Define $\widetilde{T}^m:=T_0\backslash T^m$ for any $m\geq 0$. Assume that $m_{k-1}+1\leq m_{k}+1\leq m^* \leq M_n$.\\ \\
If $\widetilde{T}^{m_{k-1}+1}=\emptyset$, then all the variables in $T$ are already selected by the algorithm at $m_{k-1}+1$. As a result, $\widetilde{T}^{m^*}$ must also be $\emptyset$.  Consequently, by Lemma \ref{lemma:bound-res}, we have that:
\begin{align}
    \|V^{m^*}\|_n^2 &= \| \mathcal{P}[X_{{T}^{m^*}}](r+\varepsilon)\|_n^2 \nonumber\\ 
    &\leq \frac{(2 C_r s\log(2p/\alpha) + 8 m^* \sigma^2/c_\phi\log(2p/\alpha) ) }{ n} \nonumber \\ 
    &= \frac{(2 C_r\log(2p/\alpha)  + 8 \sigma^2 C_{m^*}/c_\phi\log(2p/\alpha)\log(n)) s}{ n},\label{eq:v-bound}
\end{align}
which concludes the result.\\ \\
Now, suppose that $\widetilde{T}^{m_{k-1}+1}\neq \emptyset$. By construction of Algorithm \ref{algo:iterated-pba}, we have that:
\begin{equation}\label{eq:v-exp}
    V^{m_{k-1}+1} = X\beta-X\beta^{m_{k-1}+1} = (I-\mathcal{P}[X_{{T}^{m_{k-1}+1}}])X\beta - \mathcal{P}[X_{{T}^{m_{k-1}+1}}](r+\varepsilon).
\end{equation}
Define $\widetilde{V}^{m_{k-1}+1} = (I-\mathcal{P}[X_{{T}^{m_{k-1}+1}}])X\beta$, which differs from $V^{{m_{k-1}+1}}$ by $$-\mathcal{P}[X_{{T}^{m_{k-1}+1}}](r+\varepsilon).$$
Note that $\|\widetilde{V}^{m_{k-1}+1}\|_n^2$ is decreasing in $k$ as $T^{m_{k-1}+1}\subset T^{m_{k}+1}$ for all $k\geq 0$. If \begin{equation}\|\widetilde{V}^{m_{k-1}+1}\|_n^2\leq \frac{C_V s\log (2p/\alpha)\log (n)}{n}\end{equation} holds for some generic constant $C_V>1$ and $m_{k-1}+1\leq m^*$, by Lemma \ref{lemma:bound-res}, we have that: 
\begin{align}
    \|V^{m^*}\|_n^2 &= \|\widetilde{V}^{m^*}+\mathcal{P}[X_{{T}^{m_{k-1}+1}}](r+\varepsilon)\|_n^2 \nonumber\\
    &\leq 2\|\widetilde{V}^{m^*}\|_n^2 + 2\|\mathcal{P}[X_{{T}^{m_{k-1}+1}}](r+\varepsilon)\|_n^2 \nonumber\\ 
    &\leq 2\|\widetilde{V}^{{m_{k-1}+1}}\|_n^2 + 2\left(\frac{C_r s\log(2p/\alpha) + 4\sigma^2 m^*/c_{\phi}\log(2p/\alpha)}{n}\right)\nonumber \\
    &\leq \frac{(2 C_r\log(2p/\alpha)  + (2C_V + 8 \sigma^2 C_{m^*}/c_\phi)\log(2p/\alpha)\log(n)) s}{n},\label{eq:C-v-bound}
\end{align}
which concludes our result.\\ \\
From now on, we assume that \begin{equation}\label{eq:V-bound-assump}\|\widetilde{V}^{m_{k-1}+1}\|_n^2 > \frac{C_V s\log (2p/\alpha)\log (n)}{n}\end{equation} for all $k$ such that $ m_{k-1}+1\leq m_{k}+1\leq m^*$. Our strategy is to show that this is impossible if $C_{m^*}$ is large enough and $C_V$ is properly chosen.
By construction of the algorithm, we know that for any $m=m_{k-1}+1,...,m_{k}+1$, we have that \footnote{ The inequality \eqref{eq:U-descend} holds as equality for $m=m_{k-1}+1,...,m_k$ by definition of $\mathtt{pBA}$  steps.}: 
\begin{equation}\label{eq:U-descend}
    \|U^m\|_n^2 \leq   \|U^{m_{k-1}+1}\|_n^2 - \sum_{t=m_{k-1}+1}^{m-1} \gamma_{j^t}^2.
\end{equation}
Define $\delta_{m,k}:=\sum_{t=m_{k-1}+1}^{m-1}\gamma_{j^t}e_{j^t}$ as the change of the parameter $\beta^m - \beta^{m_{k-1}+1}$, and define $T_{m,k}:=\cup_{t=m_{k-1}+1}^{{m-1}}\{j^t\}$. 
For $m=m_{k-1}+1,...,m_k+1$, denote \begin{align*}
    \mathcal{V}^m&:=(I-\mathcal{P}[X_{{T}^{m_{k-1}+1}}])X\beta-\sum_{t=m_{k-1}+1}^{m-1}\gamma_{j^t}X_{j^t} \\ 
    &= (I-\mathcal{P}[X_{{T}^{m_{k-1}+1}}])X_{\widetilde{T}^{m_{k-1}+1}}\beta_{\widetilde{T}^{m_{k-1}+1}} - X_{T_{m,k}}\delta_{m,k},
\end{align*}
with \begin{equation}\label{eq:V-T-1}\mathcal{V}^{m_{k-1}+1} = \widetilde{V}^{m_{k-1}+1}\end{equation}
and \begin{equation}\label{eq:V-T-2}(I-P[X_{T^{m_k+1}}])\mathcal{V}^{m_{k}+1} = \widetilde{V}^{m_{k}+1}.\end{equation}
\noindent
Next, we divide our discussion into two cases, Case ($A_1$) and Case ($A_2$).
\paragraph{\textbf{Case} ($A_1$)} For a fixed $\omega\in (0,1)$ such that \begin{equation}\label{eq:omega-def}
    \omega = \frac{c_\phi^2}{4MC_\phi},
\end{equation} assume that
$\sum_{t=m_{k-1}+1}^{{m_k}} \gamma_{j^t}^2 = \omega_m \|\widetilde{V}^{m_{k-1}+1}\|^2_n$ with $\omega_m\geq \omega$. 

Then, by \eqref{eq:U-descend},  we have that:
\begin{align}
    \|{U}^{m_k+1}\|_n^2 &
    \leq \|{U}^{{m}_{k-1}+1}\|_n^2 - \sum_{t=m_{k-1}+1}^{m_k} \gamma_{j^t}^2 \nonumber \\
    &\leq \|{U}^{{m}_{k-1}+1}\|_n^2-\omega_m \|\widetilde{V}^{m_{k-1}+1}\|_n^2 
   \leq \|{U}^{{m}_{k-1}+1}\|_n^2-\omega \|\widetilde{V}^{m_{k-1}+1}\|_n^2  \label{eq:U-m-k}.
\end{align}

\paragraph{\textbf{Case} ($A_2$)}
Assume $\sum_{t=m_{k-1}+1}^{m_k}\gamma_{j^t}^2 \leq \omega \|\widetilde{V}^{m_{k-1}+1}\|_n^2$. 
Consequently, for any $m=m_{k-1}+1,...,m_k+1$, we have that:
\begin{equation}\label{eq:delta-bound}
    \|\delta_{m,k}\|^2\leq {M} \sum_{t=m_{k-1}+1}^{m_k}\gamma_{j^t}^2 \leq M\omega \|\widetilde{V}^{m_{k-1}+1}\|_n^2\leq M\omega  \|X_{\widetilde{T}^{m_{k-1}+1}}\beta_{\widetilde{T}^{m_{k-1}+1}} \|_n^2. 
\end{equation}
Then, we have:
\begin{align}
    {} & <\mathcal{V}^m, X_{\widetilde{T}^{m_{k-1}+1}}\beta_{\widetilde{T}^{m_{k-1}+1}}>_n \nonumber
     \\
    {} & = < (I-\mathcal{P}[{X_{{T}^{m_{k-1}+1}}}])X_{\widetilde{T}^{m_{k-1}+1}}\beta_{\widetilde{T}^{m_{k-1}+1}} - X_{T_{m,k}}\delta_{m,k}, X_{\widetilde{T}^{m_{k-1}+1}}\beta_{\widetilde{T}^{m_{k-1}+1}}>_n \nonumber\\
    &\geq c_\phi\|X_{\widetilde{T}^{m_{k-1}+1}}\beta_{\widetilde{T}^{m_{k-1}+1}}\|^2_n - C_\phi^\frac{1}{2}\|X_{\widetilde{T}^{m_{k-1}+1}}\beta_{\widetilde{T}^{m_{k-1}+1}}\|_n^2 \| \delta_{m,k}\|\label{eq:cor-1}\\
    &\geq c_\phi\left(1- \frac{\sqrt{M\omega C_\phi}}{c_\phi}\right)\|X_{\widetilde{T}^{m_{k-1}+1}} \beta_{\widetilde{T}^{m_{k-1}+1}}\|_n^2 = \frac{c_\phi}{2}\|X_{\widetilde{T}^{m_{k-1}+1}} \beta_{\widetilde{T}^{m_{k-1}+1}}\|_n^2,\label{eq:cor-lowerbound}
\end{align}
where \eqref{eq:cor-1} follows from Assumption \ref{assump:SE} and \eqref{eq:cor-lowerbound} follows from \eqref{eq:delta-bound} and \eqref{eq:omega-def}.\\ \\
Therefore by definition of $\gamma_{j^m}$, for $m=m_{k-1}+1,...,m_k+1$, we have:
\begin{align}
    |\gamma_{j^m}| \|\beta_{\widetilde{T}^{m_{k-1}+1}} \|_1
    &\geq\ <U^m, X_{\widetilde{T}^{m_{k-1}+1}}\beta_{\widetilde{T}^{m_{k-1}+1}}>_n \nonumber
    \\
    &=\ <\mathcal{V}^m, X_{\widetilde{T}^{m_{k-1}+1}}\beta_{\widetilde{T}^{m_{k-1}+1}}>_n \nonumber\\
    & + <(I-\mathcal{P}[{X_{{T}^{m_{k-1}+1}}}]) (r+\varepsilon), X_{\widetilde{T}^{m_{k-1}+1}}\beta_{\widetilde{T}^{m_{k-1}+1}}>_n \nonumber\\
    &\geq \frac{c_\phi}{2}\|X_{\widetilde{T}^{m_{k-1}+1}} \beta_{\widetilde{T}^{m_{k-1}+1}}\|_n^2 - \sqrt{\frac{C_r s\log(2p/\alpha)}{n}}\|X_{\widetilde{T}^{m_{k-1}+1}}\beta_{\widetilde{T}^{m_{k-1}+1}}\|_n  \nonumber
    \\ & + <\varepsilon,X_{\widetilde{T}^{m_{k-1}+1}}\beta_{\widetilde{T}^{m_{k-1}+1}}>_n - < \mathcal{P}[{X_{{T}^{m_{k-1}+1}}}]\varepsilon,X_{\widetilde{T}^{m_{k-1}+1}}\beta_{\widetilde{T}^{m_{k-1}+1}}>_n \nonumber\\
    & \geq \left(\frac{c_\phi}{2}-\sqrt{\frac{C_r }{C_V\log(n)}}\right)\|X_{\widetilde{T}^{m_{k-1}+1}} \beta_{\widetilde{T}^{m_{k-1}+1}}\|_n^2 \nonumber \\
    & - \|\varepsilon^\intercal X_{\widetilde{T}^{m_{k-1}+1}}  \|_n  \|\beta_{\widetilde{T}^{m_{k-1}+1}}\| \label{eq:gamma-1}\\ & - \sqrt{\frac{{C_{m^*}s\log(n)}}{c_\phi}}\lambda_n \| X_{\widetilde{T}^{m_{k-1}+1}} \beta_{\widetilde{T}^{m_{k-1}+1}}\|_n \label{eq:gamma-2} \\
    {}&  \geq \left(\frac{c_\phi}{2}-\sqrt{\frac{C_r }{C_V\log(n)}} - 2\sigma \sqrt{\frac{C_{m^*}}{C_V c_\phi}}\right)\|X_{\widetilde{T}^{m_{k-1}+1}} \beta_{\widetilde{T}^{m_{k-1}+1}}\|_n^2 \label{eq:gamma-3}\\
    {}& - \sqrt{s}\lambda_n \sqrt{\frac{1}{c_{\phi}}} \|X_{\widetilde{T}^{m_{k-1}+1}} \beta_{\widetilde{T}^{m_{k-1}+1}}\|_n\label{eq:gamma-4}\\
    {} & \geq \kappa_{n,1} \|X_{\widetilde{T}^{m_{k-1}+1}} \beta_{\widetilde{T}^{m_{k-1}+1}}\|_n^2,\label{eq:kappa-n-1}
\end{align}
with 
\begin{equation}
    \kappa_{n,1} := \frac{c_\phi}{2}-\sqrt{\frac{C_r  }{C_V\log
    (n)}}  - 2\sigma \sqrt{\frac{C_{m^*}}{C_V c_\phi}}- 2\sigma\sqrt{\frac{1}{C_V c_\phi \log(n)}},
\end{equation}
where \eqref{eq:gamma-1} holds by Cauchy-Schwarz inequality, \eqref{eq:gamma-2} holds by Lemma \ref{lemma:bound-res} and \eqref{eq:C-v-bound}, \eqref{eq:gamma-3} holds by \eqref{eq:C-v-bound}, \eqref{eq:gamma-4} holds by Assumption \ref{assump:SE}, Assumption \ref{assump:tail} and $|\widetilde{T}^{m_{k-1}+1}|\leq s$, and \eqref{eq:kappa-n-1} holds by \eqref{eq:C-v-bound}.\\ \\
For fixed $C_V > \frac{64\sigma^2 C_{m^*}}{c_\phi^3}$ and $n$ large enough, we have that $\kappa_{n,1}>\frac{c_\phi}{4}$. Therefore, 
\begin{align}
    |\gamma_{j^m}|&\geq \frac{c_\phi}{4}\frac{\|X_{\widetilde{T}^{m_{k-1}+1}} \beta_{\widetilde{T}^{m_{k-1}+1}}\|_n^2}{\|\beta_{\widetilde{T}^{m_{k-1}+1}} \|_1}\\
    &\geq \frac{c_\phi}{4}\frac{\|X_{\widetilde{T}^{m_{k-1}+1}} \beta_{\widetilde{T}^{m_{k-1}+1}}\|_n^2}{\sqrt{s}\|\beta_{\widetilde{T}^{m_{k-1}+1}} \|}\\
    &\geq \frac{c_\phi}{4\sqrt{s C_\phi}} \|X_{\widetilde{T}^{m_{k-1}+1}} \beta_{\widetilde{T}^{m_{k-1}+1}}\|_n,
\end{align}
where the last inequality follows from Assumption \ref{assump:SE}.
Therefore, we have that:
\begin{align}
\sum_{m=m_{k-1}+1}^{m_k} |\gamma_{j^m}|^2
&\geq \frac{M c_\phi^2}{16C_\phi s} \|X_{\widetilde{T}^{m_{k-1}+1}} \beta_{\widetilde{T}^{m_{k-1}+1}}\|_n^2\nonumber\\
& \geq \frac{M c_\phi^2}{16C_\phi s}\|\widetilde{V}_{T^{m_{k-1}+1}}\|_n^2.
\end{align}
Consequently, we have
\begin{align}
    \|{U}^{m_{k}+1}\|_n^2 &\leq \|{U}^{m_{k-1}+1}\|_n^2 - \sum_{t=m_{k-1}+1}^m |\gamma_{j^t}|^2 \nonumber \\
    &\leq \|{U}^{m_{k-1}+1}\|_n^2 - \frac{M c_\phi^2}{16C_\phi s}\|\widetilde{V}_{T^{m_{k-1}+1}}\|_n^2. \label{eq:exp-decay-2}
\end{align}
Combining \textbf{Case ($A_1$)} and \textbf{Case ($A_2$)}, based on \eqref{eq:U-m-k} and \eqref{eq:exp-decay-2}, we have that:
\begin{equation}\label{eq:exp-decay}
    \|\widetilde{U}^{m_{k}+1}\|_n^2 \leq  \|\widetilde{U}^{m_{k-1}+1}\|_n^2 - \kappa_{n,k} \|\widetilde{V}^{m_{k-1}+1}\|_n^2,
\end{equation}
where $\kappa_{n,k} := \min\left(\frac{c_\phi^2}{4MC_\phi}, \frac{M c_\phi^2}{16C_\phi s} \right)>0$. By assumption that $M\leq K_M\sqrt{s}$, we have that: 
\begin{equation}
    \kappa_{n,k} \geq \kappa_2 \frac{M}{s},
\end{equation}
where $\kappa_2=\min\left(\frac{c_\phi^2}{4 K_M^2  C_\phi },\frac{c_\phi^2}{16  C_\phi }\right)$ is a generic positive constant.\\ \\
Define $q:=\lceil \frac{s}{M}\rceil$ as the smallest integer $\geq \frac{s}{M}$. WLOG., we can assume that $q= \frac{s}{M}$ for simplicity, i.e., $\frac{s}{M}$ is an integer. By applying \eqref{eq:exp-decay} to $l=k+1,k+2,...,k+q$, we have that:
\begin{equation}\label{eq:s-decay}
    \|U^{m_{k+q}+1}\|_n^2\leq \|U^{m_{k}+1} \|_n^2 -\kappa_2 \frac{M}{s} \sum_{l=k}^{k+q-1}\|\widetilde{V}^{m_{l}+1}\|_n^2. 
\end{equation}
Note that $\|\widetilde{V}^{m_{l}+1}\|_n^2 = \|(I-P[X_{T^{m_l+1}}])X\beta\|_n^2$ is a decreasing sequence in $l$.\\ \\
We divide our analysis into two Cases: \textbf{Case} $(B_1)$ and \textbf{Case} $(B_2)$.
\paragraph{\textbf{Case} $(B_1)$} 
\begin{equation}\label{eq:V-bound1}
\|\widetilde{V}^{m_k+q}\| > \frac{1}{2} \| \widetilde{V}^{{m_k+1}}\|.
\end{equation}
Then, \eqref{eq:V-bound1} and \eqref{eq:s-decay} implies that: 
\begin{align}
    \|U^{m_{k+q}+1}\|_n^2&\leq \|U^{m_{k}+1} \|_n^2 -\kappa_2 \frac{M}{s} \sum_{l=k}^{k+q-1}\|\widetilde{V}^{m_{l}+1}\|_n^2 \nonumber\\
    &\leq \|U^{m_{k}+1} \|_n^2-\kappa_2 \frac{M}{s} q \frac{1}{2}\|\widetilde{V}^{m_{k}+1}\|_n^2 \nonumber\\
    &=\|U^{m_{k}+1} \|_n^2-\frac{\kappa_2}{2}\|\widetilde{V}^{m_{k}+1}\|_n^2.\label{eq:U-decay-1}
\end{align}
Recall that for any positive integer $l$,
\begin{equation*}
    U^{m_l+1} = \widetilde{V}^{m_l+1} + (I-P[X_{T^{m_l+1}}])(r+\varepsilon).
\end{equation*}
Consequently, we have:
\begin{align}
    &\quad\|U^{m_l+1}\|_n^2 - \|\widetilde{V}^{m_l+1}\|_n^2 - \|(r+\varepsilon)\|_n^2 \nonumber\\
    &= 2<\widetilde{V}^{m_l+1}, (I-\mathcal{P}[{X_{T^{m_l+1}}}](r+\varepsilon))>_n - \|\mathcal{P}[{X_{T^{m_l+1}}}] (r+\varepsilon)\|_n^2.\label{eq:uv-expand}
\end{align}
Plugging in \eqref{eq:uv-expand} with $l=k$ and $l=k+q$ in \eqref{eq:U-decay-1}, we have that:
\begin{align}
    \|\widetilde{V}^{m_{k+q}+1}\|_n^2 &\leq \left(1-\frac{\kappa_2}{2}\right)\|\widetilde{V}^{m_{k}+1}\|_n^2
    \nonumber \\
    &+ 2\underbrace{<(\widetilde{V}^{m_{k}+1}, (I-\mathcal{P}[{X_{T^{m_k+1}}}])(r+\varepsilon)>_n}_{=:\Psi_1} \nonumber\\ 
    &- 2\underbrace{<(\widetilde{V}^{m_{k+q}+1}, (I-\mathcal{P}[{X_{T^{m_{k+q}+1}}}])(r+\varepsilon)>_n}_{=:\Psi_2}\nonumber\\
    &+ \underbrace{\|\mathcal{P}[{X_{T^{m_{k+q}+1}}}](r+\varepsilon)\|_n^2 - \|\mathcal{P}[{X_{T^{m_{k}+1}}}](r+\varepsilon)\|_n^2}_{=:\Psi_3}.\label{eq:V-decay-main}
\end{align}
Next, we establish bounds on $\Psi_1,\Psi_2$ and $\Psi_3$. Note that $\Psi_1$ and $\Psi_2$ only differ by the indices $k$ and $k+q$. It is sufficient to bound $\Psi_1$, then $\Psi_2$ will obey the same bound.\\ \\
Note that $\widetilde{V}^{m_{k}+1}=(I-\mathcal{P}[{X_{T^{m_{k}+1}}}]) X_{\widetilde{T}^{m_{k}+1}}\beta_{\widetilde{T}^{m_{k}+1}}$. By Assumption \ref{assump:SE}, we have:
\begin{equation}\label{eq:project-bound}
    \|\widetilde{V}^{m_{k}+1}\|_n^2 \geq c_\phi \|\beta_{\widetilde{T}^{m_{k}+1}}\|^2 \geq \frac{c_\phi}{C_\phi}\|X_{\widetilde{T}^{m_{k}+1}}\beta_{\widetilde{T}^{m_{k}+1}}\|_n^2.
\end{equation}
Therefore, 
\begin{align}
    |\Psi_1| &=|<\widetilde{V}^{m_{k}+1}, (I-\mathcal{P}[{X_{T^{m_{k}+1}}}])(r+\varepsilon)>_n| \nonumber \\
    &= \left|\frac{1}{n}(r+\varepsilon)^\intercal (I-\mathcal{P}[{X_{T^{m_{k}+1}}}]) X_{\widetilde{T}^{m_{k}+1}}\beta_{\widetilde{T}^{m_{k}+1}}\right|\nonumber\\
    &\leq \sqrt{\frac{C_r s\log(2p/\alpha)}{n}}\|\widetilde{V}^{m_{k}+1}\|_n + \left|\frac{1}{n}\varepsilon^\intercal X_{\widetilde{T}^{m_{k}+1}} \beta_{\widetilde{T}^{m_{k}+1}}\right| + \left| \frac{1}{n}\varepsilon^\intercal \mathcal{P}[{X_{T^{m_{k}+1}}}] X_{\widetilde{T}^{m_{k}+1}} \beta_{\widetilde{T}^{m_{k}+1}}\right| \nonumber \\
    &\leq \sqrt{\frac{C_r s\log(2p/\alpha)}{n}}\|\widetilde{V}^{m_{k}+1}\|_n + \sqrt{|\widetilde{T}^{m_{k}+1}|}\lambda_n \|\beta_{\widetilde{T}^{m_{k}+1}}\| + \sqrt{\frac{|\widetilde{T}^{m_{k}+1}|}{c_\phi}}\lambda_n \|X_{\widetilde{T}^{m_{k}+1}} \beta_{\widetilde{T}^{m_{k}+1}} \|_n \nonumber\\
    &\leq \kappa_{n,3}\sqrt{\frac{s\log(n)\log(2p/\alpha)}{n}}\|\widetilde{V}^{m_{k+q}+1}\|_n,
\end{align}
where the third last inequality follows from statement (ii) of Assumption \ref{assump:AS}, the second last inequality follows from Lemma \ref{lemma:bound-res}, and the last inequality follows from \eqref{eq:project-bound}, and 
\begin{equation}
    \kappa_{n,3} = \sqrt{\frac{C_r}{\log(n)}}+2\sigma \sqrt{C_{m^*}}\left(\sqrt{\frac{1}{c_\phi}}+\sqrt{\frac{C_{\phi}}{c_{\phi}^2}}\right)\label{eq:bound-res-V}
\end{equation}
is a constant bounded by $\kappa_3:= 3\sigma \sqrt{C_{m^*}}\left(\sqrt{\frac{1}{c_\phi}}+\sqrt{\frac{C_{\phi}}{c_{\phi}^2}}\right)$ for $n$ large enough.\\ \\
Note that we assume $\|\widetilde{V}^{m_{k}+1}\|_n^2 > \frac{C_V s\log(n)\log(2p/\alpha)}{n}$. Hence, \eqref{eq:bound-res-V} implies that
\begin{equation}\label{eq:psi-1-bound}
    |\Psi_1|\leq \frac{\kappa_3}{C_V^\frac{1}{2}}\|\widetilde{V}^{m_{k}+1}\|_n^2.
\end{equation}
Similarly, we have
\begin{equation}\label{eq:psi-2-bound}
    |\Psi_2|\leq \frac{\kappa_3}{C_V^\frac{1}{2}}\|\widetilde{V}^{m_{k+q}+1}\|_n^2\leq \frac{\kappa_3}{C_V^\frac{1}{2}}\|\widetilde{V}^{m_{k}+1}\|_n^2.
\end{equation}
Note the $T^{m_{k+q}+1}\supset T^{m_k+1}$, so we have that
\begin{equation}\label{eq:psi-3-bound}
    \Psi_3\leq 0.
\end{equation}
Plugging in \eqref{eq:psi-1-bound}, \eqref{eq:psi-2-bound} and \eqref{eq:psi-3-bound} in \eqref{eq:V-decay-main}, we have that:
\begin{align}
   \|\widetilde{V}^{m_{k+q}+1}\|_n^2&\leq \left(1-\frac{\kappa_2}{2} +  2\frac{\kappa_3}{C_V^\frac{1}{2}}\right)\|\widetilde{V}^{m_{k}+1}\|_n^2 \nonumber\\
   &\leq \left(1-\frac{\kappa_2}{4}\right)\|\widetilde{V}^{m_{k}+1}\|_n^2,\label{eq:V-decay-B-1}
\end{align}
given that $C_V$ is chosen such that $C_V\geq \frac{64 \kappa_3^2}{ \kappa_2^2}$. 
\paragraph{\textbf{Case} $(B_2)$} \begin{equation}\label{eq:V-decay-B-2}
\|\widetilde{V}^{m_{k+q}+1}\|\leq \frac{1}{2} \| \widetilde{V}^{m_k+1}\|.
\end{equation}
Therefore, combining \textbf{Case $(B_1)$} and \textbf{Case $(B_2)$}, we have that:
\begin{equation}\label{eq:V-bound-final}
    \|\widetilde{V}^{m_{k+q}+1}\|_n^2\leq (1-\zeta)\|\widetilde{V}^{m_{k}+1}\|_n^2,
\end{equation}
where $\zeta = \min(\frac{1}{2},\frac{\kappa_2}{4})\in (0,1)$ is a generic constant.
Following \eqref{eq:V-bound-final}, for all $k$ such that $m_{kq}+1\leq m^*$, we have that:
\begin{gather}
    \|\widetilde{V}^{m_{kq}+1}\|_n^2 \leq (1-\zeta )^k \|\widetilde{V}^0\|_n^2.\label{eq:V-k}
\end{gather}
Since $m_{kq}+1 = k M q+1 = k s+ 1$, let $k_0 = \lceil\frac{K}{\zeta}\rceil \log(n)\leq C_{m^*}\log(n) -\log(n)$, i.e., 
\begin{equation*}
    C_{m^*}\geq 1 + \left\lceil\frac{K\log(n)}{\min\left(\frac{1}{2}, \frac{c_\phi^2}{16 C_\phi K_M^2},\frac{c_\phi^2}{64 C_\phi }\right)}\right\rceil,
\end{equation*} we have that $m_{k_0 q}+1\leq m^*$. Replace $k$ with $k_0$ in \eqref{eq:V-k}, by statement (iii) of Assumption \ref{assump:AS}, we have that for $n$ large enough:
\begin{align}
    \|\widetilde{V}^{m_{k_0 q}+1}\|_n^2 &\leq (1-\zeta )^{k_0} \|\widetilde{V}^0\|_n^2\nonumber \\
   & \leq ((1-\zeta)^\frac{1}{\zeta})^{K \log(n)} n^{K-1}\nonumber \\
    &\leq \exp(-K\log(n))n^{K-1} = \frac{1}{n}\nonumber \\
    &< \frac{C_V s\log(n)\log(2p/\alpha)}{n},\nonumber
\end{align}
which contradicts with the assumption stated in \eqref{eq:V-bound-assump}. It implies that either \eqref{eq:v-bound} or \eqref{eq:C-v-bound} must hold. Therefore, we have the conclusion.
\end{proof}

\begin{proof}[Proof of Corollary \ref{coro:pba}]
By construction of Algorithm \ref{algo:iterated-pba}, we have that $$\|U^{m+1}\|_n^2\leq \|U^m\|_n^2 -\gamma_{j^m}^2,$$ for all non-negative integers $m$. Therefore, \begin{equation}
    \frac{\|U^{m+1}\|_n^2}{\|U^{m}\|_n^2}\leq 1-\frac{\gamma_{j^m}^2}{\|U^{m}\|_n^2}.
\end{equation}
It implies that the algorithm will not stop if 
\begin{equation}\label{eq:gamma-lower-1}
    \frac{\gamma_{j^m}^2}{\|U^{m}\|_n^2}\geq 4\zeta_0 \frac{\log(2p/\alpha)}{n}.
\end{equation}
Suppose the algorithm does not stop before $M':= C' s\log(n)$ for some $C'\leq C_M$ but $C'>C_{m^*}$. By Theorem \ref{theo:PBA}, we know that at some  $m_k+1\leq m^*$, we have that 
\begin{equation}\label{eq:V-bound-up}
    \|\widetilde{V}^{m_k+1}\|_n^2\leq \frac{C_V s\log(n)\log(2p/\alpha)}{n}.
\end{equation}
Since $\|\widetilde{V}^{m_l+1}\|_n^2$ is decreasing in $l$, \eqref{eq:V-bound-up} holds for all $l\geq k$ when $k$ is replaced by $l$.
For any $m_l+1$ such that $k\leq l\leq \frac{C's\log(n)}{M}$, we have that:
\begin{align}
    &\quad|\|U^{m_l+1}\|_n^2- \|r+\varepsilon\|_n^2|\\
    &= |\|\widetilde{V}^{m_l+1}+(I-\mathcal{P}[{X_{T^{m_l+1}}}])(r+\varepsilon)\|_n^2 -\|r+\varepsilon\|_n^2|  \nonumber\\ 
    &=  2<\widetilde{V}^{m_l+1} - \mathcal{P}[{X_{T^{m_l+1}}}](r+\varepsilon), r+\varepsilon >_n+\|\widetilde{V}^{m_l+1} - \mathcal{P}[{X_{T^{m_l+1}}}](r+\varepsilon)\|_n^2\nonumber\\
    &\leq 2<X_{\widetilde{T}^{m_l+1}}\beta_{\widetilde{T}^{m_l+1}}, r+\varepsilon >_n \nonumber \\ 
    &\quad- 2\| \mathcal{P}[{X_{T^{m_l+1}}}](r+\varepsilon)\|_n^2 + 2\|\widetilde{V}^{m_l+1}\|_n^2 +  2\|\mathcal{P}[{X_{T^{m_l+1}}}](r+\varepsilon)\|_n^2 \nonumber\\
    & = 2 <X_{\widetilde{T}^{m_l+1}}\beta_{\widetilde{T}^{m_l+1}}, r >_n+ <X_{\widetilde{T}^{m_l+1}}\beta_{\widetilde{T}^{m_l+1}}, \varepsilon >_n + 2\|\widetilde{V}^{m_l+1}\|_n^2\\
    &\leq 2\|\widetilde{V}^{m_l+1}\|_n \|r\|_n + 2\sqrt{|T^{m_l+1}|}\lambda_n \|\beta_{\widetilde{T}^{m_l+1}}\| + 2\|\widetilde{V}^{m_l+1}\|_n^2 \nonumber\\ 
    &\leq \kappa_{n,4}\frac{s\log(2p/\alpha)\log(n)}{n},\label{eq:U-lower-bound}
\end{align}
 where the second last inequality follows from Assumption \ref{assump:AS}, Assumption \ref{assump:tail}, and Lemma \ref{lemma:bound-res}, the last inequality follows from Assumption \ref{assump:SE} and \eqref{eq:V-bound-up}, and 
\begin{equation}
    \kappa_{n,4} := 2\left(\sqrt{\frac{C_r C_V}{\log(n)}}+ \sqrt{\frac{C_\phi C_V C'}{c_\phi}}+C_V \right),
\end{equation}
is a constant bounded by \begin{equation}
    \kappa_4=2+{2\sqrt{\frac{C_\phi C_V C'}{c_\phi}}}+2C_V,
\end{equation}
for $n$ large enough. Hence, given the growth condition that $\frac{s\log(2p/\alpha)\log(n)}{n}\rightarrow 0$, for $n$ large enough, we have that 
\begin{align*}
    \|U^{m_l+1}\|_n^2 & \geq \|r+\varepsilon\|_n^2 - \kappa_{n,4}\frac{s\log(2p/\alpha)\log(n)}{n}\\ & \geq \|\varepsilon\|_n^2 - 2<r,\varepsilon>_n+\|r\|_n - \kappa_{n,4}\frac{s\log(2p/\alpha)\log(n)}{n} \\ & \geq  \|\varepsilon\|_n^2 - 2\|r\|_n\|\varepsilon\|_n - \kappa_{n,4}\frac{s\log(2p/\alpha)\log(n)}{n}\\
    & \geq (1-\eta)\sigma^2-2\sqrt{\frac{C_r\log(2p/\alpha)}{n}}\sqrt{(1+\eta)\sigma^2} - \kappa_4 \frac{s\log(2p/\alpha)\log(n)}{n}\\
    & \geq (1-2\eta)\sigma^2,
\end{align*}
where $\eta>0$ is a small positive constant defined in Assumption \ref{assump:tail} and the second last inequality follows from Assumption \ref{assump:tail}.
Therefore, for any $m\leq m_l+1$, and for $n$ large enough, we have that:
\begin{equation*}
    \|U^m\|_n^2\geq \|U^{m_k+1}\|_n^2 \geq (1-2\eta){\sigma}^2.
\end{equation*}
Consequently, \eqref{eq:gamma-lower-1} implies that:
\begin{equation}
    \gamma_{j^m}^2\geq 4\zeta_0 (1-2\eta)\sigma^2 \frac{\log(2p/\alpha)}{n}.
\end{equation}
As a result, assume that $m_l-m_k = M\lfloor \frac{{(C'-C_{m^*})} s \log(n)}{M}\rfloor$ such that $m_l\leq C's\log(n)$. WLOG., assume that $ \frac{{(C'-C_{m^*})} s \log(n)}{M}$ is an integer so that $m_l-m_k = {(C'-C_{m^*})} s \log(n)$. On the one hand, we have that: 
\begin{align}
    \|U^{m_l+1}\|_n^2&\leq \|U^{m_k+1}\|_n^2 -\sum_{t=m_k+1}^{m_l} \gamma_{j^t}^2 \nonumber\\
    &\leq \|U^{m_k+1}\|_n^2 - (m_l-m_k)4\zeta_0 (1-2\eta)\sigma^2 \frac{\log(2p/\alpha)}{n}\nonumber\nonumber\\
    &\leq \|U^{m_k+1}\|_n^2 - \zeta (1-2\eta)\sigma^2 {(C'-C_{m^*})} s \log(n)  \frac{\log(2p/\alpha)}{n}.\label{eq:U-diff-1}
\end{align}
On the other hand, for $n$ large enough, \eqref{eq:U-lower-bound} suggests that:
\begin{align}
    |\|U^{m_l+1}\|_n^2 - \|U^{m_k+1}\|_n^2|&\leq |\|U^{m_l+1}\|_n^2-\|(r+\varepsilon)\|_n^2| + |\|U^{m_k+1}\|_n^2-\|(r+\varepsilon)\|_n^2| \nonumber\\ 
    &\leq 2\kappa_{n,4}\frac{s\log(2p/\alpha)\log(n)}{n}\leq 2\kappa_4 \frac{s\log(2p/\alpha)\log(n)}{n}.\label{eq:U-diff-2}
\end{align}
\eqref{eq:U-diff-1} and \eqref{eq:U-diff-2} imply that 
\begin{equation*}
    2\kappa_{4} \geq 4\zeta_0(1-2\eta)\sigma^2  ({C'-C_{m^*}}), 
\end{equation*}
or equivalently, 
\begin{equation}
    4+4\sqrt{\frac{C_\phi C_V C'}{c_\phi}}+4C_V - 4\zeta_0(1-2\eta)\sigma^2  ({C'-C_{m^*}})\geq 0.
\end{equation}
This implies that
\begin{equation}\label{eq:def-C-prime}
    C'\leq \bar{M}:= \left(\frac{\sqrt{\frac{C_\phi C_V}{c_\phi}}+\sqrt{\frac{C_\phi C_V}{c_\phi}+ 4\zeta_0 (1-2\eta) \sigma^2 (1+C_V+\zeta_0(1-2\eta)\sigma^2 C_{m^*})}}{2\zeta_0(1-2\eta)\sigma^2} \right)^2.
\end{equation}
Therefore, if $C'>\bar{M} $, for $n$ large enough, the algorithm must stop before $C's\log(n)$, i.e., $M_A^*\precsim s\log(n)$, $M_A^*\leq M_n$ by assumption that $C'\leq C_M$.\\ \\
Now, we show that $\|X(\beta^{M_A^*}-\beta)\|_n^2\precsim \frac{s\log(2p/\alpha)\log(n)}{n}$. A similar bound will hold for $\|(\beta^{M_A^*}-\beta)\|^2$, due to  Assumption \ref{assump:SE}.
When the algorithm stops, then it must hold that
\begin{equation}\label{eq:gamma-A-upper}
    \frac{|\gamma_{j^{M_A^*}}|^2}{\|U^{M_A^*}\|_n^2} < 4\zeta_0\frac{\log(2p/\alpha)}{n}.
\end{equation}
For $n$ large enough, we know that $M_A^*\leq C's\log(n) <  M_n$ from analysis in the above. Therefore, there exists $k$ such that 
$M_A^*\in \{m_{k-1}+1,...,m_k\}$.
In the following, we use the notation from the proof of Theorem \ref{theo:PBA}. For any $m \in \{m_{k-1}+1,...,m_k\}$, for $n$ large enough, we have that:
\begin{align}
    \|U^m\|_n^2 &= \|\mathcal{V}^m+(I-\mathcal{P}[{X_{T^{m_{k-1}+1}}}])(r+\varepsilon)\|_n^2 \nonumber\\
    &\leq 2\|\mathcal{V}^m\|_n^2 + 2\|(I-\mathcal{P}[{X_{T^{m_{k-1}+1}}}])(r+\varepsilon)\|_n^2 \nonumber 
    \\ 
    &\leq 2\|\mathcal{V}^m\|_n^2+ 4 \left(\frac{C_r s\log(2p/\alpha)}{n} + \hat\sigma^2\right)\nonumber\\
    &\leq 2\|\mathcal{V}^m\|_n^2+ 4 (1+2\eta)\sigma^2,\label{eq:U-upper-gamma}
\end{align}
where the last inequality follows from Assumption \ref{assump:tail} and $\frac{s\log(p)}{n}\rightarrow 0$.
Therefore, \eqref{eq:gamma-A-upper} and \eqref{eq:U-upper-gamma} implies that:
\begin{equation}\label{eq:gamma-upper-assump}
    |\gamma_{j^{M_A^*}}|^2 \leq 8\zeta_0\frac{\log(2p/\alpha)}{n}\|\mathcal{V}^{M_A^*}\|_n^2 +{16\zeta_0}(1+2\eta)\sigma^2 \frac{\log(2p/\alpha)}{n}.
\end{equation}
Assume that $\|\mathcal{V}^{M_A^*}\|>\frac{C_V s\log(2p/\alpha)\log(n)}{n}$ for some fixed  constant $C_V$.  Otherwise we already have that $\|X(\beta-\beta^{M_A^*})\|_n^2 = \|\mathcal{V}^m\|_n^2\leq \frac{C_V s\log(2p/\alpha)\log(n)}{n}$, which provides the result.\\
For $C_V$ large enough, by Theorem \ref{theo:PBA}, we know that $M_A^*\leq C_m^* s\log(n)$. Recall that
\begin{equation}
    \mathcal{V}^{M_A^*}= (I-\mathcal{P}[{X_{T^{m_{k-1}+1}}}])X_{\widetilde{T}^{m_{k-1}+1}}\beta_{\widetilde{T}^{m_{k-1}+1}} - X_{T_{M_{A}^*,k}}\delta_{M_{A}^*,k}.
\end{equation}
Note that $\mathcal{V}^m$ can be rewritten as: 
\begin{align}
    &\quad(I-\mathcal{P}[{X_{T^{m_{k-1}+1}}}])X_{\widetilde{T}^{m_{k-1}+1}}\beta_{\widetilde{T}^{m_{k-1}+1}} - X_{T_{M_{A}^*,k}}\delta_{M_{A}^*,k} \\ 
    &= X_{\widetilde{T}^{m_{k-1}+1}}\beta_{\widetilde{T}^{m_{k-1}+1}} - \mathcal{P}[{X_{T^{m_{k-1}+1}}}] X_{\widetilde{T}^{m_{k-1}+1}}\beta_{\widetilde{T}^{m_{k-1}+1}}- X_{T_{M_{A}^*,k}}\delta_{M_{A}^*,k}\\
    &= X_{\hat{T}^{M_A^*}}\xi_{\hat{T}^{M_A^*}},
\end{align}
where $\xi_{\hat{T}^{M_A^*}}$ is a vector of dimension $|\hat{T}^{M_A^*}|$ with $\hat{T}^{M_A^*}:=T^{m_{k-1}+1} \cup  T_{M_A^*,k}\cup   \widetilde{T}^{m_{k-1}+1}$ that has cardinality 
$$|\hat{T}^{M_A^*}| \leq s+m_{k-1}+M < (C_{m^*}+1) s \log(n)\leq M_n.$$
As a result, by optimality of $\gamma_{j^{M_A^*}}$, we have that:
\begin{align}
    |\gamma_{j^{M_A^*}}| \|\xi_{\hat{T}^{M_A^*}}\|_1 
    &\geq <U^{M_A^*},  X_{\hat{T}^{M_A^*}}\xi_{\hat{T}^{M_A^*}}>_n \nonumber\\
    &\geq <X_{\hat{T}^{M_A^*}}\xi_{\hat{T}^{M_A^*}} ,X_{\hat{T}^{M_A^*}}\xi_{\hat{T}^{M_A^*}} >_n \nonumber\\
    &\quad + <X_{\hat{T}^{M_A^*}}\xi_{\hat{T}^{M_A^*}},(I-\mathcal{P}[{X_{T^{m_{k-1}+1}}}])(r+\varepsilon)>_n \nonumber\\
    &\geq \|X_{\hat{T}^{M_A^*}}\xi_{\hat{T}^{M_A^*}}\|_n^2 - \|X_{\hat{T}^{M_A^*}}\xi_{\hat{T}^{M_A^*}} \|_n \| (I-\mathcal{P}[{X_{T^{m_{k-1}+1}}}]) r \|_n \nonumber\\ 
    &\quad + < X_{\hat{T}^{M_A^*}}\xi_{\hat{T}^{M_A^*}}, \varepsilon>_n - < X_{\hat{T}^{M_A^*}}\xi_{\hat{T}^{M_A^*}}, \mathcal{P}[{X_{T^{m_{k-1}+1}}}]\varepsilon>_n \nonumber\\
    & \geq \|X_{\hat{T}^{M_A^*}}\xi_{\hat{T}^{M_A^*}}\|_n^2 - \|X_{\hat{T}^{M_A^*}}\xi_{\hat{T}^{M_A^*}} \|_n \| r \|_n \nonumber \\
    & \quad - \frac{1}{n}\|X_{\hat{T}^{M_A^*}}^\intercal \varepsilon\| \|\xi_{\hat{T}^{M_A^*}}\| - \| X_{\hat{T}^{M_A^*}}\xi_{\hat{T}^{M_A^*}}\|_n \|\mathcal{P}[{X_{T^{m_{k-1}+1}}}]\varepsilon\|_n \nonumber \\
    &\geq \|X_{\hat{T}^{M_A^*}}\xi_{\hat{T}^{M_A^*}}\|_n^2 - \|X_{\hat{T}^{M_A^*}}\xi_{\hat{T}^{M_A^*}} \|_n \| r \|_n \nonumber \\
    &\quad - \sqrt{|\hat{T}^{M_A^*}|}\lambda_n \| \xi_{\hat{T}^{M_A^*}}\| - \sqrt{\frac{|T^{m_{k-1}+1}|}{c_\phi}}\lambda_n \| X_{\hat{T}^{M_A^*}}\xi_{\hat{T}^{M_A^*}}\|_n \nonumber\\
    &\geq \kappa_{5,n} \| X_{\hat{T}^{M_A^*}}\xi_{\hat{T}^{M_A^*}}\|_n^2,\label{eq:xi-bound}
\end{align}
where the second last inequality follows from Assumption \ref{assump:tail}, the last inequality follows from Assumption  \ref{assump:SE}, statement (ii) of Assumption \ref{assump:AS}, and Lemma \ref{lemma:bound-res}, and
\begin{equation}
    \kappa_{5,n}:=1-\sqrt{\frac{C_r}{C_V s\log(n)}} - 2\sqrt{\frac{C_{m^*}+1}{c_\phi C_V}}
\end{equation}
is bounded from below by a fixed positive constant 
\begin{equation}
    \kappa_5:= 1- 2\sqrt{\frac{C_{m^*}+1}{c_\phi C_V}} -\eta_5,
\end{equation}
for arbitrarily small $\eta_5$ and $n$ large enough, given that $C_V>4 \frac{C_{m^*}+1}{c_\phi}$. As a result, \eqref{eq:xi-bound} yields that
\begin{equation}\label{eq:bound-gamma-final}
    |\gamma_{j^{M_A^*}}|\geq \frac{\kappa_5 \| X_{\hat{T}^{M_A^*}}\xi_{\hat{T}^{M_A^*}}\|_n^2}{\|\xi_{\hat{T}^{M_A^*}}\|_1}\geq \frac{\kappa_5\sqrt{c_\phi}\| X_{\hat{T}^{M_A^*}}\xi_{\hat{T}^{M_A^*}}\|_n}{\sqrt{(C_{m^*}+1)s\log(n)}}=\frac{\kappa_5\sqrt{c_\phi}\| \mathcal{V}^{M_A^*}\|_n}{\sqrt{(C_{m^*}+1)s\log(n)}},
\end{equation}
where the last inequality follows from Assumption \ref{assump:SE}. Plugging in \eqref{eq:bound-gamma-final} in \eqref{eq:gamma-upper-assump}, we have:
\begin{equation}
    \frac{\kappa_5^2 c_\phi \|\mathcal{V}^{M_A^*}\|_n^2}{(C_{m^*}+1)s\log(n)}\leq  8\zeta_0\frac{\log(2p/\alpha)}{n}\|\mathcal{V}^{M_A^*}\|_n^2 +{16\zeta_0}(1+2\eta)\sigma^2 \frac{\log(2p/\alpha)}{n},
\end{equation}
which yields that:
\begin{align}
    &\quad\|\mathcal{V}^{M_A^*}\|_n^2 \left(\kappa_5^2 c_\phi - 8\zeta_0(C_{m^*}+1)\frac{s\log(2p/\alpha)\log(n)}{n}\right)\\
    &\leq {16\zeta_0}(C_{m^*+1})(1+2\eta)\sigma^2 \frac{s\log(2p/\alpha)\log(n)}{n}.
\end{align}
Given the growth condition that $\frac{s\log(2p/\alpha)\log(n)}{n}\rightarrow 0$, for $n$ large enough such that $8\zeta_0(C_{m^*}+1)\frac{s\log(2p/\alpha)\log(n)}{n}<\frac{\kappa_5^2 c_\phi}{2}$, we have that:
\begin{equation*}
    \|\mathcal{V}^{M_A^*}\|_n^2\leq \frac{{32\zeta_0}(C_{m^*+1})(1+2\eta)\sigma^2} {\kappa_5^2 c_{\phi}} \frac{s\log(2p/\alpha)\log(n)}{n},
\end{equation*}
which yields the conclusion.

\end{proof}

\begin{proof}[Proof of Theorem \ref{theo:OBA}]
The result of Theorem \ref{theo:OBA} is  special case of Theorem \ref{theo:PBA} with $M=1$. Therefore, we have the conclusion.
\end{proof}

\begin{proof}[Proof of Corollary \ref{coro:oba}]
The result of Corollary \ref{coro:oba} is a special case of Corollary \ref{coro:pba} with $M=1$. Therefore, we have the conclusion.
\end{proof}

\begin{proof}[Proof of Theorem \ref{ControlTheorem}]
According to Theorem  \ref{theo:PBA} and \ref{theo:OBA}, the condition on the prediction rates in HLMS(P) in \cite{BelloniChernozhukovHansen2011} are satisfied by Assumption A.\ref{A3} (i):
\begin{align*}
    ||X(\beta^{m^*}-\beta)||_n\precsim \sqrt{\frac{s\log(p\vee n)\log(n)}{n}} = o\left(n^{-1/4}\right)
\end{align*}
\begin{align*}
    ||X(\gamma^{m^*}-\gamma)||_n\precsim \sqrt{\frac{s\log(p\vee n)\log(n)}{n}} = o\left(n^{-1/4}\right).
\end{align*}
Although we have $||\beta^{m^*}||_0\le Cs\log(n)$ and $||\gamma^{m^*}||_0\le s\log(n)$, the slightly stronger growth conditions in Assumption A.\ref{A3} allow us to apply the proof of Theorem 2 in \cite{BelloniChernozhukovHansen2011}. It is worth to notice that the growth rate $n^{2/q}\frac{s^2 \log (p \vee n)\log(n)^2}{n}\rightarrow 0$ is needed for consistent variance estimation (see Step 5 in the proof of Theorem 2 in \cite{BelloniChernozhukovHansen2011}). Condition ASTE(P) and Condition SE(P) hold due to Assumptions A.\ref{A1}-A.\ref{A3}. Hence, Theorem 2 in \cite{BelloniChernozhukovHansen2011} yields the result.
\end{proof}


\begin{proof}[Proof of Theorem \ref{ivtheorem}]
Assumptions B.\ref{B1}-B.\ref{B3} ensure sufficiently fast convergence rates of the fitted optimal instruments $\hat{D}(z_i)=z_i^\intercal\hat{\gamma}$ via \texttt{I}-\texttt{pBA} or \texttt{oBA} in the first step, i.e., 
\begin{align*}
||\hat{D}(z_i)-D(z_i)||_{n}\precsim\sqrt{\frac{s\log(p\vee n)\log(n)}{n}}
\end{align*}
with probability $1-o(1)$ by Theorem \ref{theo:PBA} and Theorem \ref{theo:OBA}, respectively. Since the maximal sparse eigenvalues are uniformly bounded from above due to Assumption B.\ref{B2},
we conclude
\begin{align*}
||\hat\gamma-\gamma||_2\le C||\hat{D}(z_i)-D(z_i)||_{n}\precsim \sqrt{\frac{s\log(p\vee n)\log(n)}{n}}
\end{align*}
with probability $1-o(1)$ which implies
\begin{align*}
||\hat\gamma-\gamma||_1\le C\sqrt{s\log(n)}||\hat\gamma-\gamma||_2\precsim C\sqrt{\frac{s^2\log(p\vee n)\log(n)^2}{n}}.
\end{align*}
Assuming $\frac{s^2 \log^2(p \vee n)\log^2(n)}{n} \rightarrow 0$ in Assumption B.\ref{B4}, it implies equation (F.2) in \cite{BCCH:2012}:
$$\sqrt{\log(p)}||\hat\gamma-\gamma||_1\precsim C\sqrt{\frac{s^2\log^2(p\vee n)\log(n)^2}{n}}\rightarrow 0.$$
This allows us to apply the proof of Theorem 4 in \cite{BCCH:2012}. It is also worth to notice that the growth rate $\frac{s^2 \log^2 (p \vee n)\log^2(n)}{n}\rightarrow 0$ ensures consistency of the variance estimation (see Step 5 in the proof of Theorem 4 in \cite{BCCH:2012}). This concludes the proof.
\end{proof}

\pagebreak

\section{Additional Results of the Simulation Study}\label{app:add_sim}
\nopagebreak

\begin{table}[h]
\resizebox{0.8\textwidth}{!}{%
\begin{tabular}{|r|rrrrrr|rrrrrr|r|}
  \hline
  &\multicolumn{6}{c|}{iid covariates}&\multicolumn{6}{c|}{correlated covariates}&\\
  \hline
  & \textbf{Lasso} & \textbf{pBA}& \textbf{I-pBA} &  \textbf{oBA} & \textbf{CV-pBA} & \textbf{CV-Lasso} & \textbf{Lasso} & \textbf{pBA}& \textbf{I-pBA} & \textbf{oBA} & \textbf{CV-pBA} & \textbf{CV-Lasso} &  \\
   \hline
  \textbf{p}&\multicolumn{12}{c|}{Mean Absolute Error}&\\
  \hline
  200 & 0.042 & 0.042 & 0.045&0.042 & 0.051 & 0.051 & 0.045 & 0.046 & 0.046&0.045 & 0.049 & 0.047 & \parbox[t]{2mm}{\multirow{5}{*}{\rotatebox[origin=c]{-90}{Control-1}}}\\ 
  600 & 0.046 & 0.044 & 0.045&0.045 & 0.063 & 0.070 & 0.048 & 0.047 & 0.044&0.049 & 0.062 & 0.064 &\\ 
  1000 & 0.055 & 0.050 &0.048& 0.052 & 0.077 & 0.091 & 0.049 & 0.044 &0.049& 0.046 & 0.068 & 0.075 &\\ 
  1400 & 0.062 & 0.047 &0.049& 0.050 & 0.080 & 0.099 & 0.059 & 0.047 &0.049& 0.050 & 0.074 & 0.089& \\ 
  1800 & 0.076 & 0.049 &0.049& 0.052 & 0.089 & 0.112 & 0.068 & 0.050 &0.047& 0.054 & 0.082 & 0.098& \\ \hline
  200 & 0.044 & 0.042 & 0.040&0.042 & 0.049 & 0.049 & 0.040 & 0.040 & 0.042&0.040 & 0.045 & 0.043 & \parbox[t]{2mm}{\multirow{5}{*}{\rotatebox[origin=c]{-90}{Control-2}}}\\ 
  600 & 0.045 & 0.042 & 0.045&0.043 & 0.061 & 0.059 & 0.046 & 0.045 & 0.044&0.045 & 0.057 & 0.056& \\ 
  1000 & 0.048 & 0.046 &0.044& 0.045 & 0.074 & 0.078 & 0.045 & 0.044 &0.043& 0.044 & 0.063 & 0.064& \\ 
  1400 & 0.051 & 0.045 &0.044& 0.046 & 0.075 & 0.073 & 0.049 & 0.044 &0.045& 0.046 & 0.072 & 0.076& \\ 
  1800 & 0.056 & 0.046 &0.046& 0.046 & 0.083 & 0.082 & 0.052 & 0.044 &0.041& 0.045 & 0.072 & 0.074& \\ \hline
  200 & 0.059 & 0.057 & 0.059&0.056 & 0.055 & 0.056 & 0.100 & 0.097 & 0.096&0.097 & 0.082 & 0.120 &  \parbox[t]{2mm}{\multirow{5}{*}{\rotatebox[origin=c]{-90}{IV}}}\\ 
  600 & 0.058 & 0.057 & 0.057&0.056 & 0.055 & 0.056 & 0.140 & 0.118 & 0.117&0.115 & 0.133 & 1.075 &\\ 
  1000 & 0.062 & 0.063 &0.063& 0.061 & 0.060 & 0.061 & 0.145 & 0.123 &0.130& 0.123 & 0.113 & 2.180& \\ 
  1400 & 0.061 & 0.065 &0.063& 0.064 & 0.059 & 0.061 & 0.417 & 0.124 &0.133& 0.125 & 0.256 & 3.089& \\ 
  1800 & 0.062 & 0.064 &0.063& 0.062 & 0.060 & 0.176 & 0.229 & 0.133 &0.135& 0.131 & 0.207 & 2.318& \\ \hline
  \textbf{p}&\multicolumn{12}{c|}{Standard Deviation}&\\
  \hline
  200 & 0.053 & 0.053 & 0.054&0.053 & 0.065 & 0.058 & 0.055 & 0.057 & 0.056&0.055 & 0.061 & 0.055 & \parbox[t]{2mm}{\multirow{5}{*}{\rotatebox[origin=c]{-90}{Control-1}}}\\ 
  600 & 0.057 & 0.054 & 0.054&0.053 & 0.078 & 0.065 & 0.060 & 0.060 & 0.054&0.060 & 0.079 & 0.067 &\\ 
  1000 & 0.062 & 0.057 &0.055& 0.057 & 0.097 & 0.072 & 0.058 & 0.055 &0.058& 0.054 & 0.086 & 0.068& \\ 
  1400 & 0.062 & 0.054 &0.055& 0.055 & 0.100 & 0.075 & 0.064 & 0.058 &0.057& 0.058 & 0.093 & 0.071& \\ 
  1800 & 0.071 & 0.057 &0.056& 0.058 & 0.115 & 0.081 & 0.065 & 0.062 &0.056& 0.059 & 0.106 & 0.072& \\ \hline
  200 & 0.053 & 0.052 & 0.050&0.053 & 0.062 & 0.058 & 0.050 & 0.049 & 0.052&0.050 & 0.057 & 0.053 & \parbox[t]{2mm}{\multirow{5}{*}{\rotatebox[origin=c]{-90}{Control-2}}}\\ 
  600 & 0.055 & 0.052 & 0.056&0.053 & 0.073 & 0.066 & 0.057 & 0.056 & 0.055&0.056 & 0.072 & 0.063 &\\ 
  1000 & 0.060 & 0.057 &0.054& 0.057 & 0.092 & 0.076 & 0.055 & 0.053 &0.053& 0.054 & 0.080 & 0.066& \\ 
  1400 & 0.062 & 0.056 &0.055& 0.057 & 0.093 & 0.074 & 0.061 & 0.055 &0.056& 0.056 & 0.090 & 0.078& \\ 
  1800 & 0.068 & 0.058 &0.057& 0.058 & 0.104 & 0.083 & 0.060 & 0.055 &0.053& 0.055 & 0.089 & 0.072& \\ \hline
  200 & 0.073 & 0.070 & 0.074&0.069 & 0.063 & 0.068 & 0.122 & 0.119 & 0.119&0.118 & 0.094 & 0.322 &  \parbox[t]{2mm}{\multirow{5}{*}{\rotatebox[origin=c]{-90}{IV}}}\\ 
  600 & 0.072 & 0.069 & 0.072&0.068 & 0.059 & 0.067 & 0.296 & 0.140 & 0.136&0.136 & 0.724 & 3.754& \\ 
  1000 & 0.074 & 0.077 &0.076& 0.075 & 0.065 & 0.073 & 0.314 & 0.145 &0.152& 0.145 & 0.283 & 10.787& \\ 
  1400 & 0.072 & 0.079 &0.077& 0.077 & 0.061 & 0.072 & 4.993 & 0.142 &0.156& 0.144 & 1.424 & 9.670& \\ 
  1800 & 0.071 & 0.077 &0.076& 0.074 & 0.061 & 2.551 & 1.443 & 0.153 &0.153& 0.150 & 1.541 & 7.232& \\ 
        \hline
  \textbf{p}&\multicolumn{12}{c|}{Rejection Rate}&\\
  \hline
  200 & 0.070 & 0.076 & 0.074&0.062 & 0.058 & 0.082 & 0.064 & 0.070 & 0.090&0.066 & 0.050 & 0.080 & \parbox[t]{2mm}{\multirow{5}{*}{\rotatebox[origin=c]{-90}{Control-1}}}\\ 
  600 & 0.080 & 0.080 & 0.074&0.088 & 0.038 & 0.230 & 0.090 & 0.102 & 0.066&0.096 & 0.058 & 0.178 &\\ 
  1000 & 0.150 & 0.098 &0.100& 0.114 & 0.062 & 0.362 & 0.096 & 0.072 &0.098& 0.076 & 0.046 & 0.256& \\ 
  1400 & 0.188 & 0.076 &0.076& 0.098 & 0.056 & 0.432 & 0.172 & 0.100 &0.106& 0.108 & 0.042 & 0.340 &\\ 
  1800 & 0.276 & 0.084 &0.096& 0.126 & 0.068 & 0.500 & 0.224 & 0.114 &0.080& 0.134 & 0.066 & 0.420 &\\  \hline
  200 & 0.058 & 0.058 & 0.034&0.050 & 0.038 & 0.090 & 0.054 & 0.032 & 0.050&0.044 & 0.038 & 0.058 & \parbox[t]{2mm}{\multirow{5}{*}{\rotatebox[origin=c]{-90}{Control-2}}}\\   
  600 & 0.078 & 0.058 & 0.064&0.046 & 0.048 & 0.180 & 0.070 & 0.076 & 0.074&0.074 & 0.052 & 0.146 &\\ 
  1000 & 0.094 & 0.074 &0.064& 0.082 & 0.080 & 0.264 & 0.058 & 0.058 &0.048& 0.062 & 0.068 & 0.194 &\\ 
  1400 & 0.118 & 0.072 &0.060& 0.084 & 0.068 & 0.242 & 0.102 & 0.066 &0.060& 0.066 & 0.062 & 0.274& \\ 
  1800 & 0.154 & 0.078 &0.070& 0.084 & 0.064 & 0.296 & 0.118 & 0.060 &0.074& 0.066 & 0.046 & 0.278& \\ \hline
  200 & 0.054 & 0.048 & 0.072&0.048 & 0.064 & 0.056 & 0.070 & 0.060 & 0.062&0.062 & 0.092 & 0.074 &  \parbox[t]{2mm}{\multirow{5}{*}{\rotatebox[origin=c]{-90}{IV}}}\\ 
  600 & 0.068 & 0.056 & 0.062&0.064 & 0.104 & 0.070 & 0.066 & 0.062 & 0.052&0.052 & 0.118 & 0.042 &\\ 
  1000 & 0.072 & 0.072 &0.080& 0.074 & 0.106 & 0.072 & 0.078 & 0.068 &0.080& 0.070 & 0.140 & 0.050& \\ 
  1400 & 0.054 & 0.056 &0.076& 0.064 & 0.098 & 0.062 & 0.070 & 0.048 &0.074& 0.054 & 0.126 & 0.040 &\\ 
  1800 & 0.076 & 0.062 &0.058& 0.060 & 0.138 & 0.064 & 0.070 & 0.068 &0.074& 0.066 & 0.140 & 0.040& \\ 
  \hline
  \textbf{p}&\multicolumn{12}{c|}{Computation Time}&\\
  \hline
  200  & 0.010 & 0.004 & 0.005 & 0.006 & 1.511  & 9.360   & 0.010 & 0.004 & 0.005 & 0.006 & 1.498  & 9.553   & \parbox[t]{2mm}{\multirow{5}{*}{\rotatebox[origin=c]{-90}{Control-1}}}\\ 
  600  & 0.019 & 0.009 & 0.020 & 0.011 & 3.956  & 22.451  & 0.018 & 0.009 & 0.019 & 0.012 & 4.077  & 23.140&  \\
  1000 & 0.032 & 0.028 & 0.018 & 0.028 & 10.032 & 63.795  & 0.029 & 0.027 & 0.015 & 0.028 & 10.313 & 65.936&  \\
  1400 & 0.051 & 0.044 & 0.052 & 0.044 & 14.901 & 100.427 & 0.050 & 0.041 & 0.049 & 0.040 & 15.121 & 97.662 & \\
  1800 & 0.080 & 0.055 & 0.066 & 0.066 & 19.741 & 131.798 & 0.066 & 0.052 & 0.078 & 0.052 & 19.832 & 131.833& \\ \hline
  200  & 0.009 & 0.004 & 0.005 & 0.006 & 1.470  & 9.446   & 0.009 & 0.004 & 0.005 & 0.005 & 1.512  & 9.482   & \parbox[t]{2mm}{\multirow{5}{*}{\rotatebox[origin=c]{-90}{Control-2}}}\\ 
  600  & 0.020 & 0.009 & 0.018 & 0.011 & 3.987  & 22.574  & 0.019 & 0.008 & 0.019 & 0.010 & 4.131  & 22.891 & \\
  1000 & 0.032 & 0.027 & 0.017 & 0.027 & 10.141 & 64.102  & 0.028 & 0.026 & 0.016 & 0.028 & 10.824 & 64.826 & \\
  1400 & 0.050 & 0.043 & 0.048 & 0.044 & 14.911 & 96.500  & 0.047 & 0.042 & 0.058 & 0.041 & 15.006 & 98.096 & \\
  1800 & 0.076 & 0.053 & 0.094 & 0.054 & 19.096 & 127.115 & 0.067 & 0.056 & 0.065 & 0.053 & 19.709 & 136.821& \\ \hline
  200  & 0.009 & 0.005 & 0.005 & 0.009 & 1.504  & 9.935   & 0.009 & 0.004 & 0.005 & 0.006 & 1.484  & 9.857   & \parbox[t]{2mm}{\multirow{5}{*}{\rotatebox[origin=c]{-90}{IV}}}\\ 
  600  & 0.019 & 0.017 & 0.015 & 0.020 & 6.459  & 44.573  & 0.016 & 0.008 & 0.021 & 0.010 & 4.197  & 22.947 & \\
  1000 & 0.035 & 0.028 & 0.015 & 0.033 & 6.756  & 42.428  & 0.028 & 0.029 & 0.021 & 0.028 & 10.607 & 68.366 & \\
  1400 & 0.047 & 0.020 & 0.057 & 0.030 & 10.642 & 91.119  & 0.052 & 0.041 & 0.061 & 0.042 & 15.744 & 100.891& \\
  1800 & 0.092 & 0.065 & 0.097 & 0.072 & 25.276 & 143.655 & 0.055 & 0.061 & 0.082 & 0.065 & 19.885 & 143.689&\\
\hline
  \end{tabular}}
  \caption{Simulation results for $n=400$ and $SNR=1$.} 
  \label{simresults_n400}
\end{table}

\newpage

\begin{table}[h]
\resizebox{0.8\textwidth}{!}{%
\begin{tabular}{|r|rrrrrr|rrrrrr|r|}
  \hline
  &\multicolumn{6}{c|}{iid covariates}&\multicolumn{6}{c|}{correlated covariates}&\\
  \hline
  & \textbf{Lasso} & \textbf{pBA}& \textbf{I-pBA} &  \textbf{oBA} & \textbf{CV-pBA} & \textbf{CV-Lasso} & \textbf{Lasso} & \textbf{pBA}& \textbf{I-pBA} & \textbf{oBA} & \textbf{CV-pBA} & \textbf{CV-Lasso} &  \\
   \hline
  \textbf{p}&\multicolumn{12}{c|}{Mean Absolute Error}&\\
  \hline
  200 & 0.044 & 0.044 & 0.042&0.045 & 0.051 & 0.051 & 0.042 & 0.043 & 0.043&0.044 & 0.051 & 0.046 & \parbox[t]{2mm}{\multirow{5}{*}{\rotatebox[origin=c]{-90}{Control-1}}}\\ 
  600 & 0.044 & 0.043 & 0.046&0.044 & 0.062 & 0.069 & 0.047 & 0.046 & 0.048&0.048 & 0.060 & 0.064 &\\ 
  1000 & 0.051 & 0.045 &0.047& 0.049 & 0.073 & 0.088 & 0.051 & 0.048 &0.046& 0.049 & 0.071 & 0.076& \\ 
  1400 & 0.060 & 0.048 &0.045& 0.051 & 0.077 & 0.099 & 0.055 & 0.044 &0.049& 0.049 & 0.077 & 0.086& \\ 
  1800 & 0.072 & 0.051 &0.048& 0.055 & 0.083 & 0.112 & 0.065 & 0.049 &0.049& 0.054 & 0.079 & 0.097& \\  \hline
  200 & 0.045 & 0.042 & 0.042&0.043 & 0.050 & 0.051 & 0.043 & 0.042 & 0.043&0.043 & 0.048 & 0.046 & \parbox[t]{2mm}{\multirow{5}{*}{\rotatebox[origin=c]{-90}{Control-2}}}\\ 
  600 & 0.045 & 0.044 & 0.044&0.044 & 0.063 & 0.064 & 0.042 & 0.041 & 0.042&0.041 & 0.059 & 0.059 &\\ 
  1000 & 0.046 & 0.045 &0.047& 0.046 & 0.073 & 0.072 & 0.045 & 0.042 &0.044& 0.043 & 0.067 & 0.070& \\ 
  1400 & 0.050 & 0.046 &0.048& 0.046 & 0.078 & 0.079 & 0.048 & 0.045 &0.044& 0.046 & 0.073 & 0.077& \\ 
  1800 & 0.052 & 0.044 &0.044& 0.045 & 0.082 & 0.092 & 0.051 & 0.043 &0.049& 0.043 & 0.081 & 0.082& \\  \hline
  200 & 0.045 & 0.045 & 0.044&0.045 & 0.044 & 0.045 & 0.071 & 0.073 & 0.075&0.072 & 0.069 & 0.072  &  \parbox[t]{2mm}{\multirow{5}{*}{\rotatebox[origin=c]{-90}{IV}}}\\ 
  600 & 0.048 & 0.048 & 0.045&0.048 & 0.048 & 0.048 & 0.080 & 0.082 & 0.076&0.080 & 0.079 & 0.080 &\\ 
  1000 & 0.043 & 0.043 &0.045& 0.042 & 0.044 & 0.043 & 0.078 & 0.080 &0.080& 0.079 & 0.076 & 0.076& \\ 
  1400 & 0.045 & 0.045 &0.046& 0.045 & 0.045 & 0.045 & 0.079 & 0.081 &0.088& 0.080 & 0.078 & 0.078& \\ 
  1800 & 0.047 & 0.047 &0.048& 0.047 & 0.047 & 0.047 & 0.080 & 0.084 &0.081& 0.081 & 0.079 & 0.080& \\ 
  \hline
  \textbf{p}&\multicolumn{12}{c|}{Standard Deviation}&\\
  \hline
  200 & 0.054 & 0.054 & 0.053&0.055 & 0.062 & 0.059 & 0.053 & 0.053 & 0.053&0.053 & 0.064 & 0.056 & \parbox[t]{2mm}{\multirow{5}{*}{\rotatebox[origin=c]{-90}{Control-1}}}\\ 
  600 & 0.053 & 0.052 & 0.054&0.052 & 0.076 & 0.063 & 0.057 & 0.055 & 0.057&0.056 & 0.076 & 0.068 &\\ 
  1000 & 0.056 & 0.054 &0.055& 0.055 & 0.092 & 0.072 & 0.060 & 0.058 &0.053& 0.058 & 0.090 & 0.074& \\ 
  1400 & 0.060 & 0.054 &0.054& 0.055 & 0.098 & 0.073 & 0.059 & 0.051 &0.055& 0.055 & 0.100 & 0.074& \\ 
  1800 & 0.067 & 0.057 &0.054& 0.060 & 0.109 & 0.082 & 0.062 & 0.055 &0.056& 0.057 & 0.100 & 0.076& \\  \hline
  200 & 0.053 & 0.053 & 0.052&0.054 & 0.063 & 0.059 & 0.053 & 0.051 & 0.055&0.051 & 0.059 & 0.054 & \parbox[t]{2mm}{\multirow{5}{*}{\rotatebox[origin=c]{-90}{Control-2}}}\\ 
  600 & 0.055 & 0.055 & 0.056&0.055 & 0.079 & 0.065 & 0.053 & 0.051 & 0.053&0.052 & 0.075 & 0.064 &\\ 
  1000 & 0.058 & 0.057 &0.059& 0.057 & 0.090 & 0.069 & 0.057 & 0.053 &0.056& 0.054 & 0.084 & 0.072& \\ 
  1400 & 0.061 & 0.056 &0.060& 0.057 & 0.096 & 0.075 & 0.060 & 0.057 &0.055& 0.057 & 0.092 & 0.075& \\ 
  1800 & 0.061 & 0.055 &0.055& 0.055 & 0.103 & 0.079 & 0.062 & 0.054 &0.061& 0.054 & 0.103 & 0.077& \\  \hline
  200 & 0.056 & 0.056 & 0.054&0.056 & 0.054 & 0.056 & 0.088 & 0.089 & 0.095&0.088 & 0.082 & 0.087  &  \parbox[t]{2mm}{\multirow{5}{*}{\rotatebox[origin=c]{-90}{IV}}}\\ 
  600 & 0.059 & 0.059 & 0.055&0.059 & 0.057 & 0.059 & 0.098 & 0.101 & 0.094&0.099 & 0.089 & 0.096& \\ 
  1000 & 0.053 & 0.053 &0.056& 0.053 & 0.051 & 0.053 & 0.092 & 0.096 &0.099& 0.094 & 0.078 & 0.089& \\ 
  1400 & 0.056 & 0.057 &0.057& 0.057 & 0.053 & 0.055 & 0.090 & 0.096 &0.103& 0.095 & 0.079 & 0.090 &\\ 
  1800 & 0.057 & 0.059 &0.058& 0.058 & 0.055 & 0.057 & 0.096 & 0.104 &0.100& 0.100 & 0.083 & 0.096 &\\ 
        \hline
  \textbf{p}&\multicolumn{12}{c|}{Rejection Rate}&\\
  \hline
  200 & 0.040 & 0.048 & 0.054&0.054 & 0.040 & 0.090 & 0.064 & 0.060 & 0.054&0.064 & 0.066 & 0.076 & \parbox[t]{2mm}{\multirow{5}{*}{\rotatebox[origin=c]{-90}{Control-1}}}\\ 
  600 & 0.064 & 0.056 & 0.086&0.072 & 0.032 & 0.210 & 0.078 & 0.076 & 0.084&0.080 & 0.058 & 0.160 &\\ 
  1000 & 0.104 & 0.090 &0.078& 0.096 & 0.060 & 0.334 & 0.110 & 0.094 &0.072& 0.106 & 0.078 & 0.244& \\ 
  1400 & 0.170 & 0.068 &0.074& 0.098 & 0.056 & 0.410 & 0.134 & 0.076 &0.076& 0.092 & 0.050 & 0.288& \\ 
  1800 & 0.220 & 0.100 &0.102& 0.146 & 0.048 & 0.454 & 0.172 & 0.072 &0.104& 0.114 & 0.052 & 0.364& \\  \hline
  200 & 0.070 & 0.062 & 0.052&0.066 & 0.052 & 0.126 & 0.054 & 0.042 & 0.066&0.042 & 0.046 & 0.048 & \parbox[t]{2mm}{\multirow{5}{*}{\rotatebox[origin=c]{-90}{Control-2}}}\\ 
  600 & 0.076 & 0.072 & 0.072&0.092 & 0.066 & 0.200 & 0.056 & 0.060 & 0.064&0.060 & 0.054 & 0.148& \\ 
  1000 & 0.084 & 0.090 &0.084& 0.098 & 0.058 & 0.236 & 0.070 & 0.062 &0.066& 0.066 & 0.052 & 0.240& \\ 
  1400 & 0.088 & 0.066 &0.092& 0.078 & 0.058 & 0.288 & 0.088 & 0.068 &0.064& 0.074 & 0.060 & 0.262& \\ 
  1800 & 0.100 & 0.068 &0.066& 0.060 & 0.054 & 0.336 & 0.092 & 0.056 &0.088& 0.074 & 0.052 & 0.314& \\  \hline
  200 & 0.042 & 0.044 & 0.050&0.044 & 0.042 & 0.040 & 0.062 & 0.058 & 0.072&0.060 & 0.072 & 0.064  &  \parbox[t]{2mm}{\multirow{5}{*}{\rotatebox[origin=c]{-90}{IV}}}\\ 
  600 & 0.052 & 0.046 & 0.048&0.048 & 0.064 & 0.058 & 0.072 & 0.072 & 0.048&0.080 & 0.100 & 0.080 &\\ 
  1000 & 0.056 & 0.044 &0.052& 0.044 & 0.062 & 0.054 & 0.058 & 0.060 &0.066& 0.050 & 0.076 & 0.054& \\ 
  1400 & 0.056 & 0.054 &0.050& 0.064 & 0.072 & 0.054 & 0.062 & 0.046 &0.072& 0.054 & 0.098 & 0.064& \\ 
  1800 & 0.054 & 0.058 &0.052& 0.056 & 0.082 & 0.050 & 0.090 & 0.056 &0.060& 0.056 & 0.110 & 0.060& \\ 
  \hline
  \textbf{p}&\multicolumn{12}{c|}{Computation Time}&\\
  \hline
200  & 0.009 & 0.004 & 0.005 & 0.006 & 1.485  & 10.039  & 0.009 & 0.005 & 0.005 & 0.008 & 1.481  & 9.349   & \parbox[t]{2mm}{\multirow{5}{*}{\rotatebox[origin=c]{-90}{Control-1}}}\\
600  & 0.021 & 0.009 & 0.018 & 0.011 & 4.128  & 23.135  & 0.017 & 0.009 & 0.019 & 0.010 & 4.031  & 22.996&  \\
1000 & 0.034 & 0.028 & 0.015 & 0.029 & 10.037 & 64.201  & 0.027 & 0.030 & 0.018 & 0.029 & 10.237 & 64.420 & \\
1400 & 0.050 & 0.044 & 0.047 & 0.044 & 15.091 & 97.301  & 0.044 & 0.043 & 0.063 & 0.042 & 15.084 & 98.632  &\\
1800 & 0.067 & 0.054 & 0.065 & 0.055 & 19.509 & 128.696 & 0.061 & 0.054 & 0.080 & 0.054 & 19.521 & 130.089& \\ \hline
200  & 0.009 & 0.004 & 0.005 & 0.006 & 1.485  & 9.563   & 0.009 & 0.004 & 0.005 & 0.005 & 1.504  & 9.659   & \parbox[t]{2mm}{\multirow{5}{*}{\rotatebox[origin=c]{-90}{Control-2}}}\\
600  & 0.018 & 0.009 & 0.018 & 0.010 & 3.941  & 22.460  & 0.018 & 0.009 & 0.018 & 0.018 & 4.073  & 22.845 & \\
1000 & 0.029 & 0.028 & 0.016 & 0.030 & 9.959  & 64.120  & 0.027 & 0.028 & 0.016 & 0.029 & 10.039 & 64.334 & \\
1400 & 0.049 & 0.043 & 0.055 & 0.044 & 14.815 & 96.972  & 0.046 & 0.041 & 0.050 & 0.042 & 15.173 & 98.414 & \\
1800 & 0.065 & 0.055 & 0.068 & 0.055 & 19.333 & 129.302 & 0.061 & 0.053 & 0.091 & 0.054 & 19.484 & 130.222& \\ \hline
200  & 0.009 & 0.004 & 0.005 & 0.006 & 1.509  & 9.469   & 0.010 & 0.004 & 0.005 & 0.006 & 1.575  & 9.984   &  \parbox[t]{2mm}{\multirow{5}{*}{\rotatebox[origin=c]{-90}{IV}}} \\
600  & 0.020 & 0.016 & 0.011 & 0.017 & 6.261  & 43.109  & 0.016 & 0.008 & 0.018 & 0.010 & 4.050  & 24.002 & \\
1000 & 0.033 & 0.028 & 0.015 & 0.029 & 6.446  & 38.138  & 0.030 & 0.026 & 0.016 & 0.027 & 10.601 & 67.820 & \\
1400 & 0.051 & 0.026 & 0.048 & 0.031 & 9.239  & 81.808  & 0.050 & 0.036 & 0.049 & 0.044 & 15.391 & 97.780 & \\
1800 & 0.058 & 0.049 & 0.062 & 0.052 & 19.554 & 128.446 & 0.054 & 0.055 & 0.065 & 0.055 & 20.042 & 144.281&\\
\hline
  \end{tabular}}
  \caption{Simulation results for $n=400$ and $SNR=3$.} 
  \label{simresults_n400_SNR3}
\end{table}

\newpage

\begin{table}[h]
\resizebox{0.8\textwidth}{!}{%
\begin{tabular}{|r|rrrrrr|rrrrrr|r|}
  \hline
  &\multicolumn{6}{c|}{iid covariates}&\multicolumn{6}{c|}{correlated covariates}&\\
  \hline
 & \textbf{Lasso} & \textbf{pBA}& \textbf{I-pBA} &  \textbf{oBA} & \textbf{CV-pBA} & \textbf{CV-Lasso} & \textbf{Lasso} & \textbf{pBA}& \textbf{I-pBA} & \textbf{oBA} & \textbf{CV-pBA} & \textbf{CV-Lasso} &  \\ 
   \hline
  \textbf{p}&\multicolumn{12}{c|}{Mean Absolute Error}&\\
  \hline
  200 & 0.033 & 0.033 & 0.035&0.034 & 0.037 & 0.038 & 0.035 & 0.036 & 0.035&0.036 & 0.038 & 0.037 & \parbox[t]{2mm}{\multirow{5}{*}{\rotatebox[origin=c]{-90}{Control-1}}}\\ 
  600 & 0.033 & 0.034 & 0.037&0.034 & 0.042 & 0.050 & 0.035 & 0.036 & 0.035&0.036 & 0.041 & 0.044 &\\ 
  1000 & 0.037 & 0.036 &0.036& 0.036 & 0.049 & 0.059 & 0.036 & 0.036 &0.036& 0.037 & 0.044 & 0.052& \\ 
  1400 & 0.038 & 0.037 &0.038& 0.038 & 0.050 & 0.066 & 0.038 & 0.035 &0.035& 0.037 & 0.048 & 0.061& \\ 
  1800 & 0.042 & 0.037 &0.039& 0.038 & 0.053 & 0.073 & 0.041 & 0.037 &0.037& 0.038 & 0.050 & 0.066& \\  \hline
  200 & 0.036 & 0.036 & 0.033&0.036 & 0.039 & 0.039 & 0.035 & 0.035 & 0.033&0.035 & 0.038 & 0.035& \parbox[t]{2mm}{\multirow{5}{*}{\rotatebox[origin=c]{-90}{Control-2}}} \\ 
  600 & 0.036 & 0.034 & 0.035&0.034 & 0.043 & 0.048 & 0.036 & 0.035 & 0.036&0.035 & 0.044 & 0.044& \\ 
  1000 & 0.035 & 0.034 &0.034& 0.034 & 0.048 & 0.051 & 0.035 & 0.035 &0.035& 0.035 & 0.045 & 0.047& \\ 
  1400 & 0.038 & 0.037 &0.035& 0.038 & 0.052 & 0.056 & 0.037 & 0.037 &0.035& 0.038 & 0.047 & 0.053 &\\ 
  1800 & 0.036 & 0.035 &0.036& 0.035 & 0.052 & 0.059 & 0.036 & 0.035 &0.034& 0.035 & 0.050 & 0.054& \\  \hline
  200 & 0.045 & 0.044 & 0.044&0.044 & 0.043 & 0.044 & 0.079 & 0.079 & 0.079&0.078 & 0.070 & 0.077  &  \parbox[t]{2mm}{\multirow{5}{*}{\rotatebox[origin=c]{-90}{IV}}}\\ 
  600 & 0.046 & 0.045 & 0.045&0.046 & 0.045 & 0.047 & 0.091 & 0.086 & 0.087&0.086 & 0.073 & 0.098 &\\ 
  1000 & 0.043 & 0.042 &0.042& 0.042 & 0.044 & 0.044 & 0.093 & 0.090 &0.091& 0.090 & 0.073 & 0.208& \\ 
  1400 & 0.046 & 0.046 &0.045& 0.046 & 0.047 & 0.047 & 0.096 & 0.095 &0.095& 0.094 & 0.078 & 0.380& \\ 
  1800 & 0.046 & 0.046 &0.044& 0.046 & 0.048 & 0.047 & 0.098 & 0.102 &0.096& 0.100 & 0.080 & 0.895& \\ 
  \hline
  \textbf{p}&\multicolumn{12}{c|}{Standard Deviation}&\\
  \hline
  200 & 0.042 & 0.041 & 0.043&0.041 & 0.047 & 0.044 & 0.044 & 0.045 & 0.043&0.045 & 0.048 & 0.046 & \parbox[t]{2mm}{\multirow{5}{*}{\rotatebox[origin=c]{-90}{Control-1}}}\\ 
  600 & 0.040 & 0.040 & 0.045&0.040 & 0.052 & 0.046 & 0.043 & 0.043 & 0.044&0.043 & 0.051 & 0.046& \\ 
  1000 & 0.045 & 0.044 &0.044& 0.044 & 0.061 & 0.056 & 0.044 & 0.044 &0.042& 0.044 & 0.056 & 0.050& \\ 
  1400 & 0.045 & 0.044 &0.045& 0.045 & 0.064 & 0.053 & 0.044 & 0.042 &0.041& 0.043 & 0.061 & 0.055& \\ 
  1800 & 0.047 & 0.044 &0.044& 0.045 & 0.068 & 0.059 & 0.044 & 0.043 &0.044& 0.044 & 0.063 & 0.053& \\  \hline
  200 & 0.043 & 0.043 & 0.043&0.043 & 0.048 & 0.045 & 0.044 & 0.044 & 0.043&0.044 & 0.049 & 0.045 & \parbox[t]{2mm}{\multirow{5}{*}{\rotatebox[origin=c]{-90}{Control-2}}}\\ 
  600 & 0.044 & 0.044 & 0.044&0.044 & 0.054 & 0.054 & 0.044 & 0.044 & 0.044&0.044 & 0.055 & 0.050 &\\ 
  1000 & 0.044 & 0.043 &0.043& 0.043 & 0.059 & 0.054 & 0.044 & 0.043 &0.043& 0.044 & 0.057 & 0.051& \\ 
  1400 & 0.048 & 0.046 &0.044& 0.047 & 0.065 & 0.060 & 0.046 & 0.046 &0.044& 0.047 & 0.060 & 0.054 &\\ 
  1800 & 0.045 & 0.043 &0.044& 0.043 & 0.063 & 0.058 & 0.046 & 0.045 &0.043& 0.045 & 0.062 & 0.054& \\  \hline
  200 & 0.056 & 0.054 & 0.054&0.055 & 0.052 & 0.055 & 0.097 & 0.095 & 0.094&0.093 & 0.080 & 0.093  &  \parbox[t]{2mm}{\multirow{5}{*}{\rotatebox[origin=c]{-90}{IV}}}\\ 
  600 & 0.057 & 0.057 & 0.057&0.057 & 0.053 & 0.056 & 0.112 & 0.105 & 0.106&0.105 & 0.076 & 0.195 &\\ 
  1000 & 0.053 & 0.053 &0.052& 0.053 & 0.048 & 0.052 & 0.114 & 0.106 &0.107& 0.106 & 0.072 & 0.782 &\\ 
  1400 & 0.057 & 0.058 &0.057& 0.057 & 0.051 & 0.057 & 0.119 & 0.115 &0.111& 0.114 & 0.079 & 2.407& \\ 
  1800 & 0.057 & 0.058 &0.056& 0.058 & 0.052 & 0.057 & 0.117 & 0.119 &0.113& 0.117 & 0.077 & 5.709 &\\ 
        \hline
  \textbf{p}&\multicolumn{12}{c|}{Rejection Rate}&\\
  \hline
  200 & 0.056 & 0.058 & 0.064&0.056 & 0.048 & 0.070 & 0.064 & 0.072 & 0.066&0.072 & 0.064 & 0.072 & \parbox[t]{2mm}{\multirow{5}{*}{\rotatebox[origin=c]{-90}{Control-1}}}\\ 
  600 & 0.052 & 0.050 & 0.088&0.052 & 0.032 & 0.184 & 0.062 & 0.072 & 0.060&0.074 & 0.046 & 0.126 &\\ 
  1000 & 0.082 & 0.076 &0.072& 0.072 & 0.062 & 0.268 & 0.064 & 0.074 &0.058& 0.072 & 0.056 & 0.186& \\ 
  1400 & 0.080 & 0.076 &0.086& 0.084 & 0.052 & 0.314 & 0.074 & 0.064 &0.062& 0.068 & 0.060 & 0.248& \\ 
  1800 & 0.110 & 0.078 &0.072& 0.086 & 0.060 & 0.358 & 0.104 & 0.082 &0.070& 0.098 & 0.044 & 0.312& \\  \hline
  200 & 0.066 & 0.056 & 0.060&0.050 & 0.048 & 0.076 & 0.070 & 0.074 & 0.060&0.068 & 0.074 & 0.068 & \parbox[t]{2mm}{\multirow{5}{*}{\rotatebox[origin=c]{-90}{Control-2}}}\\ 
  600 & 0.062 & 0.064 & 0.060&0.064 & 0.062 & 0.176 & 0.074 & 0.068 & 0.054&0.060 & 0.054 & 0.124& \\ 
  1000 & 0.080 & 0.064 &0.056& 0.066 & 0.058 & 0.210 & 0.062 & 0.058 &0.052& 0.058 & 0.070 & 0.148& \\ 
  1400 & 0.088 & 0.068 &0.076& 0.084 & 0.084 & 0.206 & 0.088 & 0.074 &0.058& 0.086 & 0.064 & 0.224& \\ 
  1800 & 0.072 & 0.044 &0.066& 0.056 & 0.052 & 0.238 & 0.082 & 0.066 &0.054& 0.074 & 0.058 & 0.226 &\\  \hline
  200 & 0.046 & 0.042 & 0.042&0.044 & 0.052 & 0.048 & 0.070 & 0.070 & 0.062&0.072 & 0.088 & 0.074 &  \parbox[t]{2mm}{\multirow{5}{*}{\rotatebox[origin=c]{-90}{IV}}} \\ 
  600 & 0.066 & 0.066 & 0.046&0.072 & 0.110 & 0.068 & 0.072 & 0.074 & 0.056&0.082 & 0.122 & 0.084 &\\ 
  1000 & 0.044 & 0.052 &0.048& 0.052 & 0.062 & 0.046 & 0.068 & 0.054 &0.072& 0.054 & 0.116 & 0.052 &\\ 
  1400 & 0.066 & 0.072 &0.064& 0.068 & 0.090 & 0.072 & 0.070 & 0.074 &0.076& 0.078 & 0.126 & 0.066& \\ 
  1800 & 0.066 & 0.066 &0.052& 0.078 & 0.106 & 0.078 & 0.062 & 0.066 &0.078& 0.064 & 0.142 & 0.060& \\ 
  \hline
  \textbf{p}&\multicolumn{12}{c|}{Computation Time}&\\
  \hline
200  & 0.017 & 0.006 & 0.007 & 0.008 & 2.569  & 14.904  & 0.017 & 0.005 & 0.008 & 0.007 & 2.628  & 14.978 & \parbox[t]{2mm}{\multirow{5}{*}{\rotatebox[origin=c]{-90}{Control-1}}} \\
600  & 0.048 & 0.016 & 0.014 & 0.021 & 7.242  & 34.499  & 0.045 & 0.014 & 0.015 & 0.015 & 7.540  & 34.767 & \\
1000 & 0.084 & 0.029 & 0.026 & 0.031 & 17.650 & 83.566  & 0.082 & 0.024 & 0.034 & 0.025 & 12.618 & 72.873 & \\
1400 & 0.147 & 0.050 & 0.041 & 0.045 & 19.214 & 110.761 & 0.108 & 0.038 & 0.044 & 0.038 & 17.078 & 89.647 & \\
1800 & 0.213 & 0.096 & 0.066 & 0.096 & 23.171 & 140.031 & 0.188 & 0.091 & 0.076 & 0.089 & 23.637 & 145.171& \\ \hline
200  & 0.018 & 0.007 & 0.007 & 0.008 & 2.648  & 15.093  & 0.016 & 0.007 & 0.007 & 0.008 & 2.708  & 16.294& \parbox[t]{2mm}{\multirow{5}{*}{\rotatebox[origin=c]{-90}{Control-2}}}  \\
600  & 0.056 & 0.014 & 0.019 & 0.015 & 7.525  & 34.156  & 0.058 & 0.013 & 0.016 & 0.015 & 8.743  & 37.906  &\\
1000 & 0.084 & 0.022 & 0.033 & 0.026 & 12.087 & 72.064  & 0.083 & 0.032 & 0.026 & 0.030 & 13.759 & 88.049  &\\
1400 & 0.136 & 0.046 & 0.048 & 0.045 & 16.950 & 91.942  & 0.125 & 0.036 & 0.044 & 0.037 & 21.957 & 97.750  &\\
1800 & 0.220 & 0.099 & 0.058 & 0.100 & 26.852 & 157.284 & 0.193 & 0.095 & 0.067 & 0.091 & 23.613 & 145.881 &\\ \hline
200  & 0.021 & 0.006 & 0.007 & 0.008 & 2.620  & 14.720  & 0.014 & 0.005 & 0.009 & 0.006 & 2.699  & 14.693  &  \parbox[t]{2mm}{\multirow{5}{*}{\rotatebox[origin=c]{-90}{IV}}} \\
600  & 0.049 & 0.013 & 0.015 & 0.015 & 7.270  & 41.529  & 0.041 & 0.024 & 0.022 & 0.026 & 7.279  & 34.232 & \\
1000 & 0.081 & 0.047 & 0.033 & 0.054 & 18.754 & 88.751  & 0.072 & 0.022 & 0.028 & 0.023 & 16.718 & 82.223 & \\
1400 & 0.139 & 0.044 & 0.042 & 0.047 & 17.031 & 93.072  & 0.088 & 0.062 & 0.047 & 0.062 & 16.578 & 86.673 & \\
1800 & 0.188 & 0.090 & 0.062 & 0.087 & 33.180 & 241.334 & 0.189 & 0.090 & 0.066 & 0.084 & 23.154 & 142.315&\\
\hline
  \end{tabular}}
  \caption{Simulation results for $n=600$ and $SNR=1$.} 
\label{simresults_n600}
\end{table}

\newpage

\begin{table}[h]
\resizebox{0.8\textwidth}{!}{%
\begin{tabular}{|r|rrrrrr|rrrrrr|r|}
  \hline
  &\multicolumn{6}{c|}{iid covariates}&\multicolumn{6}{c|}{correlated covariates}&\\
  \hline
 & \textbf{Lasso} & \textbf{pBA}& \textbf{I-pBA} &  \textbf{oBA} & \textbf{CV-pBA} & \textbf{CV-Lasso} & \textbf{Lasso} & \textbf{pBA}& \textbf{I-pBA} & \textbf{oBA} & \textbf{CV-pBA} & \textbf{CV-Lasso} &  \\
   \hline
  \textbf{p}&\multicolumn{12}{c|}{Mean Absolute Error}&\\
  \hline
  200 & 0.033 & 0.034 & 0.035&0.034 & 0.038 & 0.039 & 0.035 & 0.035 & 0.035&0.035 & 0.038 & 0.037 & \parbox[t]{2mm}{\multirow{5}{*}{\rotatebox[origin=c]{-90}{Control-1}}}\\ 
  600 & 0.036 & 0.036 & 0.038&0.037 & 0.044 & 0.052 & 0.035 & 0.035 & 0.036&0.036 & 0.042 & 0.046 &\\ 
  1000 & 0.039 & 0.039 &0.036& 0.039 & 0.050 & 0.061 & 0.035 & 0.035 &0.037& 0.036 & 0.043 & 0.050&\\ 
  1400 & 0.039 & 0.036 &0.038& 0.038 & 0.048 & 0.071 & 0.039 & 0.036 &0.036& 0.038 & 0.048 & 0.061 &\\ 
  1800 & 0.040 & 0.036 &0.038& 0.037 & 0.053 & 0.071 & 0.041 & 0.038 &0.036& 0.039 & 0.053 & 0.066 &\\  \hline
  200 & 0.033 & 0.033 & 0.034&0.034 & 0.038 & 0.038 & 0.035 & 0.035 & 0.035&0.035 & 0.037 & 0.037& \parbox[t]{2mm}{\multirow{5}{*}{\rotatebox[origin=c]{-90}{Control-2}}} \\ 
  600 & 0.034 & 0.034 & 0.035&0.035 & 0.043 & 0.048 & 0.035 & 0.033 & 0.035&0.034 & 0.040 & 0.043& \\ 
  1000 & 0.035 & 0.034 &0.035& 0.034 & 0.045 & 0.057 & 0.034 & 0.034 &0.037& 0.035 & 0.045 & 0.050& \\ 
  1400 & 0.034 & 0.033 &0.035& 0.034 & 0.049 & 0.061 & 0.036 & 0.035 &0.034& 0.035 & 0.052 & 0.054 &\\ 
  1800 & 0.038 & 0.036 &0.036& 0.037 & 0.054 & 0.065 & 0.036 & 0.035 &0.035& 0.036 & 0.049 & 0.057& \\  \hline
  200 & 0.035 & 0.035 & 0.039&0.035 & 0.034 & 0.035 & 0.058 & 0.058 & 0.059&0.058 & 0.058 & 0.059  &  \parbox[t]{2mm}{\multirow{5}{*}{\rotatebox[origin=c]{-90}{IV}}}\\ 
  600 & 0.036 & 0.036 & 0.040&0.036 & 0.037 & 0.036 & 0.058 & 0.060 & 0.064&0.059 & 0.058 & 0.059 &\\ 
  1000 & 0.038 & 0.038 &0.038& 0.038 & 0.038 & 0.038 & 0.062 & 0.062 &0.063& 0.061 & 0.062 & 0.062& \\ 
  1400 & 0.037 & 0.037 &0.037& 0.037 & 0.037 & 0.037 & 0.061 & 0.060 &0.058& 0.060 & 0.063 & 0.061& \\ 
  1800 & 0.037 & 0.038 &0.036& 0.038 & 0.038 & 0.037 & 0.062 & 0.061 &0.064& 0.062 & 0.064 & 0.063& \\ 
  \hline
  \textbf{p}&\multicolumn{12}{c|}{Standard Deviation}&\\
  \hline
  200 & 0.042 & 0.043 & 0.042&0.043 & 0.048 & 0.046 & 0.042 & 0.043 & 0.043&0.043 & 0.046 & 0.045 & \parbox[t]{2mm}{\multirow{5}{*}{\rotatebox[origin=c]{-90}{Control-1}}}\\ 
  600 & 0.044 & 0.044 & 0.045&0.045 & 0.056 & 0.050 & 0.044 & 0.043 & 0.043&0.043 & 0.053 & 0.049& \\ 
  1000 & 0.047 & 0.046 &0.043& 0.045 & 0.063 & 0.054 & 0.042 & 0.042 &0.044& 0.042 & 0.053 & 0.047& \\ 
  1400 & 0.044 & 0.043 &0.044& 0.044 & 0.060 & 0.055 & 0.045 & 0.043 &0.044& 0.044 & 0.061 & 0.052& \\ 
  1800 & 0.044 & 0.042 &0.045& 0.043 & 0.066 & 0.055 & 0.047 & 0.045 &0.042& 0.046 & 0.067 & 0.055& \\  \hline
  200 & 0.041 & 0.041 & 0.043&0.043 & 0.047 & 0.045 & 0.044 & 0.044 & 0.045&0.044 & 0.046 & 0.045 & \parbox[t]{2mm}{\multirow{5}{*}{\rotatebox[origin=c]{-90}{Control-2}}}\\ 
  600 & 0.041 & 0.042 & 0.044&0.043 & 0.053 & 0.049 & 0.043 & 0.042 & 0.044&0.043 & 0.051 & 0.048& \\ 
  1000 & 0.044 & 0.043 &0.043& 0.043 & 0.057 & 0.055 & 0.043 & 0.043 &0.046& 0.044 & 0.057 & 0.052& \\ 
  1400 & 0.043 & 0.043 &0.045& 0.042 & 0.062 & 0.054 & 0.045 & 0.044 &0.042& 0.044 & 0.064 & 0.053& \\ 
  1800 & 0.047 & 0.045 &0.045& 0.046 & 0.066 & 0.058 & 0.045 & 0.043 &0.044& 0.044 & 0.059 & 0.054& \\  \hline
  200 & 0.043 & 0.043 & 0.049&0.043 & 0.042 & 0.043 & 0.072 & 0.072 & 0.073&0.072 & 0.070 & 0.073 &  \parbox[t]{2mm}{\multirow{5}{*}{\rotatebox[origin=c]{-90}{IV}}} \\ 
  600 & 0.045 & 0.044 & 0.048&0.045 & 0.044 & 0.044 & 0.072 & 0.073 & 0.080&0.072 & 0.069 & 0.073& \\ 
  1000 & 0.047 & 0.047 &0.047& 0.047 & 0.045 & 0.047 & 0.076 & 0.077 &0.077& 0.076 & 0.070 & 0.074& \\ 
  1400 & 0.047 & 0.047 &0.046& 0.047 & 0.045 & 0.046 & 0.075 & 0.076 &0.072& 0.076 & 0.069 & 0.075& \\ 
  1800 & 0.046 & 0.047 &0.045& 0.047 & 0.045 & 0.047 & 0.074 & 0.076 &0.079& 0.076 & 0.068 & 0.075& \\ 
        \hline
  \textbf{p}&\multicolumn{12}{c|}{Rejection Rate}&\\
  \hline
  200 & 0.056 & 0.060 & 0.060&0.064 & 0.060 & 0.086 & 0.050 & 0.046 & 0.064&0.050 & 0.050 & 0.060 & \parbox[t]{2mm}{\multirow{5}{*}{\rotatebox[origin=c]{-90}{Control-1}}}\\ 
  600 & 0.048 & 0.058 & 0.082&0.060 & 0.052 & 0.178 & 0.066 & 0.068 & 0.070&0.060 & 0.064 & 0.134 &\\ 
  1000 & 0.106 & 0.082 &0.070& 0.088 & 0.060 & 0.276 & 0.052 & 0.052 &0.076& 0.062 & 0.038 & 0.178 &\\ 
  1400 & 0.102 & 0.080 &0.080& 0.096 & 0.038 & 0.348 & 0.094 & 0.080 &0.080& 0.094 & 0.054 & 0.258 &\\ 
  1800 & 0.080 & 0.056 &0.094& 0.060 & 0.048 & 0.350 & 0.104 & 0.084 &0.076& 0.106 & 0.066 & 0.298 &\\  \hline
  200 & 0.068 & 0.052 & 0.064&0.054 & 0.048 & 0.074 & 0.072 & 0.074 & 0.080&0.084 & 0.048 & 0.072 & \parbox[t]{2mm}{\multirow{5}{*}{\rotatebox[origin=c]{-90}{Control-2}}}\\ 
  600 & 0.054 & 0.050 & 0.066&0.056 & 0.044 & 0.158 & 0.060 & 0.066 & 0.070&0.058 & 0.056 & 0.136 &\\ 
  1000 & 0.076 & 0.066 &0.048& 0.064 & 0.046 & 0.232 & 0.062 & 0.064 &0.082& 0.066 & 0.064 & 0.176 &\\ 
  1400 & 0.064 & 0.060 &0.066& 0.066 & 0.050 & 0.256 & 0.060 & 0.056 &0.046& 0.058 & 0.058 & 0.200 &\\ 
  1800 & 0.082 & 0.056 &0.064& 0.064 & 0.060 & 0.284 & 0.064 & 0.056 &0.064& 0.058 & 0.040 & 0.220 &\\  \hline
  200 & 0.038 & 0.036 & 0.066&0.040 & 0.038 & 0.038 & 0.054 & 0.050 & 0.044&0.058 & 0.072 & 0.060 &  \parbox[t]{2mm}{\multirow{5}{*}{\rotatebox[origin=c]{-90}{IV}}} \\ 
  600 & 0.034 & 0.040 & 0.052&0.036 & 0.050 & 0.036 & 0.042 & 0.048 & 0.084&0.050 & 0.068 & 0.050 &\\ 
  1000 & 0.052 & 0.054 &0.054& 0.052 & 0.064 & 0.056 & 0.058 & 0.052 &0.058& 0.050 & 0.088 & 0.056& \\ 
  1400 & 0.058 & 0.044 &0.044& 0.050 & 0.074 & 0.058 & 0.058 & 0.048 &0.040& 0.052 & 0.082 & 0.052 &\\ 
  1800 & 0.052 & 0.048 &0.038& 0.050 & 0.064 & 0.046 & 0.060 & 0.050 &0.056& 0.056 & 0.104 & 0.060 &\\ 
  \hline
  \textbf{p}&\multicolumn{12}{c|}{Computation Time}&\\
  \hline
200  & 0.018 & 0.011 & 0.006 & 0.015 & 2.632  & 15.112  & 0.016 & 0.011 & 0.007 & 0.017 & 2.644  & 16.353  & \parbox[t]{2mm}{\multirow{5}{*}{\rotatebox[origin=c]{-90}{Control-1}}}\\
600  & 0.052 & 0.014 & 0.014 & 0.014 & 7.264  & 33.826  & 0.051 & 0.026 & 0.014 & 0.025 & 7.998  & 38.480&  \\
1000 & 0.066 & 0.027 & 0.027 & 0.031 & 11.879 & 71.832  & 0.090 & 0.027 & 0.034 & 0.027 & 15.192 & 94.239 & \\
1400 & 0.111 & 0.045 & 0.048 & 0.043 & 16.823 & 88.479  & 0.129 & 0.057 & 0.050 & 0.043 & 21.691 & 112.505& \\
1800 & 0.187 & 0.109 & 0.066 & 0.102 & 22.687 & 139.379 & 0.191 & 0.102 & 0.085 & 0.103 & 26.678 & 160.596 &\\ \hline
200  & 0.016 & 0.008 & 0.007 & 0.010 & 2.619  & 15.614  & 0.015 & 0.007 & 0.008 & 0.010 & 2.637  & 15.442 & \parbox[t]{2mm}{\multirow{5}{*}{\rotatebox[origin=c]{-90}{Control-2}}} \\
600  & 0.054 & 0.021 & 0.015 & 0.019 & 7.557  & 35.584  & 0.048 & 0.013 & 0.015 & 0.018 & 7.577  & 35.965 & \\
1000 & 0.071 & 0.025 & 0.026 & 0.033 & 12.927 & 76.753  & 0.076 & 0.031 & 0.034 & 0.026 & 13.128 & 77.762 & \\
1400 & 0.127 & 0.048 & 0.046 & 0.042 & 18.283 & 107.774 & 0.120 & 0.038 & 0.063 & 0.048 & 18.235 & 108.451& \\
1800 & 0.202 & 0.104 & 0.062 & 0.106 & 25.190 & 150.992 & 0.186 & 0.101 & 0.077 & 0.102 & 25.740 & 157.942 &\\ \hline
200  & 0.015 & 0.006 & 0.007 & 0.007 & 2.587  & 14.741  & 0.015 & 0.005 & 0.007 & 0.006 & 2.651  & 14.758  &  \parbox[t]{2mm}{\multirow{5}{*}{\rotatebox[origin=c]{-90}{IV}}} \\
600  & 0.045 & 0.027 & 0.014 & 0.028 & 7.215  & 33.472  & 0.040 & 0.023 & 0.015 & 0.024 & 7.098  & 33.478  &\\
1000 & 0.080 & 0.045 & 0.027 & 0.045 & 17.049 & 82.625  & 0.072 & 0.023 & 0.026 & 0.027 & 16.565 & 82.740  &\\
1400 & 0.133 & 0.036 & 0.041 & 0.037 & 16.372 & 85.120  & 0.088 & 0.062 & 0.045 & 0.063 & 16.539 & 86.188  &\\
1800 & 0.189 & 0.088 & 0.058 & 0.083 & 31.102 & 211.774 & 0.170 & 0.090 & 0.068 & 0.084 & 22.727 & 141.746&\\
\hline
  \end{tabular}}
  \caption{Simulation results for $n=600$ and $SNR=3$.} 
\label{simresults_n600_SNR3}
\end{table}

\pagebreak


\begin{table}[h]
\resizebox{0.8\textwidth}{!}{%
  \begin{tabular}{|r|rrrrrr|rrrrrr|r|}
  \hline
  &\multicolumn{6}{c|}{iid covariates}&\multicolumn{6}{c|}{correlated covariates}&\\
  \hline
 & \textbf{Lasso} & \textbf{pBA}& \textbf{I-pBA} &  \textbf{oBA} & \textbf{CV-pBA} & \textbf{CV-Lasso} & \textbf{Lasso} & \textbf{pBA}& \textbf{I-pBA} & \textbf{oBA} & \textbf{CV-pBA} & \textbf{CV-Lasso} &  \\ 
  \hline
  \textbf{p}&\multicolumn{12}{c|}{Mean Absolute Error}&\\
  \hline
  200 & 0.029 & 0.030 & 0.028&0.030 & 0.031 & 0.032 & 0.030 & 0.030 & 0.027&0.031 & 0.032 & 0.032 & \parbox[t]{2mm}{\multirow{5}{*}{\rotatebox[origin=c]{-90}{Control-1}}}\\  
  600 & 0.028 & 0.028 & 0.030&0.028 & 0.033 & 0.038 & 0.030 & 0.030 & 0.030&0.031 & 0.034 & 0.037& \\ 
  1000 & 0.032 & 0.032 &0.030& 0.032 & 0.038 & 0.047 & 0.029 & 0.029 &0.031& 0.030 & 0.036 & 0.041& \\ 
  1400 & 0.033 & 0.032 &0.030& 0.033 & 0.042 & 0.051 & 0.031 & 0.031 &0.030& 0.032 & 0.035 & 0.047 &\\ 
  1800 & 0.032 & 0.031 &0.032& 0.032 & 0.039 & 0.058 & 0.029 & 0.028 &0.032& 0.029 & 0.039 & 0.049& \\ \hline
  200 & 0.031 & 0.029 & 0.028&0.029 & 0.032 & 0.031 & 0.028 & 0.029 & 0.029&0.029 & 0.031 & 0.029& \parbox[t]{2mm}{\multirow{5}{*}{\rotatebox[origin=c]{-90}{Control-2}}}\\  
  600 & 0.029 & 0.028 & 0.031&0.028 & 0.032 & 0.035 & 0.029 & 0.029 & 0.028&0.029 & 0.033 & 0.034 &\\ 
  1000 & 0.031 & 0.031 &0.029& 0.031 & 0.037 & 0.041 & 0.031 & 0.030 &0.029& 0.031 & 0.036 & 0.037& \\ 
  1400 & 0.031 & 0.030 &0.030& 0.030 & 0.040 & 0.044 & 0.029 & 0.029 &0.031& 0.029 & 0.039 & 0.038& \\ 
  1800 & 0.030 & 0.029 &0.030& 0.030 & 0.041 & 0.047 & 0.030 & 0.029 &0.030& 0.029 & 0.038 & 0.041& \\ \hline
  200 & 0.038 & 0.037 & 0.042&0.037 & 0.038 & 0.038 & 0.060 & 0.061 & 0.062&0.061 & 0.056 & 0.059& \parbox[t]{2mm}{\multirow{5}{*}{\rotatebox[origin=c]{-90}{IV}}} \\ 
  600 & 0.038 & 0.038 & 0.039&0.038 & 0.039 & 0.039 & 0.073 & 0.072 & 0.069&0.071 & 0.064 & 0.070& \\ 
  1000 & 0.041 & 0.041 &0.040& 0.041 & 0.040 & 0.041 & 0.078 & 0.076 &0.072& 0.075 & 0.071 & 0.075& \\ 
  1400 & 0.041 & 0.041 &0.039& 0.042 & 0.042 & 0.042 & 0.079 & 0.078 &0.075& 0.076 & 0.071 & 0.242& \\ 
  1800 & 0.039 & 0.039 &0.041& 0.039 & 0.041 & 0.040 & 0.074 & 0.074 &0.077& 0.073 & 0.070 & 0.078& \\ \hline
  \textbf{p}&\multicolumn{12}{c|}{Standard Deviation}&\\
  \hline
  200 & 0.036 & 0.037 & 0.035&0.037 & 0.039 & 0.037 & 0.037 & 0.037 & 0.034&0.037 & 0.041 & 0.039 & \parbox[t]{2mm}{\multirow{5}{*}{\rotatebox[origin=c]{-90}{Control-1}}}\\ 
  600 & 0.036 & 0.035 & 0.035&0.035 & 0.041 & 0.038 & 0.036 & 0.036 & 0.036&0.036 & 0.042 & 0.039& \\ 
  1000 & 0.038 & 0.037 &0.036& 0.037 & 0.047 & 0.042 & 0.036 & 0.037 &0.038& 0.037 & 0.045 & 0.042& \\ 
  1400 & 0.039 & 0.038 &0.036& 0.039 & 0.052 & 0.045 & 0.037 & 0.036 &0.037& 0.037 & 0.045 & 0.043& \\ 
  1800 & 0.037 & 0.036 &0.036& 0.036 & 0.050 & 0.045 & 0.035 & 0.034 &0.037& 0.034 & 0.048 & 0.042& \\ \hline
  200 & 0.037 & 0.037 & 0.036&0.037 & 0.040 & 0.038 & 0.036 & 0.036 & 0.036&0.036 & 0.039 & 0.036 & \parbox[t]{2mm}{\multirow{5}{*}{\rotatebox[origin=c]{-90}{Control-2}}}\\  
  600 & 0.035 & 0.035 & 0.038&0.035 & 0.041 & 0.038 & 0.037 & 0.036 & 0.035&0.036 & 0.042 & 0.039& \\ 
  1000 & 0.038 & 0.038 &0.037& 0.038 & 0.046 & 0.045 & 0.038 & 0.038 &0.037& 0.038 & 0.046 & 0.041& \\ 
  1400 & 0.038 & 0.037 &0.037& 0.037 & 0.050 & 0.045 & 0.037 & 0.037 &0.038& 0.037 & 0.048 & 0.042& \\ 
  1800 & 0.038 & 0.037 &0.038& 0.037 & 0.050 & 0.048 & 0.037 & 0.036 &0.038& 0.036 & 0.046 & 0.044& \\ \hline
  200 & 0.047 & 0.046 & 0.052&0.046 & 0.045 & 0.046 & 0.073 & 0.074 & 0.077&0.074 & 0.065 & 0.071& \parbox[t]{2mm}{\multirow{5}{*}{\rotatebox[origin=c]{-90}{IV}}} \\ 
  600 & 0.047 & 0.047 & 0.049&0.047 & 0.044 & 0.047 & 0.090 & 0.086 & 0.081&0.086 & 0.070 & 0.085& \\ 
  1000 & 0.050 & 0.051 &0.051& 0.051 & 0.046 & 0.050 & 0.091 & 0.088 &0.085& 0.087 & 0.069 & 0.086& \\ 
  1400 & 0.050 & 0.050 &0.049& 0.051 & 0.047 & 0.050 & 0.094 & 0.090 &0.088& 0.088 & 0.066 & 3.653& \\ 
  1800 & 0.049 & 0.049 &0.051& 0.049 & 0.045 & 0.049 & 0.086 & 0.086 &0.087& 0.084 & 0.063 & 0.117& \\ 
  \hline
  \textbf{p}&\multicolumn{12}{c|}{Rejection Rate}&\\
  \hline
  200 & 0.058 & 0.056 & 0.048&0.058 & 0.054 & 0.060 & 0.042 & 0.060 & 0.038&0.064 & 0.060 & 0.074 & \parbox[t]{2mm}{\multirow{5}{*}{\rotatebox[origin=c]{-90}{Control-1}}}\\ 
  600 & 0.054 & 0.056 & 0.060&0.054 & 0.036 & 0.136 & 0.050 & 0.062 & 0.064&0.052 & 0.050 & 0.122& \\ 
  1000 & 0.068 & 0.064 &0.048& 0.068 & 0.046 & 0.228 & 0.058 & 0.060 &0.062& 0.064 & 0.044 & 0.172& \\ 
  1400 & 0.084 & 0.078 &0.052& 0.086 & 0.068 & 0.272 & 0.062 & 0.062 &0.062& 0.062 & 0.044 & 0.206& \\ 
  1800 & 0.086 & 0.074 &0.062& 0.080 & 0.052 & 0.336 & 0.078 & 0.066 &0.078& 0.070 & 0.050 & 0.240& \\ \hline
  200 & 0.080 & 0.058 & 0.046&0.060 & 0.052 & 0.060 & 0.052 & 0.054 & 0.042&0.058 & 0.060 & 0.060 & \parbox[t]{2mm}{\multirow{5}{*}{\rotatebox[origin=c]{-90}{Control-2}}}\\  
  600 & 0.042 & 0.042 & 0.062&0.040 & 0.040 & 0.098 & 0.060 & 0.054 & 0.046&0.058 & 0.046 & 0.110 &\\ 
  1000 & 0.070 & 0.072 &0.056& 0.062 & 0.050 & 0.148 & 0.070 & 0.048 &0.048& 0.056 & 0.072 & 0.118& \\ 
  1400 & 0.056 & 0.048 &0.054& 0.046 & 0.070 & 0.210 & 0.058 & 0.062 &0.060& 0.066 & 0.060 & 0.148& \\ 
  1800 & 0.054 & 0.052 &0.066& 0.056 & 0.056 & 0.210 & 0.058 & 0.064 &0.068& 0.056 & 0.044 & 0.164& \\ \hline
  200 & 0.046 & 0.048 & 0.054&0.050 & 0.056 & 0.048 & 0.042 & 0.052 & 0.060&0.048 & 0.060 & 0.042 & \parbox[t]{2mm}{\multirow{5}{*}{\rotatebox[origin=c]{-90}{IV}}}\\ 
  600 & 0.046 & 0.044 & 0.052&0.050 & 0.056 & 0.048 & 0.070 & 0.076 & 0.058&0.076 & 0.122 & 0.096& \\ 
  1000 & 0.054 & 0.058 &0.064& 0.056 & 0.082 & 0.054 & 0.072 & 0.084 &0.056& 0.078 & 0.154 & 0.088& \\ 
  1400 & 0.068 & 0.070 &0.052& 0.074 & 0.112 & 0.078 & 0.056 & 0.060 &0.066& 0.062 & 0.136 & 0.058& \\ 
  1800 & 0.070 & 0.066 &0.070& 0.064 & 0.098 & 0.072 & 0.064 & 0.058 &0.060& 0.060 & 0.142 & 0.062& \\ 
  \hline
  \textbf{p}&\multicolumn{12}{c|}{Computation Time}&\\
  \hline
  200  & 0.025 & 0.007 & 0.008 & 0.009 & 2.655  & 18.334  & 0.020 & 0.006 & 0.009 & 0.008 & 2.704  & 18.442  & \parbox[t]{2mm}{\multirow{5}{*}{\rotatebox[origin=c]{-90}{Control-1}}}\\
  600  & 0.071 & 0.021 & 0.019 & 0.022 & 5.766  & 49.771  & 0.064 & 0.018 & 0.022 & 0.019 & 5.334  & 48.332&  \\
  1000 & 0.125 & 0.037 & 0.042 & 0.040 & 17.081 & 88.104  & 0.107 & 0.033 & 0.043 & 0.035 & 15.216 & 81.753 & \\
  1400 & 0.164 & 0.061 & 0.113 & 0.063 & 24.334 & 217.421 & 0.176 & 0.056 & 0.114 & 0.058 & 23.137 & 212.172& \\
  1800 & 0.241 & 0.087 & 0.179 & 0.088 & 33.859 & 203.568 & 0.206 & 0.076 & 0.180 & 0.077 & 33.169 & 208.790& \\ \hline
  200  & 0.025 & 0.010 & 0.008 & 0.012 & 2.709  & 18.179  & 0.021 & 0.006 & 0.009 & 0.008 & 2.641  & 18.279  & \parbox[t]{2mm}{\multirow{5}{*}{\rotatebox[origin=c]{-90}{Control-2}}}\\
  600  & 0.068 & 0.019 & 0.020 & 0.020 & 5.316  & 47.188  & 0.058 & 0.018 & 0.023 & 0.019 & 5.273  & 48.248  &\\
  1000 & 0.116 & 0.038 & 0.040 & 0.037 & 15.084 & 80.883  & 0.100 & 0.033 & 0.043 & 0.034 & 15.320 & 82.305&  \\
  1400 & 0.179 & 0.057 & 0.120 & 0.059 & 23.083 & 205.908 & 0.166 & 0.058 & 0.135 & 0.060 & 23.082 & 216.171& \\
  1800 & 0.220 & 0.083 & 0.158 & 0.080 & 30.138 & 195.389 & 0.250 & 0.076 & 0.169 & 0.076 & 37.373 & 207.889 &\\ \hline
  200  & 0.028 & 0.007 & 0.009 & 0.010 & 2.733  & 18.698  & 0.028 & 0.008 & 0.008 & 0.012 & 2.802  & 19.098 & \parbox[t]{2mm}{\multirow{5}{*}{\rotatebox[origin=c]{-90}{IV}}} \\
  600  & 0.065 & 0.019 & 0.022 & 0.020 & 12.978 & 47.585  & 0.066 & 0.020 & 0.022 & 0.021 & 13.616 & 50.228 & \\
  1000 & 0.108 & 0.062 & 0.041 & 0.061 & 22.438 & 81.243  & 0.100 & 0.033 & 0.042 & 0.036 & 23.171 & 86.667 & \\
  1400 & 0.158 & 0.109 & 0.135 & 0.104 & 39.719 & 296.523 & 0.142 & 0.058 & 0.153 & 0.055 & 30.675 & 225.399& \\
  1800 & 0.260 & 0.084 & 0.174 & 0.092 & 41.156 & 202.758 & 0.185 & 0.140 & 0.195 & 0.134 & 38.929 & 212.553&\\
  \hline
  \end{tabular}}
\caption{Simulation results for $n=800$ and $SNR=1$.} 
\label{simresults_n800}
\end{table}

\pagebreak

\begin{table}[h]
\resizebox{0.8\textwidth}{!}{%
  \begin{tabular}{|r|rrrrrr|rrrrrr|r|}
  \hline
  &\multicolumn{6}{c|}{iid covariates}&\multicolumn{6}{c|}{correlated covariates}&\\
  \hline
 & \textbf{Lasso} & \textbf{pBA}& \textbf{I-pBA} &  \textbf{oBA} & \textbf{CV-pBA} & \textbf{CV-Lasso} & \textbf{Lasso} & \textbf{pBA}& \textbf{I-pBA} & \textbf{oBA} & \textbf{CV-pBA} & \textbf{CV-Lasso} &  \\
  \hline
  \textbf{p}&\multicolumn{12}{c|}{Mean Absolute Error}&\\
  \hline
  200 & 0.029 & 0.029 & 0.030&0.029 & 0.032 & 0.033 & 0.029 & 0.028 & 0.027&0.029 & 0.030 & 0.030 & \parbox[t]{2mm}{\multirow{5}{*}{\rotatebox[origin=c]{-90}{Control-1}}}\\ 
  600 & 0.029 & 0.029 & 0.032&0.030 & 0.034 & 0.039 & 0.031 & 0.031 & 0.029&0.031 & 0.034 & 0.039& \\ 
  1000 & 0.031 & 0.032 &0.029& 0.032 & 0.037 & 0.047 & 0.029 & 0.030 &0.029& 0.030 & 0.035 & 0.041& \\ 
  1400 & 0.030 & 0.030 &0.030& 0.030 & 0.039 & 0.051 & 0.030 & 0.030 &0.031& 0.030 & 0.037 & 0.046& \\ 
  1800 & 0.031 & 0.031 &0.031& 0.031 & 0.040 & 0.054 & 0.033 & 0.032 &0.032& 0.033 & 0.039 & 0.052& \\ \hline
  200 & 0.030 & 0.030 & 0.030&0.029 & 0.032 & 0.032 & 0.030 & 0.031 & 0.030&0.030 & 0.032 & 0.031 & \parbox[t]{2mm}{\multirow{5}{*}{\rotatebox[origin=c]{-90}{Control-2}}}\\ 
  600 & 0.030 & 0.029 & 0.029&0.029 & 0.037 & 0.039 & 0.031 & 0.030 & 0.028&0.030 & 0.034 & 0.034& \\ 
  1000 & 0.032 & 0.031 &0.032& 0.031 & 0.040 & 0.045 & 0.028 & 0.029 &0.029& 0.030 & 0.038 & 0.039& \\ 
  1400 & 0.030 & 0.030 &0.031& 0.030 & 0.038 & 0.045 & 0.030 & 0.029 &0.030& 0.029 & 0.036 & 0.043& \\ 
  1800 & 0.030 & 0.030 &0.030& 0.030 & 0.040 & 0.047 & 0.029 & 0.029 &0.030& 0.028 & 0.040 & 0.048& \\ \hline
  200 & 0.031 & 0.031 & 0.031&0.031 & 0.031 & 0.031 & 0.053 & 0.053 & 0.050&0.053 & 0.052 & 0.053& \parbox[t]{2mm}{\multirow{5}{*}{\rotatebox[origin=c]{-90}{IV}}} \\ 
  600 & 0.034 & 0.034 & 0.032&0.034 & 0.034 & 0.034 & 0.053 & 0.053 & 0.052&0.053 & 0.053 & 0.053& \\ 
  1000 & 0.034 & 0.034 &0.033& 0.035 & 0.035 & 0.034 & 0.055 & 0.056 &0.056& 0.055 & 0.055 & 0.055& \\ 
  1400 & 0.031 & 0.031 &0.032& 0.031 & 0.031 & 0.031 & 0.056 & 0.055 &0.054& 0.055 & 0.056 & 0.056 &\\ 
  1800 & 0.032 & 0.032 &0.031& 0.032 & 0.033 & 0.032 & 0.054 & 0.054 &0.053& 0.054 & 0.053 & 0.053& \\ \hline
  \textbf{p}&\multicolumn{12}{c|}{Standard Deviation}&\\
  \hline
  200 & 0.037 & 0.037 & 0.038&0.037 & 0.040 & 0.040 & 0.035 & 0.035 & 0.034&0.035 & 0.038 & 0.036 & \parbox[t]{2mm}{\multirow{5}{*}{\rotatebox[origin=c]{-90}{Control-1}}}\\ 
  600 & 0.036 & 0.036 & 0.038&0.036 & 0.043 & 0.040 & 0.037 & 0.037 & 0.037&0.037 & 0.042 & 0.040& \\ 
  1000 & 0.037 & 0.038 &0.035& 0.038 & 0.046 & 0.044 & 0.036 & 0.036 &0.035& 0.037 & 0.043 & 0.041& \\ 
  1400 & 0.036 & 0.036 &0.036& 0.036 & 0.048 & 0.044 & 0.037 & 0.037 &0.036& 0.037 & 0.047 & 0.043 &\\ 
  1800 & 0.038 & 0.038 &0.037& 0.038 & 0.051 & 0.047 & 0.039 & 0.038 &0.038& 0.039 & 0.050 & 0.046& \\ \hline
  200 & 0.037 & 0.037 & 0.037&0.037 & 0.040 & 0.039 & 0.037 & 0.038 & 0.038&0.038 & 0.040 & 0.038& \parbox[t]{2mm}{\multirow{5}{*}{\rotatebox[origin=c]{-90}{Control-2}}} \\ 
  600 & 0.037 & 0.037 & 0.036&0.037 & 0.046 & 0.043 & 0.037 & 0.037 & 0.035&0.037 & 0.043 & 0.041& \\ 
  1000 & 0.039 & 0.039 &0.040& 0.039 & 0.050 & 0.047 & 0.036 & 0.037 &0.037& 0.038 & 0.048 & 0.042& \\ 
  1400 & 0.038 & 0.038 &0.039& 0.038 & 0.049 & 0.044 & 0.037 & 0.037 &0.038& 0.037 & 0.046 & 0.043& \\ 
  1800 & 0.037 & 0.037 &0.038& 0.037 & 0.048 & 0.047 & 0.037 & 0.037 &0.037& 0.036 & 0.050 & 0.046& \\ \hline
  200 & 0.038 & 0.039 & 0.039&0.039 & 0.038 & 0.038 & 0.065 & 0.066 & 0.061&0.065 & 0.064 & 0.065 & \parbox[t]{2mm}{\multirow{5}{*}{\rotatebox[origin=c]{-90}{IV}}}\\ 
  600 & 0.042 & 0.042 & 0.040&0.042 & 0.041 & 0.042 & 0.066 & 0.066 & 0.066&0.066 & 0.062 & 0.065& \\ 
  1000 & 0.042 & 0.043 &0.042& 0.043 & 0.042 & 0.042 & 0.070 & 0.071 &0.070& 0.070 & 0.066 & 0.069& \\ 
  1400 & 0.040 & 0.039 &0.040& 0.039 & 0.038 & 0.039 & 0.069 & 0.069 &0.067& 0.069 & 0.065 & 0.069& \\ 
  1800 & 0.040 & 0.040 &0.039& 0.040 & 0.039 & 0.040 & 0.066 & 0.066 &0.066& 0.066 & 0.062 & 0.065& \\ 
  \hline
  \textbf{p}&\multicolumn{12}{c|}{Rejection Rate}&\\
  \hline
  200 & 0.068 & 0.062 & 0.062&0.070 & 0.072 & 0.078 & 0.038 & 0.042 & 0.034&0.040 & 0.048 & 0.048 & \parbox[t]{2mm}{\multirow{5}{*}{\rotatebox[origin=c]{-90}{Control-1}}}\\ 
  600 & 0.046 & 0.050 & 0.086&0.050 & 0.042 & 0.138 & 0.068 & 0.068 & 0.068&0.062 & 0.042 & 0.146& \\ 
  1000 & 0.066 & 0.068 &0.054& 0.076 & 0.046 & 0.206 & 0.060 & 0.066 &0.058& 0.068 & 0.044 & 0.150& \\ 
  1400 & 0.064 & 0.068 &0.060& 0.078 & 0.044 & 0.296 & 0.060 & 0.086 &0.052& 0.072 & 0.062 & 0.188& \\ 
  1800 & 0.078 & 0.066 &0.062& 0.064 & 0.050 & 0.262 & 0.090 & 0.080 &0.090& 0.084 & 0.060 & 0.278& \\ \hline
  200 & 0.068 & 0.064 & 0.058&0.068 & 0.048 & 0.082 & 0.046 & 0.058 & 0.062&0.050 & 0.052 & 0.068 & \parbox[t]{2mm}{\multirow{5}{*}{\rotatebox[origin=c]{-90}{Control-2}}}\\ 
  600 & 0.066 & 0.068 & 0.050&0.064 & 0.070 & 0.146 & 0.052 & 0.062 & 0.040&0.062 & 0.054 & 0.098& \\ 
  1000 & 0.070 & 0.060 &0.068& 0.058 & 0.068 & 0.210 & 0.054 & 0.062 &0.050& 0.066 & 0.064 & 0.158& \\ 
  1400 & 0.066 & 0.064 &0.062& 0.066 & 0.056 & 0.214 & 0.064 & 0.066 &0.064& 0.068 & 0.058 & 0.176& \\ 
  1800 & 0.072 & 0.050 &0.060& 0.046 & 0.040 & 0.234 & 0.068 & 0.056 &0.060& 0.060 & 0.056 & 0.230& \\ \hline
  200 & 0.024 & 0.028 & 0.048&0.032 & 0.034 & 0.032 & 0.074 & 0.068 & 0.062&0.066 & 0.074 & 0.070 & \parbox[t]{2mm}{\multirow{5}{*}{\rotatebox[origin=c]{-90}{IV}}}\\ 
  600 & 0.048 & 0.042 & 0.052&0.040 & 0.056 & 0.044 & 0.054 & 0.058 & 0.060&0.058 & 0.068 & 0.054 &\\ 
  1000 & 0.052 & 0.054 &0.056& 0.054 & 0.056 & 0.058 & 0.072 & 0.076 &0.070& 0.072 & 0.086 & 0.078& \\ 
  1400 & 0.054 & 0.060 &0.054& 0.060 & 0.048 & 0.064 & 0.058 & 0.064 &0.054& 0.062 & 0.090 & 0.058& \\ 
  1800 & 0.036 & 0.038 &0.042& 0.040 & 0.048 & 0.040 & 0.048 & 0.052 &0.054& 0.054 & 0.072 & 0.052& \\ 
  \hline
  \textbf{p}&\multicolumn{12}{c|}{Computation Time}&\\
  \hline
200  & 0.025 & 0.010 & 0.009 & 0.015 & 2.656  & 18.189  & 0.023 & 0.008 & 0.008 & 0.014 & 2.648  & 18.404& \parbox[t]{2mm}{\multirow{5}{*}{\rotatebox[origin=c]{-90}{Control-1}}}\\  
600  & 0.067 & 0.046 & 0.022 & 0.039 & 5.251  & 47.713  & 0.058 & 0.044 & 0.021 & 0.035 & 5.223  & 49.271 & \\
1000 & 0.111 & 0.042 & 0.040 & 0.048 & 14.878 & 80.413  & 0.093 & 0.045 & 0.045 & 0.053 & 16.434 & 82.359 & \\
1400 & 0.175 & 0.114 & 0.108 & 0.123 & 22.739 & 206.075 & 0.157 & 0.149 & 0.134 & 0.159 & 23.500 & 209.792 &\\
1800 & 0.210 & 0.169 & 0.161 & 0.166 & 30.224 & 195.819 & 0.180 & 0.179 & 0.181 & 0.202 & 30.523 & 203.132& \\ \hline
200  & 0.026 & 0.008 & 0.008 & 0.014 & 2.692  & 18.623  & 0.023 & 0.008 & 0.008 & 0.013 & 2.724  & 18.790 & \parbox[t]{2mm}{\multirow{5}{*}{\rotatebox[origin=c]{-90}{Control-2}}}\\
600  & 0.067 & 0.039 & 0.020 & 0.032 & 5.332  & 50.050  & 0.062 & 0.043 & 0.027 & 0.037 & 5.330  & 50.029 & \\
1000 & 0.102 & 0.049 & 0.046 & 0.058 & 16.122 & 90.967  & 0.112 & 0.051 & 0.054 & 0.062 & 17.141 & 95.673&  \\
1400 & 0.185 & 0.137 & 0.128 & 0.154 & 24.310 & 214.960 & 0.164 & 0.118 & 0.113 & 0.128 & 23.210 & 216.107& \\
1800 & 0.223 & 0.165 & 0.163 & 0.162 & 30.335 & 201.714 & 0.201 & 0.251 & 0.218 & 0.230 & 32.531 & 211.216& \\ \hline
200  & 0.028 & 0.009 & 0.008 & 0.013 & 2.746  & 18.605  & 0.028 & 0.009 & 0.008 & 0.013 & 2.759  & 20.042 & \parbox[t]{2mm}{\multirow{5}{*}{\rotatebox[origin=c]{-90}{IV}}}\\
600  & 0.068 & 0.042 & 0.020 & 0.031 & 11.119 & 49.314  & 0.062 & 0.039 & 0.021 & 0.032 & 12.818 & 51.969  &\\
1000 & 0.110 & 0.042 & 0.041 & 0.050 & 20.740 & 90.738  & 0.098 & 0.043 & 0.043 & 0.051 & 23.254 & 91.486 & \\
1400 & 0.168 & 0.108 & 0.108 & 0.122 & 41.661 & 303.407 & 0.152 & 0.110 & 0.109 & 0.123 & 29.069 & 231.115& \\
1800 & 0.272 & 0.157 & 0.151 & 0.156 & 36.197 & 198.744 & 0.210 & 0.164 & 0.170 & 0.187 & 38.697 & 213.886&\\
  \hline
  \end{tabular}}
\caption{Simulation results for $n=800$ and $SNR=3$.} 
\label{simresults_n800_SNR3}
\end{table}

\pagebreak

\begin{figure}[h]
\centering
\includegraphics[width=\linewidth]{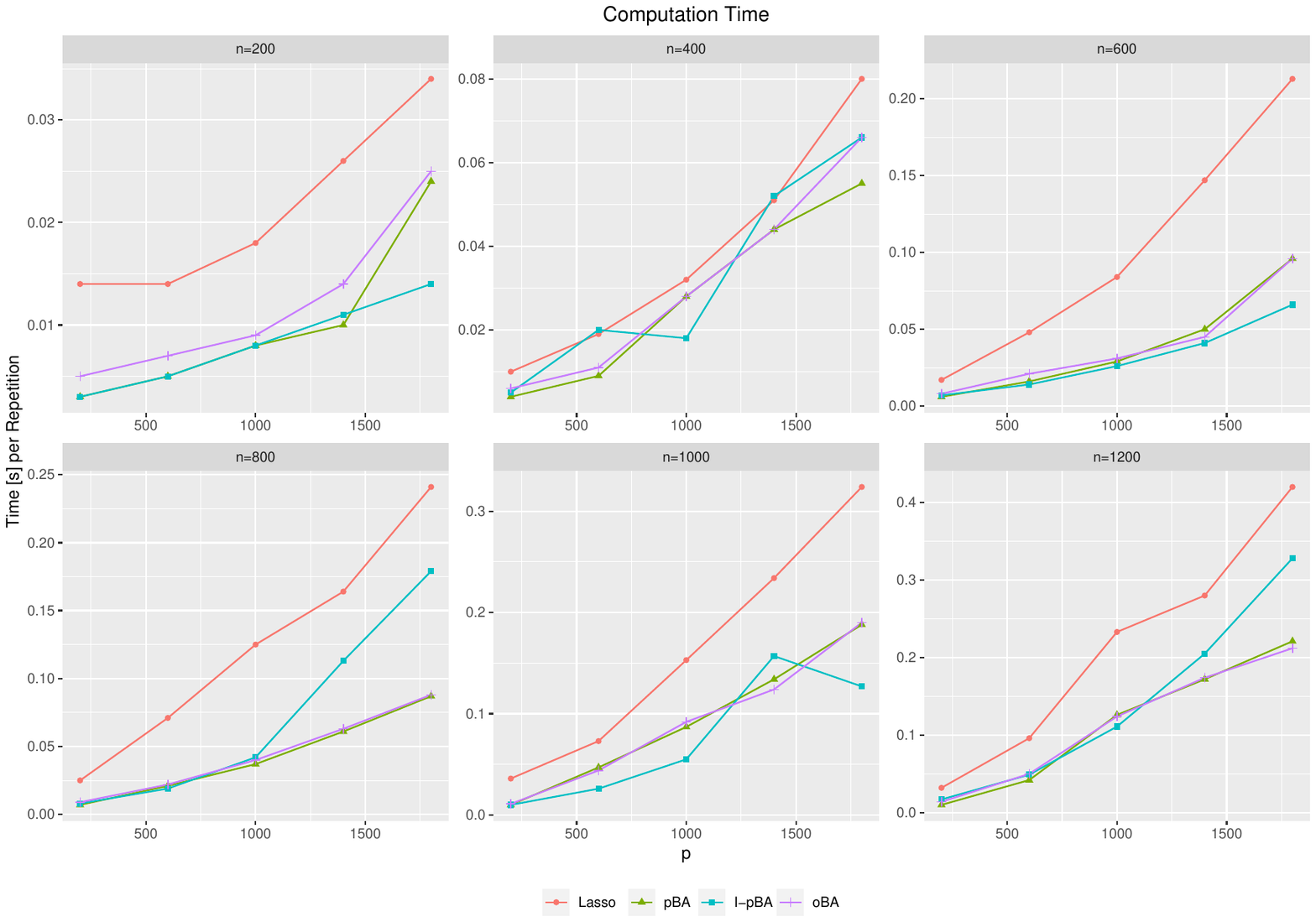}
\caption{Computation time for increasing sample size in the sparse i.i.d. setting (Control-1).}
\label{time_all}
\end{figure}
\ \\
\pagebreak

\section{Additional Details of the Application}\label{details_application}

\cite{levine2020bank} construct FNNs and SNNs as potential IVs. The following table gives the three sets of FNNs that they constructed but did not use.

For a given acquirer-target-year observation, one can list the 51-element vector of states that this acquirer state can access under that year’s bank branching deregulation laws. Each of the 51 elements equals one if the represented state can be accessed by the acquirer state and equals zero otherwise, similar to the construction of OLS independent variables in \cite{levine2020bank}. For any state X, we can also list the vector of states that state X can access.
\begin{table}[htbp]
  \centering
    \begin{tabular}{p{1.915em}p{32.5em}}
    \hline\hline \\[-2ex]
    FNN1 & 
     For a given acquirer-target-year observation, we can list the $51$-element vector of states that this acquirer state can access under that year’s bank branching deregulation laws. Each of the $51$ elements equals one if the represented state can be accessed by the acquirer state and equals zero otherwise, similar to the construction of OLS independent variables in \cite{levine2020bank}. For any state $S_0$, we can also list the vector of states that state $S_0$ can access. Then $FNN1_{S_0}$ is defined as the \textit{Overlap}, \textit{Correlation Coefficient}, or \textit{Cosine Distance} of the two vectors, multiplied by the geographic distance ratio (GDR1), which is equal to (distance between acquirer state and target state)/(distance between acquirer state and state $S_0$ + distance between target state and state $S_0$). 
     Formally, the FNN1 IV for \textit{Overlap} is defined as:
     \begin{equation*}
     k\textrm{-}NN(S_0,S_1,S_2):=\frac{\sum_{S\in \mathcal{S}} 1(S\in N_k(S_1) \cap N_k(S_0))   }{\sum_{S\in \mathcal{S}} 1(S\in N_k(S_1)\cup N_k(S_0)) }GDR1(S_0,S_1,S_2).
     \end{equation*}
     For example, $FNN1_{NJ}$ is the multiplication product between 1) the \textit{Overlap}, \textit{Correlation Coefficient}, or \textit{Cosine Distance} of acquirer-state-accessible states and New Jersey-accessible states, and 2) the distance ratio. Since $S_0$ can be any state among $50$ states and the District of Columbia ($DC$),  $51$ FNN1s are constructed.\\
    \hline\\[-2ex]
    FNN2   & 
    For any state $S_0$, $FNN2_{S_0}$ is similar to $FNN1_{S_0}$; the only difference is that the geographic distance ratio (GDR2) is replaced with 1/(distance between acquirer state and state $S_0$ + distance between target state and state $S_0$). 
    Formally, the FNN2 IV for \textit{Overlap} is defined as:
    \begin{equation*}
    k\textrm{-}NN(S_0,S_1,S_2):=\frac{\sum_{S\in \mathcal{S}} 1(S\in N_k(S_1) \cap N_k(S_0))   }{\sum_{S\in \mathcal{S}} 1(S\in N_k(S_1)\cup N_k(S_0)) }GDR2(S_0,S_1,S_2).
    \end{equation*}
    Likewise, $51$ $FNN2$s are constructed.\\
    \hline\\[-2ex]
    FNN3   & 
    For a given acquirer-target-year observation, $FNN3_{S_0}$ is an indicator variable that equals one if the acquirer state and the target state can both access state $S_0$ in that year, multiplied by the geographic distance ratio (GDR1), which is equal to (distance between acquirer state and target state)/(distance between acquirer state and state $S_0$ + distance between target state and state $S_0$). 
    Formally, the FNN3 IV is defined as:
    \begin{equation*}
    k\textrm{-}NN(S_0,S_1,S_2):= 1(S_0\in N_k(S_1) \cap N_k(S_2))   GDR1(S_0,S_1,S_2).
    \end{equation*}
    $51$ $FNN3$s are constructed.\\ \\[-2ex]
    \hline\hline
    \end{tabular}
    \caption{Description of 153 FNN IVs.}
  \label{IV_construction}
\end{table}\ \\
\newpage\ \\
\newpage

\bibliographystyle{aea}
\bibliography{Literatur_NR, mybibAR}

\end{document}